\documentclass{article}

\usepackage{etoolbox}
\newtoggle{iclr}
\togglefalse{iclr}

\newcommand{\arxiv}[1]{\iftoggle{iclr}{}{#1}}
\newcommand{\loose}{\looseness=-1}

\usepackage[utf8]{inputenc} %
\usepackage[T1]{fontenc}    %
\usepackage{url}            %
\usepackage{booktabs}       %
\usepackage{amsfonts}       %
\usepackage{nicefrac}       %
\usepackage{microtype}      %
\usepackage{multirow}
\usepackage{tocloft}            %

\usepackage[shortlabels]{enumitem}

\usepackage{breakcites}
\usepackage[normalem]{ulem}

\usepackage{mathrsfs}

\usepackage{algorithm}
\usepackage{verbatim}
\usepackage[noend]{algpseudocode}
\newcommand{\multiline}[1]{\parbox[t]{\dimexpr\linewidth-\algorithmicindent}{#1}}

\usepackage{multicol}

\usepackage{colortbl}

\usepackage{setspace}

\usepackage{transparent}

\usepackage{inconsolata}
\usepackage[scaled=.90]{helvet}
\usepackage{xspace}

\usepackage{pifont}
\usepackage{bm}

\DeclareFontFamily{U}{jkpmia}{}
\DeclareFontShape{U}{jkpmia}{m}{it}{<->s*jkpmia}{}
\DeclareFontShape{U}{jkpmia}{bx}{it}{<->s*jkpbmia}{}
\DeclareMathAlphabet{\mathfrak}{U}{jkpmia}{m}{it}
\SetMathAlphabet{\mathfrak}{bold}{U}{jkpmia}{bx}{it}

\usepackage{./txie}

\arxiv{
\usepackage[letterpaper, left=1in, right=1in, top=1in,
bottom=1in]{geometry}
  \usepackage{parskip}
}

\PassOptionsToPackage{hypertexnames=false}{hyperref}  %

\arxiv{
  \usepackage[dvipsnames]{xcolor}
  }
\usepackage[colorlinks=true, linkcolor=blue!70!black, citecolor=blue!70!black,urlcolor=black,breaklinks=true]{hyperref}
\colorlet{txblue}{RoyalBlue!70!NavyBlue}
\hypersetup{linkcolor=txblue,
            citecolor=txblue}
\usepackage{microtype}
\usepackage{hhline}

\usepackage{amsthm}
\usepackage{mathtools}
\usepackage{amsmath}
\usepackage{bbm}
\usepackage{amsfonts}
\usepackage{amssymb}
\usepackage[nameinlink,capitalize]{cleveref}

\makeatletter
\newcommand{\neutralize}[1]{\expandafter\let\csname c@#1\endcsname\count@}
\makeatother

\usepackage{algorithm}

\arxiv{
\usepackage{natbib}
\bibliographystyle{plainnat}
\bibpunct{(}{)}{;}{a}{,}{,}
}

\usepackage{xpatch}

\usepackage{thmtools}
\usepackage{thm-restate}
\declaretheorem[name=Theorem,parent=section]{theorem}
\declaretheorem[name=Lemma,parent=section]{lemma}
\declaretheorem[name=Assumption, parent=section]{assumption}
\declaretheorem[name=Condition, parent=section]{condition}
\declaretheorem[qed=$\triangleleft$,name=Example,style=definition, parent=section]{example}
\declaretheorem[name=Remark, parent=section]{remark}
\declaretheorem[name=Proposition, parent=section]{proposition}

\usepackage{crossreftools}
\newcommand{\creftitle}[1]{\crtcref{#1}}

\makeatletter
  \renewenvironment{proof}[1][Proof]%
  {%
   \par\noindent{\bfseries\upshape {#1.}\ }%
  }%
  {\qed\newline}
  \makeatother

\theoremstyle{definition}  %

\newtheorem{corollary}{Corollary}[section]

\theoremstyle{plain}
\newtheorem{definition}{Definition}[section]

\xpatchcmd{\proof}{\itshape}{\normalfont\proofnameformat}{}{}
\newcommand{\proofnameformat}{\bfseries}

\newcommand{\pfref}[1]{Proof of \cref{#1}}

\renewcommand{\eqref}[1]{\texorpdfstring{\hyperref[#1]{(\ref*{#1})}}{(\ref*{#1})}}
\crefformat{equation}{#2Eq.\,(#1)#3}
\Crefformat{equation}{#2Eq.\,(#1)#3}

\Crefformat{figure}{#2Figure~#1#3}
\Crefformat{assumption}{#2Assumption~#1#3}

\Crefname{assumption}{Assumption}{Assumptions}

\crefname{fact}{Fact}{Facts}

\Crefformat{figure}{#2Figure #1#3}
\Crefformat{assumption}{#2Assumption #1#3}

\usepackage{crossreftools}
\usepackage{xparse}

\ExplSyntaxOn
\DeclareDocumentCommand{\XDeclarePairedDelimiter}{mm}
 {
  \__egreg_delimiter_clear_keys: %
  \keys_set:nn { egreg/delimiters } { #2 }
  \use:x %
   {
    \exp_not:n {\NewDocumentCommand{#1}{sO{}m} }
     {
      \exp_not:n { \IfBooleanTF{##1} }
       {
        \exp_not:N \egreg_paired_delimiter_expand:nnnn
         { \exp_not:V \l_egreg_delimiter_left_tl }
         { \exp_not:V \l_egreg_delimiter_right_tl }
         { \exp_not:n { ##3 } }
         { \exp_not:V \l_egreg_delimiter_subscript_tl }
       }
       {
        \exp_not:N \egreg_paired_delimiter_fixed:nnnnn 
         { \exp_not:n { ##2 } }
         { \exp_not:V \l_egreg_delimiter_left_tl }
         { \exp_not:V \l_egreg_delimiter_right_tl }
         { \exp_not:n { ##3 } }
         { \exp_not:V \l_egreg_delimiter_subscript_tl }
       }
     }
   }
 }

\keys_define:nn { egreg/delimiters }
 {
  left      .tl_set:N = \l_egreg_delimiter_left_tl,
  right     .tl_set:N = \l_egreg_delimiter_right_tl,
  subscript .tl_set:N = \l_egreg_delimiter_subscript_tl,
 }

\cs_new_protected:Npn \__egreg_delimiter_clear_keys:
 {
  \keys_set:nn { egreg/delimiters } { left=.,right=.,subscript={} }
 }

\cs_new_protected:Npn \egreg_paired_delimiter_expand:nnnn #1 #2 #3 #4
 {%
  \mathopen{}
  \mathclose\c_group_begin_token
   \left#1
   #3
   \group_insert_after:N \c_group_end_token
   \right#2
   \tl_if_empty:nF {#4} { \c_math_subscript_token {#4} }
 }
\cs_new_protected:Npn \egreg_paired_delimiter_fixed:nnnnn #1 #2 #3 #4 #5
 {
  \mathopen{#1#2}#4\mathclose{#1#3}
  \tl_if_empty:nF {#5} { \c_math_subscript_token {#5} }
 }
\ExplSyntaxOff

\XDeclarePairedDelimiter{\supnorm}{
  left=\lVert,
  right=\rVert,
  subscript=\infty
  }

\DeclarePairedDelimiter{\abs}{\lvert}{\rvert} %
\DeclarePairedDelimiter{\brk}{[}{]}
\DeclarePairedDelimiter{\crl}{\{}{\}}
\DeclarePairedDelimiter{\prn}{(}{)}
\DeclarePairedDelimiter{\nrm}{\|}{\|}
\DeclarePairedDelimiter{\tri}{\langle}{\rangle}

\DeclareMathOperator{\En}{\mathbb{E}}

\newcommand{\mb}[1]{\boldsymbol{#1}}
\newcommand{\wb}[1]{\widebar{#1}}

\def\ddefloop#1{\ifx\ddefloop#1\else\ddef{#1}\expandafter\ddefloop\fi}
\def\ddef#1{\expandafter\def\csname bb#1\endcsname{\ensuremath{\mathbb{#1}}}}
\ddefloop ABCDEFGHIJKLMNOPQRSTUVWXYZ\ddefloop
\def\ddefloop#1{\ifx\ddefloop#1\else\ddef{#1}\expandafter\ddefloop\fi}
\def\ddef#1{\expandafter\def\csname b#1\endcsname{\ensuremath{\mathbf{#1}}}}
\ddefloop ABCDEFGHIJKLMNOPQRSTUVWXYZ\ddefloop
\def\ddef#1{\expandafter\def\csname sf#1\endcsname{\ensuremath{\mathsf{#1}}}}
\ddefloop ABCDEFGHIJKLMNOPQRSTUVWXYZ\ddefloop
\def\ddef#1{\expandafter\def\csname c#1\endcsname{\ensuremath{\mathcal{#1}}}}
\ddefloop ABCDEFGHIJKLMNOPQRSTUVWXYZ\ddefloop
\def\ddef#1{\expandafter\def\csname h#1\endcsname{\ensuremath{\widehat{#1}}}}
\ddefloop ABCDEFGHIJKLMNOPQRSTUVWXYZ\ddefloop
\def\ddef#1{\expandafter\def\csname hc#1\endcsname{\ensuremath{\widehat{\mathcal{#1}}}}}
\ddefloop ABCDEFGHIJKLMNOPQRSTUVWXYZ\ddefloop
\def\ddef#1{\expandafter\def\csname t#1\endcsname{\ensuremath{\widetilde{#1}}}}
\ddefloop ABCDEFGHIJKLMNOPQRSTUVWXYZ\ddefloop
\def\ddef#1{\expandafter\def\csname tc#1\endcsname{\ensuremath{\widetilde{\mathcal{#1}}}}}
\ddefloop ABCDEFGHIJKLMNOPQRSTUVWXYZ\ddefloop
\def\ddefloop#1{\ifx\ddefloop#1\else\ddef{#1}\expandafter\ddefloop\fi}
\def\ddef#1{\expandafter\def\csname scr#1\endcsname{\ensuremath{\mathscr{#1}}}}
\ddefloop ABCDEFGHIJKLMNOPQRSTUVWXYZ\ddefloop

\newcommand{\ls}{\ell}

\newcommand{\eps}{\epsilon}
\newcommand{\veps}{\varepsilon}

\newcommand{\ldef}{\vcentcolon=}
\newcommand{\rdef}{=\vcentcolon}

\newcommand{\pidpo}{\pihat_{\normalfont\textsf{DPO}}}
\newcommand{\pichipo}{\pihat_{\normalfont\textsf{\algshort}}}

\newcommand{\Zklr}[1][r]{Z_{\beta,#1;\mathsf{KL}}}
\newcommand{\Zklrstar}{\Zklr[\rstar]}
\newcommand{\Zfr}[1][r]{Z_{\beta,#1;f}}
\newcommand{\Zfrstar}{\Zfr[\rstar]}
\newcommand{\Zfmixr}[1][r]{Z_{\beta,#1;\fmix}}
\newcommand{\Zr}[1][r]{Z_{\beta,#1}}
\newcommand{\Zrstar}{\Zr[\rstar]}

\newcommand{\Pidpo}{\Pi_{\textsf{DPO},\beta}}
\newcommand{\Pichipo}{\Pi_{\textsf{\algshort},\beta}}

\newcommand{\Jr}[1][r]{J_{#1}}

\newcommand{\pistarkl}{\pistar_{\beta;\mathsf{KL}}}

\newcommand{\betastar}{\beta^{\star}}

\newcommand{\dom}{\mathrm{dom}}
\newcommand{\chimix}{\chi_{\textsf{mix}}}
\newcommand{\fmix}{f_{\chimix}}
\newcommand{\Jmix}{J_{\beta}^{\chimix}}
\newcommand{\Jmixr}[1][r]{J_{\beta,#1}^{\chimix}}
\newcommand{\Jmixgr}[1][r]{J_{\beta,\gamma,#1}^{\chimix}}
\newcommand{\Jmixg}[1][r]{J_{\beta,\gamma}^{\chimix}}
\newcommand{\Dfmix}[2]{D_{\fmix}\prn*{#1\,\|\,#2}}

\newcommand{\pistarb}[1][r]{\pistar_{\sss{\beta}}}

\newcommand{\rstar}{r^\star}

\newcommand{\Dsmthchis}[2]{D_{\chi^2;\eta}\prn*{#1\dmid{}#2}}

\newcommand{\vepsopt}{\veps_{\mathrm{opt}}}
\newcommand{\Jhatsmth}{\wh{J}_{\beta,\eta}}

\newcommand{\fmixg}{f_{\chimix,\gamma}}
\newcommand{\Dfmixg}[2]{D_{\fmixg}\prn*{#1\,\|\,#2}}

\newcommand{\pistarg}{\pistarbg}
\newcommand{\linkg}{\link_\gamma}
\newcommand{\rhatdiff}{\wh\Delta}
\newcommand{\rstardiff}{\Delta^\star}
\newcommand{\lambdag}{\lambda^\star_{\beta,\gamma}}
\newcommand{\alphag}{\alpha^\star_{\beta,\gamma}}

\newcommand{\pistarbg}{\pistar_{\beta,\gamma}}

\newcommand{\pistarsmth}{\pi^\star_{\beta,\eta}}
\newcommand{\Jsmth}[1][{}]{J_{\beta,\eta}^{{#1}}}
\newcommand{\Jsmthr}[1][{}]{J_{\beta,\eta}^{{#1}}}
\newcommand{\Jhatsmthr}[1][{}]{\wh{J}_{\beta,\eta}^{{#1}}}
\newcommand{\rdiff}[1][r]{\Delta^{{#1}}}
\newcommand{\Csmth}[1][\pi]{\cC_\eta^{#1}}
\newcommand{\Chatsmth}[1][\pi]{\wh{\cC}_\eta^{#1}}

\newcommand{\vepsx}{\veps_{\mathrm{x}}}

\newcommand{\nopref}{{\mathsf{x}}}
\newcommand{\cDnopref}{\cD_\nopref}

\newcommand{\link}{\phi}
\newcommand{\linkmix}{\link_{\gamma}}
\newcommand{\clip}{\mathsf{clip}}

\newcommand{\DG}{\mathsf{DG}}

\newcommand{\cDpref}{\cD_{\pref}}

\newcommand{\pimw}{\pi_{\mathsf{MW}}}

\newcommand{\setcomp}{\mathrm{c}}

\newcommand{\ap}{a_{+}}
\newcommand{\am}{a_{-}}
\newcommand{\atil}{\wt{a}}

\newcommand{\Cone}[1][\pi]{\cC^{#1}}
\newcommand{\Cinf}[1][\pi]{\cC_{\infty}^{#1}}

\newcommand{\Jchi}{J_{\beta}^{\chi}}

\newcommand{\chis}{$\chi^2$}
\newcommand{\chisb}{$\mb{\chi^2}$}

\newcommand{\tautil}{\wt{\tau}}%

\newcommand{\mainalg}{$\chi^2$-Preference Optimization\xspace}
\newcommand{\alglong}{\mainalg}
\newcommand{\mainalgb}{$\mb{\chi^2}$-Preference Optimization\xspace}
\newcommand{\alglongb}{\mainalgb}
\newcommand{\algshort}{\texttt{$\chi$PO}\xspace} %
\newcommand{\algshortb}{\texttt{$\mb{\chi}$PO}\xspace} %
\newcommand{\rlhfalg}{\texttt{$\chi^2$-RLHF}\xspace}

\newcommand{\dpo}{\texttt{DPO}\xspace}
\newcommand{\sft}{\texttt{SFT}\xspace}
\newcommand{\ppo}{\texttt{PPO}\xspace}

\newcommand{\taup}{\tau_{+}}
\newcommand{\taum}{\tau_{-}}

\newcommand{\pitil}{\wt{\pi}}%

\newcommand{\pibar}{\wb{\pi}}

\newcommand{\vepsstat}{\veps_{\mathrm{stat}}}

\renewcommand{\emptyset}{\varnothing}

\newcommand{\M}[1]{^{{\scriptscriptstyle M}}}  %

\newcommand{\sss}[1]{{\scriptscriptstyle#1}}

\newcommand{\pistar}{\pi^{\star}}

\newcommand{\pihat}{\wh{\pi}}

\newcommand{\psdgt}{\succ}
\newcommand{\approxleq}{\lesssim}
\newcommand{\approxgeq}{\gtrsim}

\newcommand{\ind}[1]{^{#1}}

\newcommand{\bigoh}{O}
\newcommand{\bigoht}{\wt{O}}
\newcommand{\bigom}{\Omega}

\newcommand{\indic}{\mathbb{I}}

\newcommand{\poly}{\mathrm{poly}}

\newcommand{\Df}[2]{D_{f}\prn*{#1\,\|\,#2}}
\newcommand{\Dkl}[2]{D_{\mathsf{KL}}\prn*{#1\,\|\,#2}}

\newcommand{\Dhel}[2]{D_{\mathsf{H}}\prn*{#1,#2}}

\newcommand{\Dhels}[2]{D^{2}_{\mathsf{H}}\prn*{#1,#2}}

\newcommand{\Dchis}[2]{D_{\chi^2}\prn*{#1\dmid{}#2}}

\newcommand{\Ber}{\mathrm{Ber}}

\newcommand{\dmid}{\;\|\;}

\newcommand{\mathand}{\quad\text{and}\quad}

\def\multiset#1#2{\ensuremath{\left(\kern-.3em\left(\genfrac{}{}{0pt}{}{#1}{#2}\right)\kern-.3em\right)}}

\renewcommand{\ls}{\ell}

\renewcommand{\emptyset}{\varnothing}

\input{widebar}

\usepackage[suppress]{color-edits}
\addauthor[Tengyang]{tx}{violet}
\addauthor{tx}{violet}
 
 \addauthor{df}{ForestGreen}
\newcommand{\dfc}[1]{}
\addauthor{ak}{orange}

\addauthor{ws}{red}

\addauthor{ah}{brown}

\makeatletter
\let\OldStatex\Statex
\renewcommand{\Statex}[1][3]{%
  \setlength\@tempdima{\algorithmicindent}%
  \OldStatex\hskip\dimexpr#1\@tempdima\relax}
\makeatother

 \usepackage{accents}

\let\oldparagraph\paragraph

\renewcommand{\paragraph}[1]{\oldparagraph{#1.}}

\allowdisplaybreaks

\arxiv{                      
\title{Correcting the Mythos of KL-Regularization: \\ {\Large Direct Alignment without Overoptimization via \chis-Preference Optimization}}
}

\author{\and
  Audrey Huang
\\
  \footnotesize\href{mailto:audreyh5@illinois.edu}{\texttt{audreyh5@illinois.edu}}
\and
Wenhao Zhan
\\
  \footnotesize\href{mailto:wenhao.zhan@princeton.edu}{\texttt{wenhao.zhan@princeton.edu}}
\and
Tengyang Xie
\\
\footnotesize
\href{mailto:tx@cs.wisc.edu}{\texttt{tx@cs.wisc.edu}}
\and\and
\and
Jason D. Lee
\\
\footnotesize
\href{mailto:jasondlee88@gmail.com}{\texttt{jasondlee88@gmail.com}}
\and
Wen Sun
\\
\footnotesize
\href{mailto:ws455@cornell.edu}{\texttt{ws455@cornell.edu}}
\and
Akshay Krishnamurthy
\\
\footnotesize
\href{mailto:akshaykr@microsoft.com}{\texttt{akshaykr@microsoft.com}}
\and
Dylan J. Foster
\\
\footnotesize
\href{mailto:dylanfoster@microsoft.com}{\texttt{dylanfoster@microsoft.com}}
}

\date{\today}

  \addtocontents{toc}{\protect\setcounter{tocdepth}{0}}

\begin{document}
\maketitle

\begin{abstract}
  
  Language model alignment methods such as reinforcement learning from human feedback (RLHF)
have led to impressive advances in language model capabilities, 
but are limited by a widely observed phenomenon known as \emph{overoptimization}, 
where the quality of the language model degrades over the course of the alignment process. 
As the model optimizes performance with respect to an offline reward model, 
it overfits to inaccuracies and drifts away from preferred responses 
covered by the data. 
To discourage such distribution shift, KL-regularization is widely employed in existing offline alignment methods, 
but overoptimization continues to harm performance. 
Lending theoretical insight into the source of these empirical observations, 
we first show that the KL-regularization is too weak to prevent overfitting,
then raise the following question: 
is it possible to design an efficient algorithm that is provably robust to overoptimization? 
\loose

\arxiv{
We address this question with a new algorithm for offline
alignment, \emph{\alglong} (\algshort). \arxiv{
\noindent }\algshort is a one-line change to Direct Preference
Optimization (\dpo; \citet{rafailov2024direct}), which only involves
modifying the logarithmic link function in the \dpo{} objective.
Despite this minimal change, \algshort implicitly implements the
principle of \emph{pessimism in the face of uncertainty} via
regularization with the \chis-divergence---which quantifies uncertainty
more effectively than KL-regularization---and provably alleviates
overoptimization, achieving sample-complexity guarantees based on
\emph{single-policy concentrability}---the gold standard in offline
reinforcement learning. \algshort's simplicity and strong
  guarantees make it the first practical and general-purpose offline
  alignment algorithm that is provably robust to overoptimization.
  }

\end{abstract}

\section{Introduction}
\label{sec:intro}
Large language models (LLMs) trained on unsupervised text data exhibit impressive and
surprising capabilities
\citep{brown2020language,ouyang2022training,touvron2023llama,achiam2023gpt,anil2023palm},
but can be difficult to control without further guidance. \emph{Reinforcement learning from human feedback (RLHF)} and other alignment methods have emerged as a
central tool to align these models to human values and elicit desired 
behavior
\citep{christiano2017deep,bai2022training,ouyang2022training,rafailov2024direct}. This
is achieved by treating the language model as a  \emph{policy}, and using techniques from
reinforcement learning to optimize for desirable outcomes under a
(explicit or implicit) reward model learned from a dataset of
human-labeled responses. \loose

Alignment methods like RLHF have led to significant advances in language model capabilities,
but existing techniques are limited by a
widely observed phenomenon known as \emph{reward overoptimization} or
\emph{reward hacking}
\citep{michaud2020understanding,tien2022causal,gao2023scaling,rafailov2024scaling}. Since
the reward model 
is an imperfect proxy for human
preferences, the true quality of the language model can degrade as
training proceeds, even as its performance under the reward model
continues to improve. Intuitively, this occurs because the language
model may drift away from the manifold covered by the human-labeled
data used to train the reward model and end up in a region where the
reward model is inaccurate.\loose

Overoptimization is distinct from the classical concept of overfitting
  because it is a causal or counterfactual
  phenomenon:
  When the human-labeled dataset does not cover all possible alternatives, the decision maker---in this case, a language model policy---cannot
directly evaluate the effect of their actions. This perspective is
supported by the fact that overoptimization can be mitigated by
\emph{online} alignment techniques
\citep{guo2024direct,gao2024rebel,dong2024rlhf}, which exploit
interactive access to human or AI feedback to iteratively improve the
reward model; unfortunately, gathering such
feedback is costly and impractical in many settings. This raises natural  questions regarding the role of overoptimization in
\emph{offline alignment}:
\begin{itemize}
\item Is overoptimization in %
    offline alignment an \emph{information-theoretic phenomenon}? This would
  mean that there is simply not enough information in the
  human-labeled (offline) preference dataset due to
  partial coverage, and no algorithmic intervention 
  can avoid the overoptimization issue.
\item Alternatively, is overoptimization an \emph{algorithmic
    phenomenon}? This would mean that existing algorithms are not making the most of the
data they have (e.g., due to optimizing the wrong objective
and converging toward suboptimal solutions) and would suggest that their
\emph{sample-efficiency} can be improved, perhaps by
taking more aggressive measures to avoid overfitting to the reward model.
\end{itemize}
Previous developments in the theory of offline
  reinforcement learning suggest that the answer may be the latter. Indeed, this literature has
addressed the challenge of overoptimization---typically
referred to as \emph{distribution shift}---through the
principle of \emph{pessimism in the face of uncertainty}, which asserts that, given an offline dataset with
partial coverage, a decision maker should choose their response according to the
  most pessimistic view of the world supported by the data. Pessimism
  encourages the model to avoid overfitting to the offline
  dataset and is supported by a rich theory \arxiv{showing that it offers}%
  provable robustness to overoptimization in stylized settings \citep{liu2020provably,jin2021pessimism\arxiv{,rashidinejad2021bridging}}.\loose

  Perhaps the greatest barrier to implementing pessimism in
  language models is the efficient quantification of uncertainty in the offline reward, and the distillation of this information into actionable form. 
Most existing offline alignment methods employ KL-regularization, which penalizes the learned policy for drifting from the reference policy, but this form of
uncertainty quantification is insufficient to induce pessimism \citep{gao2023scaling} and is provably suboptimal in theory
\citep[][see also~\cref{sec:rpo_lower}]{zhu2023principled,song2024understanding}.
On the other hand,
offline reinforcement learning theory offers abstract pessimistic algorithms that are suitable---at least statistically---for large
models
\citep{xie2021bellman,uehara2021pessimistic,zhan2022offline,chen2022offline},
but cannot be
implemented directly without losing theoretical fidelity or making
unrealistic modeling assumptions \citep{zhu2023principled,zhan2023provable,li2023reinforcement,xiong2023gibbs,liu2024provably,cen2024value, fisch2024robust,ji2024selfplay}.
Notably, the so-called ``\texttt{DPO}+\sft'' approach developed by~\citet{liu2024provably,cen2024value,fisch2024robust} is provably suboptimal unless the language model satisfies an unrealistic convexity property (\cref{sec:rpo_lower}).
\arxiv{Can we develop practical offline alignment methods with provable robustness
to overoptimization by exploiting the unique structure of the language modeling problem?\loose}

  \subsection{Contributions}

We introduce a new algorithm for offline alignment, \emph{\alglong}
(\algshort). \algshort is simple and straightforward to implement,
requiring only a single-line change to Direct Preference Optimization
(\citet{rafailov2024direct}), yet it is provably robust to
overoptimization.  Algorithmically, \algshort only differs from \dpo{}
in that we replace the usual logarithmic link function in the \dpo{}
objective with a new link function that implicitly implements
pessimism via regularization with the \chis-divergence---a divergence that
(i) plays a fundamental role in statistics due to its ability to
quantify uncertainty \citep{tsybakov2008introduction}; and (ii)
penalizes off-manifold behavior more effectively than
KL-regularization.  Statistically, we formalize robustness to
overoptimization via a sample complexity guarantee based on
\emph{single-policy concentrability}---the gold standard in offline
reinforcement learning---which we establish under minimal statistical
and function approximation assumptions. This result implies that, in
contrast to most prior work, \algshort enjoys meaningful guarantees
even when the reference policy has poor coverage.
Summarizing:\loose
\begin{center}
\arxiv{  \emph{\algshort is the first practical, general-purpose algorithm for offline alignment \\with
        provable robustness to overoptimization.}}
    \end{center}

The result above concerns the classical language model alignment
formulation, which assumes the Bradley-Terry preference model
\citep{christiano2017deep,bai2022training,ouyang2022training,rafailov2024direct}. Turning
our attention to general preference models
\citep{munos2023nash,swamy2024minimaximalist,rosset2024direct} where the goal is to find an approximate Nash equilibrium, we
show that \emph{achieving guarantees based on single-policy
concentrability is impossible}. Nonetheless, we show that an iterative
variant of \algshort based on self-play 
achieves a sample complexity guarantee that scales with a new local coverage condition
---a condition that is
stronger than single policy concentrability,
 but much weaker than
global concentrability and the notion of unilateral concentrability introduced by \cite{cui2022offline}.
This result provides additional evidence for the value of regularization with \chis-divergence for obtaining sharp sample complexity guarantees\arxiv{ in language model alignment}.\loose

\paragraph{Technical highlights}
Our analysis of \algshort leverages several new techniques. 
First, we show that RLHF with \chis-regularization is sufficient to achieve guarantees based on single-policy concentrability (\cref{sec:framework} and \cref{sec:rlhf}). 
Next, we show that a variant of the \dpo
reparameterization trick that combines \chis-regularization with
KL-regularization (``mixed'' \chis-regularization) can be used to reformulate our objective
into a purely policy-based objective, in spite of the fact that
\chis-regularization fails to satisfy certain regularity conditions
found in prior work \citep{wang2023beyond}. Finally, and perhaps most
importantly, we use a novel analysis to show that pessimism is preserved after
reparameterization.
\arxiv{

}
Compared to prior approaches to pessimism in offline RL 
\citep{xie2021bellman,uehara2021pessimistic,zhan2022offline,chen2022offline},
\chis-regularization strikes a useful balance between generality and
tractability.\arxiv{ We expect our techniques---particularly the
use of mixed \chis-regularization---to find broader use.\loose}

\arxiv{
\subsection{Paper Organization}
  \cref{sec:background} provides background on offline
  alignment and the suboptimality of existing algorithms. \cref{sec:main}
    presents our main algorithm, \algshort, and accompanying theoretical
    guarantees. \cref{sec:understanding} then presents detailed intuition
    into how \algshort modulates the bias-overoptimization tradeoff
    and implements pessimism, and \cref{sec:proof_sketch} sketches the
    proof for its main statistical guarantee.
    We perform experimental evaluations of \algshort against \dpo
    in the TL;DR summarization task \citep{stiennon2020learning}, 
    which is included in \cref{sec:experiments}.

    \cref{sec:general_preference} contains results for general
    preference models, including an impossibility result for
obtaining guarantees under single-policy concentrability in this setting. We conclude with discussion in
    \cref{sec:discussion}. Proofs and additional results are deferred
    to the appendix, with highlights including (i) detailed discussion
    on suboptimality of existing pessimistic approaches
    (\cref{sec:related}), and (ii) additional algorithms and
    guarantees based on the \chis-regularization framework
    (\cref{sec:rlhf}).\loose
}

\arxiv{
\paragraph{Notation}
  For an integer $n\in\bbN$, we let $[n]$ denote the set
  $\{1,\dots,n\}$. For a set $\cX$, we let $\Delta(\cX)$ denote the
  set of all probability distributions over $\cX$. We adopt
    standard big-oh notation, and write $f=\bigoht(g)$ to denote that
    $f = \bigoh(g\cdot{}\max\crl*{1,\mathrm{polylog}(g)})$ and
    $a\approxleq{}b$ as shorthand for $a=\bigoh(b)$. \loose
}

\section{Background}
\label{sec:background}

In this section, we provide necessary background\arxiv{. We formally introduce the problem of language model alignment from human feedback (offline alignment), review
standard algorithms (\ppo and \dpo), and highlight that in general, these algorithms
suffer from provably suboptimal sample complexity arising from overoptimization, necessitating
algorithmic interventions.} %

\subsection{Alignment from Human Feedback}
\label{sec:rlhf}

Following prior work (e.g.,
\citet{rafailov2024direct,ye2024theoretical}), we adopt a contextual
bandit formulation of the alignment problem. We formalize the language model as
a \emph{policy} $\pi:\cX\to\Delta(\cA)$ which maps a context (prompt)
$x\in\cX$ to an action (response) $a\in\cA$ via
$a\sim{}\pi(\cdot\mid{}x)$, and let $\rho\in\Delta(\cX)$ denote the distribution over contexts/prompts.

\paragraph{Offline alignment}
In the offline alignment problem
\citep{christiano2017deep,bai2022training,ouyang2022training}, we
assume access to a dataset $\cD_\pref=\crl*{(x,\ap,\am)}$ of $n$ prompts
and labeled response pairs generated from a reference policy (language
model) $\piref$, which is typically obtained through
\arxiv{supervised fine tuning}. Here, $\ap$ is a positive action/response and $\am$ is a
negative action/response. Given the context/prompt $x\sim\rho$, the pair $(\ap,\am)$ is generated by sampling
a pair $(a,b)$ as $a\sim{}\piref(\cdot\mid{}x)$ and
$b\sim{}\piref(\cdot\mid{}x)$, and then ordering them
as $(\ap,\am)$ based on a binary preference
$y\sim{}\bbP(a\psdgt{}b\mid{}x)$. 
We assume
that preferences follow the \emph{Bradley-Terry} model
\citep{bradley1952rank}\arxiv{, in which}\loose
\begin{align}
\label{eq:bt}
\bbP(a\psdgt{}b\mid{}x) = \frac{\exp\prn*{\rstar(x,a)}}{\exp\prn*{\rstar(x,a)} + \exp\prn*{\rstar(x,b)}},
\end{align}
for an unknown reward function $\rstar:\cX\times\cA\to\brk*{0,\Rmax}$ for some $\Rmax \ge 1$. From the preference dataset $\cD_\pref$, we aim to learn a
policy $\pihat$ that has high reward in the sense that 
\arxiv{
  \begin{align*}
  J(\pistar) - J(\pihat) \leq \veps,
  \end{align*}
}
  for a small $\veps>0$, 
  where $J(\pi)\ldef{}\En_{x\sim\rho,a\sim\pi(\cdot\mid{}x)}\brk*{\rstar(x,a)}$ is the true expected reward, 
  and $\pistar$ is any comparator policy of interest. 
We abbreviate
$\En_{\pi}\brk*{\cdot}\ldef{}\En_{x\sim\rho,a\sim\pi(\cdot\mid{}x)}\brk{\cdot}$, 
and assume that $\rho(x)>0$ for all $x$ and 
$\piref(a \mid x) > 0$ for all $x,a$ without loss of generality. \loose

\paragraph{Offline RLHF with KL-regularization}
Classical algorithms for offline alignment \citep{christiano2017deep,ouyang2022training} are based on reinforcement learning with a
 \emph{KL-regularized} reward objective, defined for a regularization parameter $\beta>0$, via
\begin{align}
    \label{eq:kl_reward}
    J^{\mathsf{KL}}_\beta(\pi) \coloneqq &~ J(\pi) -
                             \beta\cdot{}\Dkl{\pi}{\piref}
                             = \E_{\pi} \left[\rstar(x,a) - \beta \log \frac{\pi(a\mid{}x)}{\piref(a\mid{}x)}\right],
\end{align}
where we adopt the shorthand $\Dkl{\pi}{\pi'}=\En_{x\sim\rho}\brk*{\Dkl{\pi(\cdot\mid{}x)}{\pi'(\cdot\mid{}x)}}$.
These methods first estimate a reward function $\wh{r}$ from $\cD_\pref$ using maximum
likelihood under the Bradley-Terry model:
\begin{equation}
  \label{eq:mle}
  \rhat=\argmax_{r\in\cR}\sum_{(x,\ap,\am)\in\cDpref}\log\sigma\prn*{r(\ap\mid{}x)-r(\am\mid{}x)},
  \end{equation}
where $\sigma(x) \ldef{}\frac{\exp(x)}{1+\exp(x)}$ is the sigmoid
function and $\cR$ is a class of reward functions, which is typically parameterized by a neural network. Then, they apply standard policy
optimization methods like \ppo to optimize an estimated version of \arxiv{the KL-regularized objective}: %
\arxiv{\begin{align*}
\pihat=\argmax_{\pi\in\Pi}\E_{\pi} \left[ \rhat(x,a) - \beta
                                        \log\frac{\pi(a \mid
                                        x)}{\piref(a \mid x)}
                                        \right].
       \end{align*}
       }
     The regularization term in \cref{eq:kl_reward} is intended to
     encourage $\pihat$ to stay close to $\piref$, with the hope of
     preventing the policy from overfitting to the potentially
     inaccurate reward model $\rhat$.

     \paragraph{Direct preference optimization (\dpo)}
\algshort is based on an alternative offline alignment approach, Direct Preference 
Optimization (\dpo; \citet{rafailov2024direct}). \dpo uses the
closed-form solution of the optimal KL-regularized policy under the
objective \cref{eq:kl_reward}---which can be viewed as implicitly
modeling rewards---to define a single policy optimization objective
that removes the need for direct reward function estimation.
Given a user specified policy class $\Pi$, \dpo solves
\begin{align}
  \label{eq:dpo}
  \pidpo=\argmax_{\pi\in\Pi}
\sum_{(x,\ap,\am) \in \Dcal_\pref} \log\left[\sigma\left(
  \beta\log\frac{\pi(\ap\mid{}x)}{\piref(\ap\mid x)} -
  \beta\log\frac{\pi(\am\mid x)}{\piref(\am \mid x)} \right) \right],
\end{align}
with the convention that the value of the objective is $-\infty$ if
  $\pi$ does not satisfy $\pi\ll{}\piref$. %

\subsection{Overoptimization and Insufficiency of KL-Regularization}

Empirically, both classical RLHF and direct alignment methods like \dpo have been observed
to suffer from overoptimization
\citep{gao2023scaling,guo2024direct,rafailov2024scaling,song2024understanding}, wherein
model quality degrades during the optimization process as the learned
policy drifts away from $\piref$. 
\arxiv{
The degree of degradation is affected by a number of factors, such as the objective used, the optimization landscape it induces, 
and the statistical properties of the algorithm. 
In this paper, we focus on mitigating the \emph{statistical problems} 
\akreplace{behind the empirical phenomena of}{underlying the} overoptimization phenomenon. 
As we will see, \akreplace{these phenomena are}{this phenomenon is} an issue of sample-inefficiency 
when offline data coverage is inadequate, 
which can be understood through the lens of
  coverage coefficients developed in the theory of offline
  reinforcement learning~\citep{liu2020provably,jin2021pessimism,rashidinejad2021bridging}.
}
\loose

\paragraph{Coverage coefficients}

  In offline reinforcement learning theory, the sample efficiency of an algorithm refers to the number of samples required to guarantee that $J(\pihat) \approx J(\pistar)$. 
  It is typically quantified by a \emph{coverage coefficient} (or
  concentrability coefficient)
  that measures the quality of the data collected by the reference $\piref$
  \citep{farahmand2010error,xie2020q,zanette2021provable}. 
  We will utilize the $L_1$ coverage coefficient, defined for a policy $\pi$ as $\Cone[\pi] \ldef{} \En_\pi\brk*{\frac{\pi(a\mid{}x)}{\piref(a\mid{}x)}}$.
  \emph{Single policy concentrability} is the gold standard for sample efficiency, and is obtained by an algorithm if, \emph{for any comparator policy $\pistar$}, the sample size required to learn $J(\pihat) \approx J(\pistar)$ scales with $\Cone[\pistar]$, the coverage coefficient of $\pistar$. 
  This guarantees that $\pihat$ is competitive with the best policy that is sufficiently covered by offline data, and, importantly, also guarantees that $\pihat$ is never much worse than $\piref$ itself.
Single policy concentrability is typically achieved by pessimistic algorithms that penalize the evaluations of candidate policies according to their uncertainty under the offline data, 
which prevents the learner from overfitting to inaccurate offline reward models.  

  In contrast, the performance of non-pessimistic algorithms typically scales with \emph{all-policy concentrability}---meaning that sample complexity scales with $\max_{\pi \in \Pi} \Cone[\pi]$ 
  \citep{liu2020provably,jin2021pessimism,rashidinejad2021bridging}--- 
  which is a guarantee achieved by even greedy algorithms
  that directly optimize the offline reward model without regularization.
  All-policy concentrability describes algorithms that cannot adapt to 
  the quality of the data, 
  and are thereby prone to overoptimization (of the offline reward model) unless 
  the data is rich enough to cover all candidate policies sufficiently well. 
  In contrast, single policy concentrability serves as a theoretical certification that an algorithm is robust to poor data coverage and will not overfit.

\loose

\paragraph{Pessimism in offline alignment}
\citet{zhu2023principled} show that the performance of \ppo and \dpo scales with all-policy concentrability, $\max_\pi \Cinf[\pi]$, 
for the stylized case of alignment with linearly parameterized policies where
$\pi_{\theta}(a\mid{}x)\propto{}\exp(\tri*{\phi(x,a),\theta})$ for a
known feature embedding $\phi(x,a)\in\bbR^{d}$ 
(see also
\citet{zhu2024iterative,song2024understanding}). 
They also propose a pessimistic algorithm that achieves %
\arxiv{
\[
J(\pistar) - J(\pihat)\approxleq{} \sqrt{\frac{\poly(\Cinf[\pistar],d)}{n}},
\] simultaneously for all $\pistar$.}\arxiv{\footnote{\citet{zhu2023principled} achieve guarantees based on a
  \emph{feature coverage} coefficient, which improves on concentrability coefficients, but is specialized to linear function approximation.}}
While encouraging, these results are restricted to
linearly parameterized policies, and cannot be directly applied to
large language models. Most existing theoretical algorithms for
offline alignment are similar in nature, and either place restrictive
assumptions on the policy class $\Pi$ \citep{zhu2023principled,zhan2023provable,li2023reinforcement,xiong2023gibbs}
or are not feasible to implement in a way that is faithful to theory
\citep{ye2024theoretical,ji2024selfplay}.

Most relevant to our work, a series of recent papers
\citep{liu2024provably,cen2024value,fisch2024robust} propose
implementing pessimism for general policy classes $\Pi$ by solving the
\arxiv{so-called }``\texttt{DPO}+\sft'' objective\loose
\begin{align}
  \label{eq:rpo}
\argmax_{\pi\in\Pi}\crl*{
  \alpha\cdot\En_{\piref}\brk*{\beta\log\pi(a\mid{}x)}
  + 
\frac{1}{n}\sum_{(x,\ap,\am) \in \Dcal_\pref} \log\left[\sigma\left(
  \beta\log\frac{\pi(\ap\mid x)}{\piref(\ap\mid x)} -
  \beta\log\frac{\pi(\am\mid x)}{\piref(\am\mid x)} \right) \right]
  },
\end{align}%
which augments the \dpo objective (the second term) with an additional supervised fine-tuning-like (\sft) loss (the first term). While this objective is simple to apply to general policy
classes, the existing single-policy concentrability guarantees for this method assume that $\Pi$ satisfies restrictive \emph{convexity} conditions which do not hold in practice for large language models. 
Perhaps surprisingly, we show
(\cref{sec:rpo_lower}) that without
convexity, \emph{the objective in \cref{eq:rpo} fails to achieve
  a single-policy concentrability guarantee}.\footnote{This finding is\arxiv{ rather}
  surprising because \citet{xie2024exploratory} show that an
  \emph{optimistic online} counterpart to \cref{eq:rpo}, which negates the SFT term, enjoys online RLHF guarantees\arxiv{ with
  general policy classes} without requiring analogous convexity conditions.\loose}
In other words, \texttt{DPO}+\sft is 
insufficient to mitigate overoptimization.

\section{\alglongb}
\label{sec:main}

This section presents our main algorithm, \algshort. We begin by
introducing \chis-regularization as a general framework for mitigating
overoptimization in offline alignment (\cref{sec:framework}), then derive
the \algshort algorithm (\cref{sec:algorithm}) and finally present our main theoretical
guarantee (\cref{sec:theoretical}).

\subsection{Framework: \chisb-Regularized Reward Optimization}
\label{sec:framework}
The central algorithm design
principle for our work is to (implicitly or explicitly) optimize a
variant of the classical RLHF objective (\cref{eq:kl_reward}) that
replaces KL-regularization with regularization
via \chis-divergence, defined for a pair of probability measures
$\bbP$ and $\bbQ$ with $\bbP\ll \bbQ$ via %
\arxiv{
\begin{align*}
    \Dchis{\bbP}{\bbQ}\ldef{}\frac{1}{2}\int\prn*{\frac{\mathrm{d}\bbP}{\mathrm{d}\bbQ}-1}^2\mathrm{d}\bbQ.  
\end{align*}}
\chis-divergence is a more aggressive form of regularization than
KL-divergence; we have $\Dkl{\bbP}{\bbQ}\leq{}2\Dchis{\bbP}{\bbQ}$,
but the converse is not true in general. We consider the following \chis-regularized RL objective:\footnote{
  Note the definition of $\Dchis{\pi}{\piref}$ differs from $\En[\Dchis{\pi(\cdot\mid{}x)}{\piref(\cdot\mid{}x)}]$ only by a constant scaling and shift, both of which are inconsequential when used as regularization in an optimization objective.%
  }\loose
  \begin{align}
    \label{eq:chis_reward}
    \Jchi(\pi)\ldef{}\En_{\pi}\brk*{\rstar(x,a)} - \beta\cdot\Dchis{\pi}{\piref}, \qquad \Dchis{\pi}{\piref} := \En_{\pi}\brk*{\frac{\pi(a\mid{}x)}{\piref(a \mid{} x)}}.
  \end{align}
Moving to a form of regularization that penalizes deviations from
$\piref$ more forcefully than KL-regularization is a natural approach to mitigating
overoptimization, but an immediate concern is that this may lead to
overly conservative algorithms. As we will show, however, \chis-divergence is
better suited to the geometry of offline alignment, as it has the unique
property (not shared by KL-divergence) that its value quantifies the extent to which the
accuracy of a
reward model $\rhat$ trained under $\piref$
will transfer to a downstream policy $\pi$ of interest
(\cref{lem:clip-dpo-estimation}). This implies that the
\chis-regularized RL objective in \cref{eq:chis_reward} meaningfully
implements a form of pessimism in the face of uncertainty, and by tuning the regularization
parameter $\beta>0$, we can keep the learned policy
$\pihat$ close to $\piref$ in the ``right'' (uncertainty-aware) way. 
As such, we view optimizing \chis-regularized rewards, i.e., $\argmax_{\pi \in \Pi}\Jchi(\pi)$
as a general principle to guide algorithm
design for offline alignment (as well as offline RL more
broadly), which we expect to find broader use.\loose

We now turn our attention to the matter of how to optimize
this objective. One natural approach, in the vein of classical
RLHF \arxiv{algorithms }\citep{christiano2017deep,ouyang2022training}, is
to estimate a reward model $\rhat$ using maximum likelihood (\cref{eq:mle}), and then
use PPO or other policy optimization methods to solve\loose
\begin{align}
\label{eq:rlhf}
\pihat=\argmax_{\pi\in\Pi}\E_{\pi} \left[ \rhat(x,a)
  \right]-\beta\cdot\Dchis{\pi}{\piref}
  = \argmax_{\pi\in\Pi}\E_{\pi} \left[ \rhat(x,a) -\beta\frac{\pi(a\mid{}x)}{\piref(a\mid{}x)}\right].
\end{align}
While this indeed leads to strong statistical guarantees
(cf. \cref{sec:rlhf}), we adopt a simpler and more direct approach
inspired by \dpo, which removes the need for a separate reward estimation
step.

  \subsection{The \algshortb Algorithm}
  \label{sec:algorithm}

\begin{algorithm}[tp]
\caption{\mainalg (\algshort)}
\label{alg:main}
\begin{adjustbox}{max width=\textwidth}
\begin{minipage}{\linewidth}
\begin{algorithmic}[1]
  \Statex[0] \mbox{{\bfseries input:}
    Reference policy $\piref$, preference dataset $\cD_\pref$,
    \chis-regularization coefficient $\beta>0$.}
  \State Define
  \begin{align}
    \label{eq:link}
    \link(z)\ldef{} z + \log{}z.
  \end{align}
    \State%
    Optimize \chis-regularized preference optimization objective:
    \begin{align}
      \label{eq:chi_dpo}
        \pihat \leftarrow \argmax_{\pi \in \Pi}
        \sum_{(x,\ap,\am) \in
        \Dcal_\pref}\log\left[\sigma\left(
        \clip_{2\Rmax}\brk*{\beta\link\prn*{\frac{\pi(\ap\mid x)}{\piref(\ap\mid x)}} -
        \beta\link\prn*{\frac{\pi(\am\mid x)}{\piref(\am\mid x)}}} \right) \right].
      \end{align}
     \label{line:chi_dpo}
    \State \textbf{return:}
    $\pihat$.
\end{algorithmic}
\end{minipage}
\end{adjustbox}
\end{algorithm}

Our main algorithm, \algshort, is described in \cref{alg:main}. Given
a preference dataset $\cD_\pref$ and\arxiv{ user-specified} policy class
$\Pi$, the algorithm learns a policy $\pihat$ by solving the \dpo-like
optimization objective \cref{eq:chi_dpo}, which replaces the usual
$\log\frac{\pi(a\mid{}x)}{\piref(a\mid{}x)}$ terms in the original \dpo objective
(\cref{eq:dpo}) with a new link function\arxiv{ given by}\loose
\[
  \link{}\prn*{\frac{\pi(a\mid{}x)}{\piref(a\mid{}x)}}
  =   \frac{\pi(a\mid{}x)}{\piref(a\mid{}x)}
  + \log\prn*{\frac{\pi(a\mid{}x)}{\piref(a\mid{}x)}}.
\]
A secondary modification is that we handle potentially unbounded
density ratios by clipping to the interval $\brk*{-2\Rmax,+2\Rmax}$ via the operator
$\clip_{R}(z)=\max\crl*{\min\crl*{R,z},-R}$. In what follows, we will show that this simple\arxiv{
and practical} modification to \dpo---that is, incorporating an additional density
ratio term outside the logarithm---implicitly
implements pessimism via \chis-regularization.

\paragraph{Algorithm derivation}
Recall that \dpo is derived \citep{rafailov2024direct} by observing that
the optimal KL-regularized policy
$\pistarkl\ldef\argmax_{\pi}\crl*{\En_{\pi}\brk*{\rstar(x,a)}-\beta\Dkl{\pi}{\piref}}$
 \arxiv{satisfies the following identity for all $x\in\cX$ and $a\in\cA$.
  \begin{align*}
    \rstar(x,a) = \beta\log\frac{\pistarkl(a\mid{}x)}{\piref(a\mid{}x)} + \Zklrstar(x),
  \end{align*}}
  where $\Zklrstar(x)$ is a normalization constant that depends on
  $x$ but not $a$. 
  This facilitates reparameterizing the reward model in the maximum
  likelihood estimation objective (\cref{eq:mle}) in terms of a
  learned policy, yielding the \dpo objective in \cref{eq:dpo}.\loose

  To apply a similar reparameterization trick for \chis-divergence,
  a natural starting point is an observation from
  \citet{wang2023beyond}, who show that an analogous characterization
  for the optimal regularized policy holds for a general class of
  \emph{$f$-divergences}. For a convex function
  $f:\bbR_{+}\to\bbR$, define the induced $f$-divergence by %
  \arxiv{
\[
  \Df{\bbP}{\bbQ}=\int{}f\prn*{\frac{\mathrm{d}\bbP}{\mathrm{d}\bbQ}}\mathrm{d}\bbQ
  =\En_{\bbQ}\brk*{f\prn*{\frac{\mathrm{d}\bbP}{\mathrm{d}\bbQ}}}.
\]}
\citet{wang2023beyond} show that for any differentiable $f$ that
satisfies the technical condition $0\notin\dom(f')$, the optimal
$f$-regularized policy
$\pistar_{\beta;f}=\argmax_{\pi}\crl*{\En_{\pi}\brk*{\rstar(x,a)}-\beta\Df{\pi}{\piref}}$
satisfies
  \begin{equation}
    \label{eq:f_opt}
    \rstar(x,a) = \beta{}f'\prn*{\frac{\pistar_{\beta;f}(a\mid{}x)}{\piref(a\mid{}x)}} + \Zfrstar(x)
  \end{equation}
  for a normalization constant $\Zfrstar(x)$, allowing for a similar
  reparameterization. Informally, the condition $0\notin\dom(f')$ means that
  $\Df{\cdot}{\piref}$ acts as a \emph{barrier} for the positive
  orthant, automatically forcing $\pistar_{\beta;f}$ to place positive
  probability mass on any action $a$ for which
  $\piref(a\mid{}x)>0$. 

  The \chis-divergence is an $f$-divergence
  corresponding to $f(z)=\frac{1}{2}(z-1)^2$, but unfortunately does
  not satisfy the condition $0\notin\dom(f')$, making \cref{eq:f_opt}
  inapplicable. Indeed, the optimal \chis-regularized policy can clip action
  probabilities to zero in a non-smooth fashion even when
  $\piref(a\mid{}x)>0$, which means that the identity \cref{eq:f_opt}
  does not apply. To address this issue, we augment
  \chis-regularization by considering the \emph{mixed \chis-divergence} given by
  $\fmix(z) \ldef{} \frac{1}{2}(z-1)^2 + z\log{}z$, which has
\[
\Dfmix{\bbP}{\bbQ} = \Dchis{\bbP}{\bbQ} + \Dkl{\bbP}{\bbQ}.
\]
In other words, we use \emph{both \chis-regularization and
  KL-regularization}; \chis-regularization enforces pessimism, while
KL-regularization enforces the barrier property and facilitates
reparameterization. Indeed, the link function
$\link$ (\cref{eq:link}) used in \algshort has
$\link(z)\ldef{}\fmix'(z) = z + \log{}z$, which satisfies
$0\notin{}\dom(\fmix')$, so \cref{eq:f_opt} yields the
reparameterization
$\rstar(x,a) = \beta\link\prn[\Big]{\frac{\pistar_{\beta;\fmix}(a\mid{}x)}{\piref(a\mid{}x)}} + \Zfmixr[\rstar](x)$.
Substituting this identity into the maximum likelihood estimation
objective (\cref{eq:mle}) yields the \algshort algorithm.\loose

Going forward, we define $\Jmixr[r](\pi) = \En_{\pi}\brk*{r(x,a)} -
\beta\cdot\Dchis{\pi}{\piref}-\beta\cdot\Dkl{\pi}{\piref}$ for a
reward function $r$. We use the shorthand
$\pistarb=\argmax_{\pi}\Jmixr[\rstar](\pi)$ as the optimal policy under
mixed \chis-regularization, and abbreviate $\Zr(x)\ldef{}\Zfmixr(x)$,
so that \loose
\begin{equation}
  \label{eq:chis_opt}
\rstar(x,a) = \beta\link\prn*{\frac{\pistarb(a\mid{}x)}{\piref(a\mid{}x)}} + \Zrstar(x).
\end{equation}

  \subsection{Theoretical Guarantees}
  \label{sec:theoretical}

To state our main sample complexity guarantee for \algshort, we begin
by making standard statistical assumptions. Let the regularization
parameter $\beta>0$ in \algshort be fixed. We first make a
\emph{realizability} assumption, which states that the policy class $\Pi$ used in
\algshort is sufficiently expressive to represent the optimal policy
under mixed \chis-regularization (\cref{eq:chis_opt}); recall that in
the context of language modeling, $\Pi$ represents a class of language
models with fixed architecture and varying weights.
\begin{assumption}[Policy realizability]
  \label{ass:realizability}
  The policy class $\Pi$ satisfies $\pistarb\in\Pi$, where $\pistarb$
  is the optimal policy under mixed \chis-regularization (\cref{eq:chis_opt}).
\end{assumption}
Policy realizability is a standard assumption for sample-efficient
reinforcement learning
\citep{agarwal2019reinforcement,lattimore2020bandit,foster2023foundations},
and is equivalent to reward model realizability in our setting via reparameterization.
\arxiv{

}
\arxiv{Our }second assumption asserts that the implicit reward models
induced by the policy class $\Pi$ in \algshort have bounded range.
\begin{assumption}[Bounded implicit rewards]
\label{ass:vmax}
For a parameter $\Vmax\geq{}\Rmax$, it holds that for all $\pi\in\Pi$, $x\in\cX$, and $a,b\in\cA$, %
\arxiv{
\begin{equation*}
\abs*{\beta\link\prn*{\frac{\pi(a\mid{}x)}{\piref(a\mid{}x)}}-\beta\link\prn*{\frac{\pi(b\mid{}x)}{\piref(b\mid{}x)}}} \leq \Vmax.
\end{equation*}}
\end{assumption}
  \cref{ass:vmax} generalizes analogous assumptions made in the
  analysis of \dpo-like algorithms in prior work
  \citep{rosset2024direct,xie2024exploratory}, and our guarantees scale polynomially with this parameter; see \cref{sec:vmax_appendix}
  for a detailed comparison. We emphasize that in practice, $\Vmax$
  can be measured and directly controlled (e.g., via clipping).
  \begin{example}[Policy classes induced by reward models]
    \label{ex:reward_model}
    A natural setting in which both \cref{ass:realizability} and \cref{ass:vmax} hold  is when
    the policy class $\Pi$ is induced by a class of bounded reward function $\cR \subset (\cX\times\cA \to [0,\Rmax])$
    through the mixed-\chis{} parameterization, for $\beta >0$:
    \begin{equation}
      \label{eq:reward_policy}
      \Pi_{\cR,\beta}\ldef{}
      \crl*{\pi(a\mid{}x)=\piref(a\mid{}x)\cdot\link^{-1}(\beta^{-1}(r(x,a)-\Zr[r](x)))\mid{}r\in\cR}.
    \end{equation}
    Here, \cref{ass:realizability} holds whenever $\rstar\in\cR$, and \cref{ass:vmax} \arxiv{is
    satisfied} with $\Vmax\leq{}2\Rmax$. \loose %
\end{example}

Finally, 
recall the definition of the $L_1$ concentrability coefficient, $\Cone \ldef{}\En_{\pi}\brk*{\frac{\pi(a\mid x)}{\piref(a \mid x)}}$, which %
is equivalent to the \chis-divergence up to a constant shift,
i.e., $\Cone=1+2\Dchis{\pi}{\piref}$.
We  use $L_1$ concentrability to quantify 
\ahreplace{the coverage of
a policy $\pi$ by the offline preference dataset
$\cD_{\pref}$ generated by $\piref$.}{
  how well the offline preference dataset $\cD_{\pref}$, 
  generated by $\piref$, covers a policy $\pi$, 
  and the following result is our main sample complexity guarantee for \algshort. 
}
\loose
\begin{theorem}[Sample complexity bound for \algshort]
  \label{thm:main}
  Suppose \cref{ass:realizability,ass:vmax} hold for some $\beta>0$. 
  With probability at
  least $1-\delta$, \algshort (\cref{alg:main}) produces a policy $\wh\pi$ such
  that for all policies $\pistar$ simultaneously, we have
  \begin{align}
    \label{eq:main1}
    J(\pistar) - J(\pihat) \approxleq \Vmax e^{2\Rmax}\cdot\sqrt{\frac{\Cone[\pistar]\log(|\Pi|/\delta)}{n}} + \beta\cdot\Cone[\pistar] + \beta^{-1}\cdot\frac{\Vmax^2 e^{4\Rmax}\log(|\Pi|/\delta)}{n} .  
  \end{align} 
\arxiv{  In particular, given} any comparator policy $\pistar$, we can choose
  the regularization parameter $\beta$ to achieve \loose
  \begin{align}
    \label{eq:main2}
    J(\pistar) - J(\pihat) \approxleq \Vmax e^{2\Rmax}\cdot\sqrt{\frac{\Cone[\pistar]\log(|\Pi|/\delta)}{n}}. 
  \end{align}
\end{theorem}
\cref{thm:main} shows that \algshort achieves a sample complexity
guarantee that scales only with the single-policy concentrability
parameter $\cC^{\pistar}$ for the comparator policy $\pistar$, for all
policies $\pistar$ simultaneously. In particular, roughly
$n=\bigoh\prn*{\frac{\cC^{\pistar}\log(\abs{\Pi}/\delta)}{\veps^2}}$
examples are sufficient to learn a policy that is $\veps$-suboptimal
relative to $\pistar$. 
As a result, \algshort is robust to overoptimization
since the learned policy is as good as any $\pistar$ that is sufficiently covered by $\piref$ (in the sense that $\cC^{\pistar}=\bigoh(1)$),
which is effectively the best one can hope for in the purely offline
setting. In contrast, naive offline alignment methods like \dpo have sample
complexity that scales with
\emph{all-policy concentrability} (roughly, $\max_{\pi}\cC^{\pi}$),
even when the comparator policy $\pistar$ is sufficiently covered
\citep{zhu2023principled,song2024understanding}. \arxiv{

}
To highlight this, in
\cref{fig:regret} (see \cref{sec:understanding} for details) we give a concrete example in which \algshort allows the user to tune $\beta$ to achieve tight statistical
rates, yet no choice of $\beta$ for \dpo leads to comparable
performance. Effectively, any choice of $\beta$ for \dpo is
either susceptible to overoptimization, or \akreplace{has unacceptably high
bias}{is unacceptably conservative}. All prior
works that achieve similar sample complexity guarantees based on
single-policy concentrability are either impractical, or require more
restrictive statistical assumptions on the policy class
\citep{ye2024theoretical,liu2024provably,cen2024value,fisch2024robust,ji2024selfplay}.\footnote{A
  notable difference is that some
  of these works achieve guarantees based on tighter coverage parameters than single-policy $L_1$-concentrability, which reflect the
  structure of the policy or reward class. This appears to be
  out of reach for our techniques.\loose}\loose

Regarding the parameter $\Vmax$, we
observe that since the policy $\pistarb$ satisfies
$\abs*{\beta\link\prn*{\frac{\pistarb(a\mid{}x)}{\piref(a\mid{}x)}}-\beta\link\prn*{\frac{\pistarb(b\mid{}x)}{\piref(b\mid{}x)}}}
\leq 2\Rmax$, information-theoretically we can always achieve
$\Vmax=2\Rmax$ by pre-filtering the policy class $\Pi$ to remove all
policies for which this inequality does not hold. 
Since this may be non-trivial\arxiv{ in practice}, we incorporate clipping in \cref{eq:chi_dpo}, which, for precisely the above reason, we expect to improve performance\arxiv{ empirically}. 
In \arxiv{\cref{sec:understanding}}, we discuss the role of the $\Vmax$ parameter and
\cref{ass:vmax} in greater depth. See also the guarantees for the \rlhfalg
algorithm in \cref{sec:rlhf}, which avoid dependence on this parameter.

\paragraph{Tuning the regularization parameter}
To achieve optimal dependence on $\cC^{\pistar}$,
\cref{thm:main} requires tuning $\beta>0$ as a function of this
parameter, similar to other pessimistic schemes
\citep{liu2024provably}. With no prior knowledge, setting
$\beta\propto{}\sqrt{\frac{\Vmax^2 e^{4\Rmax}\log(|\Pi|/\delta)}{n}}$
suffices to ensure that, simultaneously for all comparator policies $\pistar$, we have %
 \arxiv{
 \begin{align*}
    J(\pistar) - J(\pihat) \approxleq \Vmax e^{2\Rmax}\cdot\sqrt{\frac{(\Cone[\pistar])^2\log(|\cR|/\delta)}{n}}, 
 \end{align*}}
This guarantee achieves a slightly worse
rate than~\cref{eq:main2} but holds simultaneously for all comparator
policies rather than the specific one that was used to tune $\beta$.
The following result, specializing to the setting in
\cref{ex:reward_model}, shows that there exists an optimal parameter
$\betastar>0$ that recovers the rate in~\cref{eq:main2} and holds simultaneously for all comparator policies. 
  \begin{corollary}[Sample complexity bound for \algshort with a
    reward model]
    \label{cor:reward_model}
    Consider the setting in \cref{ex:reward_model}, where
    the policy class $\Pi_{\cR,\beta}$ is the set of mixed \chis-regularized
    policies induced by a reward model class $\cR$ with $\rstar\in\cR$ and
    $\beta>0$. 
    For any $\delta \in (0,1)$, there exists a choice\footnote{It is
      unclear how to select $\betastar$ in a data-driven manner, as it
      depends on the\arxiv{ (unknown)} functionals $\pi \mapsto
      C^\pi$\arxiv{ and }$\pi \mapsto J(\pi)$.\loose} for
    $\betastar>0$
such that with probability at
  least $1-\delta$, \algshort (\cref{alg:main}), with class $\Pi_{\cR,\betastar}$, produces a policy $\wh\pi$ such
  that for all policies $\pistar$ simultaneously, we have \loose
\arxiv{  \begin{align*}
    J(\pistar) - J(\pihat) \approxleq \Rmax e^{2\Rmax}\cdot\sqrt{\frac{\Cone[\pistar]\log(|\cR|/\delta)}{n}}. 
         \end{align*}
         }
\end{corollary}
\arxiv{
\paragraph{Additional remarks}
Specializing to the case of
multi-armed bandits, we believe that the sample complexity bound in
\cref{eq:main1} is optimal in general
\citep{rashidinejad2021bridging}. Note that while we consider finite
classes $\Pi$ in \cref{thm:main} for simplicity, extension to infinite
classes is trivial via standard uniform convergence arguments. We remark that the exponential dependence on
$\Rmax$ in \cref{thm:main} is an intrinsic feature of the
Bradley-Terry model, and can be found in all prior work
\citep{rosset2024direct,xie2024exploratory}.   Finally, we remark that the weight on the KL term in the
  mixed \chis-regularized objective is not important for our
  statistical guarantees. For any $\gamma\in(0,1]$, we can replace the
  link function $\phi(\cdot)$ in \algshort
with $\linkg(z) = z + \gamma\log{}z$,
\[
  \linkg(z) = z + \gamma\log{}z,
\]
which corresponds to the regularized objective $\Jmixg(\pi) = \En_{\pi}\brk*{\rstar(x,a)} -
\beta\cdot\Dchis{\pi}{\piref}-\gamma\beta\cdot\Dkl{\pi}{\piref}$.
This leads to identical guarantees for any $\gamma\in(0,1]$ 
(\cref{thm:main-mix} in \cref{sec:proofs_main}); essentially, we only require
that $\gamma$ is positive to ensure the reparameterization in
\cref{eq:f_opt} is admissible. }

\arxiv{
\section{Understanding \algshortb: The Bias-Overoptimization Tradeoff}
\label{sec:understanding}
Having derived \algshort from the mixed \chis-regularized RLHF
objective and analyzed its performance, we now take a moment to better understand the statistical properties of
the policies the algorithm learns. We focus on the tradeoff
between overoptimization and bias (i.e., underoptimization) achieved by the regularization parameter $\beta>0$, highlighting through examples how this leads to statistical benefits over naive alignment methods like \dpo.

\subsection{Properties of Optimal Policy under Mixed \chisb-Regularization}

We begin by deriving a (nearly) closed form solution for the optimal mixed
\chis-regularized policy in \cref{eq:chis_opt}; recall that we expect \algshort to converge to this policy in the limit of infinite data.

We first observe 
  that the link function $\link(\cdot)$ is strictly increasing over $\bbR_{+}$,
  and its inverse is given by $\link^{-1}(z)=W_0(\exp(z))$; here, $W_0(y)$
  denotes the Lambert W-function
  \citep{corless1996lambert}, defined for $y\geq{}-e^{-1}$ as the inverse of the
    function $x\mapsto{}xe^{x}$.
    Consequently, for any $x$, the
  optimal policy under mixed \chis-regularization satisfies
  \begin{align*}
    \pistarb(a\mid{}x)=\piref(a\mid{}x)\cdot{}W_0\prn*{\exp\prn*{\beta^{-1}(\rstar(x,a)-\Zrstar(x))}},
\end{align*}
where $\Zrstar(x)$ is chosen such that $\sum_{a}\pistarb(a\mid{}x)=1$. We can
better understand how this policy behaves using the following simple upper and lower bounds on the inverse link function $\link^{-1}(z)=W_0(\exp(z))$.\loose
\begin{proposition}
  \label{prop:link}
  The link function $\link(z)=z+\log{}z$ is strictly increasing over
  $(0,\infty)$, and its inverse
  $\link^{-1}(z)=W_0(\exp(z))$ is strictly increasing over $(-\infty,\infty)$. The inverse
  link function $\link^{-1}$ satisfies
  \begin{align*}
    \frac{z}{2}\leq \link^{-1}(z) \leq{} z\quad\forall{}z\in[1,\infty),\mathand\quad
    e^{z-e}\leq \link^{-1}(z) \leq{} e^{z}\quad\forall{}z\in(-\infty,1].
  \end{align*}
\end{proposition}

Compared to KL-regularization, which leads to softmax policies that
satisfy $\pistarkl(a\mid{}x)=\piref(a\mid{}x)\cdot\exp\prn*{\beta^{-1}(\rstar(x,a)-\Zklr[\rstar](x))}$, we see
that the inverse link function $\link^{-1}(z)=W_0(\exp(z))$ for mixed
\chis-regularization satisfies $\link^{-1}(z)\approx{}z$
for $z\geq{}1$, leading to a more heavy-tailed action distribution for
$\pistarb$. On the other hand, for $z\leq{}1$ the inverse link behaves like the
exponential function (i.e., $\link^{-1}(z)\approx{}e^{z}$ for $z\leq{}1$); see
\cref{fig:link} for an illustration. Using these properties, we can
derive the following upper and lower bounds on the density ratio
between $\pistarb$ and $\piref$.
\arxiv{\begin{proposition}}
  \label{prop:conc_bounds}
  For all $x\in\cX$ and $a\in\cA$, the optimal policy $\pistarb$ under mixed \chis-regularization
  satisfies
  \begin{align}
    \label{eq:conc_bounds}
      \exp\prn*{-\frac{\Rmax}{\beta}}\approxleq \frac{\pistarb(a\mid{}x)}{\piref(a\mid{}x)}\approxleq 1 + \frac{\Rmax}{\beta}.
    \end{align}   
Both inequalities are tight in general (up to absolute constants).  
\end{proposition}
The upper bound in \cref{eq:conc_bounds}, which arises from the
\chis{} term in the mixed-\chis{} objective, scales inversely with the regularization
parameter $\beta$, and reflects the heavy-tailed, pessimistic behavior this regularizer induces; in contrast, the optimal policy under pure
KL-regularization only satisfies
\begin{align}
  \label{eq:kl_conc}
\exp\prn*{-\frac{\Rmax}{\beta}}\approxleq\frac{\pistarkl(a\mid{}x)}{\piref(a\mid{}x)}\approxleq{}\exp\prn*{\frac{\Rmax}{\beta}}
\end{align}
in general. The lower bound in \cref{eq:conc_bounds} arises from the
KL term in the mixed-\chis{} objective, but is not important for our
analysis (outside of allowing for \dpo-like reparameterization).

\begin{figure}[tp]
  \centering
  \arxiv{\includegraphics[scale=0.6]{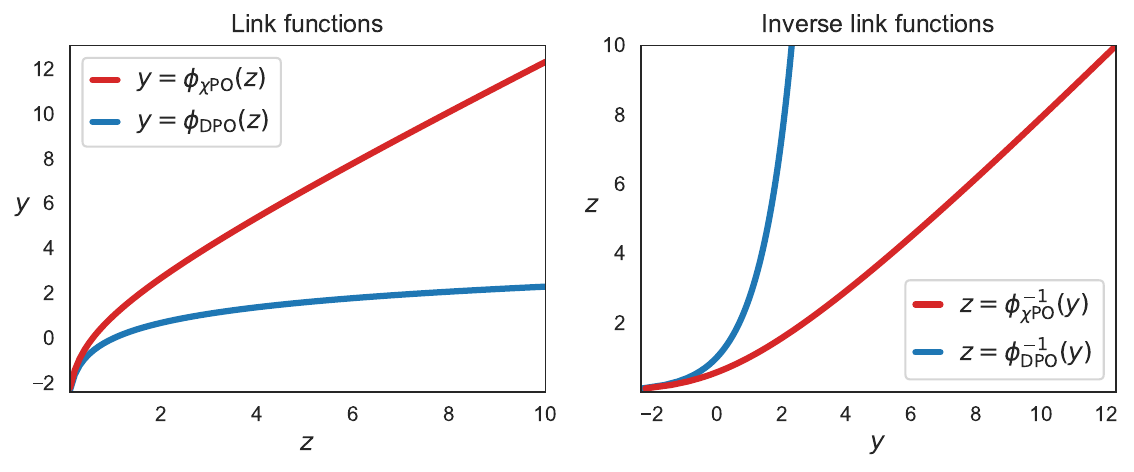}}
  \caption{
  Behavior of the mixed \chis-regularization link function
    $\link_{\algshort}(z)=z+\log{}z$ and inverse
    $\link_{\algshort}^{-1}(z)=W_0(\exp(z))$, compared to the KL-regularization
    link function $\link_{\dpo}(z)=\log{}z$ and inverse
    $\link_{\dpo}^{-1}(z)=\exp(z)$. $\link_{\algshort}^{-1}(z)\approx z$ for $z \geq 1$, leading to favorable heavy-tailed, pessimistic behavior.
  }
    \label{fig:link}
\end{figure}

\subsection{The Bias-Overoptimization Tradeoff}
\label{sec:tradeoff_appendix}

We are now well equipped to understand how \algshort modulates the
tradeoff between overoptimization and bias using the regularization
parameter $\beta$, and how this tradeoff compares to vanilla \dpo.
To showcase this, we take a reward modeling perspective, and consider
the setting in which the policy class $\Pi$ is induced by a given
reward model class $\cR$, similar to~\cref{ex:reward_model}.

Suppose
we start with a reward model class
$\cR\subset(\cX\times\cA\to\brk{0,\Rmax})$ such that $\rstar\in\cR$. If we
use the induced policy class
\begin{align}
  \label{eq:pidpo}
\Pidpo\ldef{}\crl*{\pi(a\mid{}x)=\piref(a\mid{}x)\cdot\exp(\beta^{-1}(r(x,a)-\Zklr[r](x)))\mid{}r\in\cR},
\end{align}
then \dpo can be interpreted as fitting a reward model $\rhat$ using
maximum likelihood (\cref{eq:mle})
and then outputting the policy
$\pihat_{\textsf{DPO}}(a\mid{}x)=\piref(a\mid{}x)\cdot\exp(\beta^{-1}(\rhat(x,a)-\Zklr[\rhat](x)))$. Meanwhile,
if we use the induced policy class
\begin{align}
  \label{eq:pichipo}
\Pichipo \ldef{}\crl*{\pi(a\mid{}x)=\piref(a\mid{}x)\cdot\link^{-1}(\beta^{-1}(r(x,a)-\Zr[r](x)))\mid{}r\in\cR},
\end{align}
then \algshort can be interpreted as fitting a reward model $\rhat$
with the exact same maximum likelihood objective, but instead
outputting the policy
$\pihat_{\textsf{\algshort}}(a\mid{}x)=\piref(a\mid{}x)\cdot\link^{-1}(\beta^{-1}(\rhat(x,a)-\Zr[\rhat](x)))$.  

The policies $\pihat_{\textsf{\algshort}}$ and
$\pihat_{\textsf{\dpo}}$ are induced by the same reward model
$\rhat$, and both use the parameter $\beta$ to balance bias and overoptimization. For both policies, large $\beta$ means the policy avoids
overfitting to errors in the reward model (the extreme case is
$\beta\to\infty$, in which case both policies become $\piref$), while
small $\beta$ means the policy has low \emph{bias}, i.e., low error in
the case where the model is correct in the sense that $\rhat=\rstar$ (the extreme case
is $\beta\to{}0$, in which case both policies become $x \mapsto \argmax_{a: \piref(a\mid{}x)>0} \rhat(x,a)$). Yet, for the same choice of $\beta$, $\pihat_{\textsf{\algshort}}$ is significantly more heavy-tailed than $\pihat_{\textsf{DPO}}$, a consequence of the pessimism induced by \chis-regularization; see \cref{fig:action_distribution}, which plots the
action distribution for both policies as a
function of $\beta$. 

\begin{figure}[tp]
  \centering
  \arxiv{\includegraphics[scale=.7]{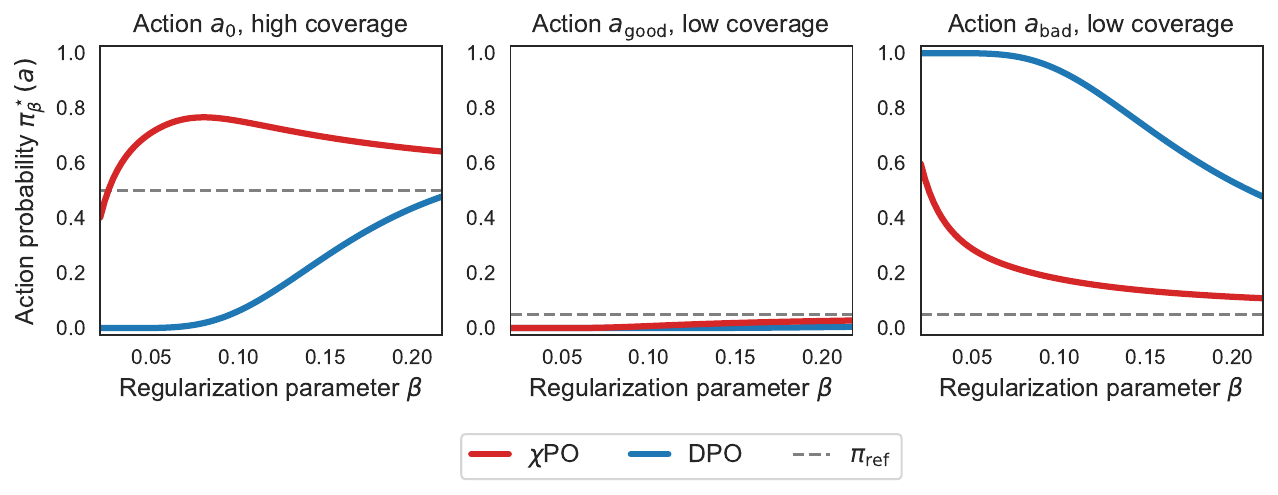}}
    \caption{
    Action probabilities for policies learned by $\algshort$ and $\dpo$ on the example from \cref{sec:illustrative}, under the ``bad'' event $\cE$ in which the true reward model is $\rstar=r_1$ but the estimated reward model is $\rhat = r_2$ ($n = 10$). 
      Here, $\rstar(a_{\mathsf{good}}) = 1$ and $\rstar(a_{\mathsf{bad}}) = 0$,  but $\rhat(a_{\mathsf{good}}) = 0$
      and $\rhat(a_{\mathsf{good}}) = 1$; both reward functions have $\rstar(a_0)=\rhat(a_0)=1/2$, and the goal is to compete with a comparator policy that deterministically plays $a_0$.\\
    \textbf{Overoptimization.} The \dpo policy is greedier with respect to the incorrect reward model and places much larger mass on the bad action $a_{\mathsf{bad}}$ for all $\beta \in (0, \frac{1}{2\log n}]$ (Right).  
   As a result, the \dpo policy places much smaller mass on the baseline action $a_0$, suffering significantly more  
    overoptimization error compared to \algshort (Left; see also \cref{fig:regret}).  \\
    \textbf{Bias.} Compared to \dpo, \algshort has a higher probability of taking both the optimal action $a_{\mathsf{good}}$ and the reference action $a_0$. %
    As a result, 
    it strikes a better bias-overoptimization tradeoff than \dpo,  
    and is competitive with respect to the comparator $a_0$ even when \dpo fails to converge.  
    }
  \label{fig:action_distribution}
\end{figure}

\subsection{An Illustrative Example}
\label{sec:illustrative}

We now give a concrete example in which \algshort allows the user to tune $\beta$ to achieve tight statistical
rates, yet no choice of $\beta$ for \dpo leads to comparable
performance (effectively, any choice of $\beta$ is
either susceptible to overoptimization, or has unacceptably high bias). This illustrates the favorable tradeoff between bias and overoptimization achieved by \algshort.\loose

  \label{ex:chipo_dpo}
      Let $n\in\bbN$ with $n\geq{}2$ be given. We consider a problem
    instance with $\cX=\crl{\emptyset}$ and $\cA=\crl*{a_0,a_1,a_2,a_3}$. 
    We define $\piref$ via 
    \begin{align*}
      \piref(a_0) = \tfrac{1}{2}, 
      \quad 
      \piref(a_1) = \piref(a_2) = \tfrac{1}{2n},
      \mathand 
      \piref(a_3) = \tfrac{n-2}{2n}. 
    \end{align*}
    We define a reward class with two
    reward functions $\cR:= \{r_1,r_2\}$ as follows. 
For $i\in\crl{1,2}$:
\begin{align*}
  &r_i(a_0) = 1/2, 
  \quad 
  r_i(a_i) = 1, 
  \quad 
  r_i(a_j) = 0, \;\; \forall j \ne i.
\end{align*}
\loose 

Let $\beta>0$ be fixed. To compare \algshort and \dpo, we consider their behavior when invoked with the induced policy classes $\Pichipo$ and $\Pidpo$ defined above.
Recall that with this choice, the two algorithms can
be interpreted as fitting a reward model $\rhat$ using maximum
likelihood (\cref{eq:mle}) and returning the policies
$\pihat_{\textsf{\algshort}}(a\mid{}x)=\piref(a\mid{}x)\cdot\link^{-1}(\beta^{-1}(\rhat(x,a)-\Zr[\rhat](x)))$
and
$\pihat_{\textsf{DPO}}(a\mid{}x)=\piref(a\mid{}x)\cdot\exp(\beta^{-1}(\rhat(x,a)-\Zklr[\rhat](x)))$,
respectively. \loose

Suppose that $r_1$ is the true reward function. It is hopeless
(information-theoretically) to compete with the
unconstrained optimal action $a_1$, as we are in a
sample-starved regime where $\cC^{a_1}=2n$ (in
the language of \cref{eq:main1}). Indeed, one can show
(see proof of \cref{prop:rpo_lower} in \cref{sec:related}) that with
constant probability, none of the examples in the offline dataset
$\cDpref$ contain actions $a_1$ or $a_2$. Under this event, which we denote by $\cE$, the value for
the maximum likelihood objective in \cref{eq:mle} is identical for $r_1$ and
$r_2$, so we may obtain $\rhat=r_2$ (due to adversarial tie-breaking). However, in spite of the fact that the policies $\pihat_{\textsf{\algshort}}$ and $\pihat_{\textsf{DPO}}$ are induced by the same (incorrect) reward function $\rhat=r_2$, they produce very different action distributions, as highlighted in \cref{fig:action_distribution}. %
\loose

\arxiv{
\begin{figure}[tp]
  \centering
  \arxiv{\includegraphics[scale=0.8]{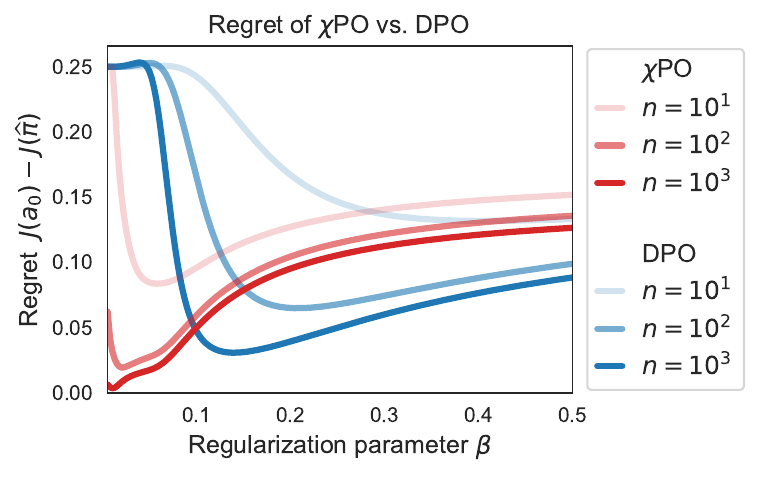}}
    \caption{
    The regret $J(a_0)-J(\pihat)$ 
    in the construction from \cref{prop:rpo_lower}
    for different values of $n$. 
    We again condition on the ``bad'' event $\cE$ 
    where $\rhat = r_2 \neq r^\star$. 
    For each $n$, 
    the error from overoptimization  
    dominates when $\beta \le (2\log n)^{-1}$ 
    (as discussed in \cref{sec:illustrative}), and
    the error from bias dominates 
    when $\beta > (2\log n)^{-1}$.  
    Taking the best choice of $\beta$ for each method, 
    \dpo converges at an exponentially slower rate than \algshort.
        }
  \label{fig:regret}
\end{figure}
}

To understand this, note that even in the sample-starved regime, we can still hope to compete with the ``baseline'' action
$a_0$; \cref{fig:regret} shows that \algshort has low regret against this action, while \dpo has high regret. In particular, since $\cC^{a_0}=2$, \cref{thm:main} (\cref{eq:main1}) implies
that \algshort achieves
\[
  J(a_0) - J(\pihat_{\textsf{\algshort}})
  \approxleq{} \sqrt{\frac{1}{n}} + \beta + \beta^{-1}\frac{1}{n},
\]
and setting $\beta\propto \sqrt{\frac{1}{n}}$ leads to $J(a_0) - J(\pihat_{\textsf{\algshort}})
\approxleq{} \sqrt{\frac{1}{n}}$. This is a consequence of the pessimistic, heavy-tailed nature
of $\pihat_{\textsf{\algshort}}$
(cf. \cref{prop:conc_bounds}), which places no more than
$\beta^{-1}/n$ probability mass on the (incorrect) greedy action $a_2$
for $\rhat=r_2$, thereby correctly capturing the inherent uncertainty
in the reward for this action.

On the other hand, it is straightforward to show that for all possible
values $\beta\leq{}(2\log{}n)^{-1}$, the \dpo policy
$\pihat_{\textsf{\dpo}}$ has regret
\[
  J(a_0) - J(\pihat_{\textsf{\dpo}})
  \geq{} 
  \frac{1}{2}\prn*{1 -
    \frac{1}{1+\frac{1}{n}e^{\frac{1}{2}} +
      (1-\frac{1}{n})e^{-\frac{1}{2\beta}}}}
    -\frac{1}{2n} \geq \bigom(1)
  \]
  whenever $n\geq{}2$. 
  This is because when $\beta\leq(2\log{}n)^{-1}$,
  $\pihat_{\textsf{\dpo}}$ assigns excessively high probability to the
  incorrect greedy action $a_2$, an instance of
  overoptimization. Meanwhile, larger choices for $\beta$ lead to
  excessively large bias in general (see \cref{sec:rpo_lower} for a more sophisticated construction which extends this lower bound to all possible $\beta$). In
  other words, as illustrated in \cref{fig:regret}, no choice of $\beta$ gives a favorable tradeoff between
  overoptimization and bias.

  To summarize, for \dpo, large values of $\beta$
are required to avoid overfitting to the reward function, incurring
high bias. Meanwhile, \algshort avoids overoptimization using comparatively small
values for $\beta$, yet has bias no worse than that of \dpo, thereby striking a better tradeoff.
  We mention that the ``\texttt{DPO}+\sft'' algorithm of
\citet{liu2024provably,cen2024value,fisch2024robust} also fails on the
construction above; see \cref{prop:rpo_lower} in \cref{sec:rpo_lower}
for details.

\begin{remark}[\dpo decreases probabilities of preferred and rejected responses]
  Various recent works have noted an empirical phenomenon in which \dpo decreases the probabilities for both preferred and rejected responses throughout training \citep{yuan2024advancing,pal2024smaug,rafailov2024r}. Interestingly, we observe that the example above exhibits this phenomenon. Notably, if $\beta < (2\log{}n)^{-1}$, then under the event $\cE$ in which the offline dataset $\cD_{\pref}$ does not contain the actions $a_1$ or $a_2$ (so that $\rhat=r_2$), we observe that
  $
    \pidpo(a_0)
    = \frac{\frac{1}{2}e^{\frac{1}{2\beta}}}{\frac{1}{2}e^{\frac{1}{2\beta}}
      + \frac{1}{2n}e^{\frac{1}{\beta}} + \frac{n-1}{2n}
      } < 
    \frac{1}{2} = \piref(a_0)$,
    and for all $i>2$,
    $
      \pidpo(a_i)
          = \frac{\frac{1}{2n}}{\frac{1}{2}e^{\frac{1}{2\beta}}
            + \frac{1}{2n}e^{\frac{1}{\beta}} + \frac{n-1}{2n}
            } < \frac{1}{2n} = \piref(a_i)$.  We conclude that for all $a\in\cD_{\pref}$,
          \[
            \pidpo(a) < \piref(a).
          \]
We emphasize that this behavior arises due to the use of function approximation. When the reward class $\cR$ (equivalently, the policy class $\Pidpo$) is restricted, the algorithm can aggressively (and incorrectly) extrapolate rewards for actions outside the dataset and, in doing so, inadvertently decrease the probabilities for preferred responses in the dataset. Meanwhile, in the same parameter range, \algshort satisfies (see \cref{fig:action_distribution})
          \[
            \pichipo(a_0) > \piref(a_0),
          \]
          highlighting that pessimism can mitigate this phenomenon.

\end{remark}

\subsection{Nontriviality and Role of $\Vmax$ Parameter}
\label{sec:vmax_appendix}

To close this section, we discuss the role of the $\Vmax$ parameter (\cref{ass:vmax}) used in the analysis of \algshort (\cref{thm:main}) in depth, motivating it from the perspective of the induced policy class $\Pichipo$ from \cref{sec:tradeoff_appendix}.\loose

\cref{ass:vmax} effectively implies that all policies $\pi\in\Pi$
satisfy
$\nrm[\big]{\frac{\pi}{\piref}}_{\infty}\approxleq\frac{\Vmax}{\beta}$;
in other words, the policy class we use in \algshort satisfies \emph{all-policy
  $L_{\infty}$-concentrability} with
$\max_{\pi\in\Pi}\cC^{\pi}_{\infty}\approxleq\frac{\Vmax}{\beta}$. At first glance, this
might seem to trivialize the offline alignment problem, since it would
suffice to prove a generalization guarantee based on all-policy
concentrability, and then plug this bound in. We will show that this is not the
case, and that this is actually an intrinsic feature of \chis-regularization.

In more detail, recall that for \algshort, we require the realizability
assumption that $\pistarb\in\Pi$ (\cref{ass:realizability}), where $\pistarb$ is the optimal
mixed \chis-regularized policy that satisfies $\rstar(x,a) =
\beta\phi\prn*{\frac{\pistarb(a\mid{}x)}{\piref(a\mid{}x)}}+\Zrstar(x)$. This
policy, via \cref{prop:conc_bounds}, satisfies
$\nrm[\big]{\frac{\pistarb}{\piref}}_{\infty}\approxleq\frac{\Rmax}{\beta}$,
so from a statistical perspective, we can take \cref{ass:vmax} to hold
without loss of generality by removing any policy that violates this
bound. In addition, as highlighted by \cref{ex:reward_model}, if we
begin from a class of bounded reward models $\cR$ with $\rstar\in\cR$,
\cref{ass:vmax} holds with $\Vmax\approxleq{}\Rmax$ for the induced class $\Pichipo$ defined in
\cref{eq:pichipo}, even though knowledge of such a reward model class
is a mild statistical assumption that clearly does not trivialize the
learning problem.

On the other hand, for \dpo, a minimal assumption is that
$\pistarkl\in\Pi$ \citep{xie2024exploratory}, where $\pistarkl$ is the optimal KL-regularized
policy that satisfies $\rstar(x,a) =
\beta\log\frac{\pistarkl(a\mid{}x)}{\piref(a\mid{}x)}+\Zklrstar(x)$. 
Unlike the
optimal mixed \chis-regularized policy, $\pistarkl$ has
$\frac{\pistarkl(a\mid{}x)}{\piref(a\mid{}x)}\approxgeq{}\exp\prn*{\frac{\Rmax}{\beta}}$. This
means that 
it is impossible to find a policy class that
simultaneously (1) realizes $\pistarkl$, and (2) satisfies all-policy
concentrability with
$\max_{\pi\in\Pi}\cC^{\pi}_{\infty}\ll{}\exp\prn*{\frac{\Rmax}{\beta}}$. As
the bias of \dpo is unacceptably large unless $\beta=\poly(1/n)$ (the
``small-$\beta$'' regime), this
leads to vacuous guarantees.\loose

In view of these observations, our analysis of \algshort can be interpreted as (implicitly) showing that for any
bounded reward class $\cR$, there exists a policy class $\Pi$ (precisely, the class $\Pichipo$ defined in \cref{eq:pichipo}) such that the following properties hold:
\begin{enumerate}[leftmargin=*]
\item \textbf{Bounded bias.} For every $r\in\cR$, there exists $\pi_{r}\in\Pi$ such that for
  all policies $\pistar$, $\Jr[r](\pistar) - \Jr[r](\pi_{r}) \approxleq{}\beta\cdot\cC^{\pistar}$.\loose
\item \textbf{Bounded overoptimization.} For all $\pi\in\Pi$, $\nrm[\big]{\frac{\pi}{\piref}}_{\infty}\approxleq\frac{\Rmax}{\beta}$.
\end{enumerate}
We view this as an interesting and non-trivial contribution in its own
right. We mention in passing that while it is indeed possible to
analyze \algshort by first proving a sample complexity guarantee based
on all-policy concentrability and then using that
$\max_{\pi\in\Pi}\cC^{\pi}_{\infty}\approxleq\frac{\Vmax}{\beta}$, this would lead
to a loose bound relative to \cref{thm:main}.\loose

}

\arxiv{
\section{Analysis of \algshortb: Proof Sketch for \creftitle{thm:main}}
\label{sec:proof_sketch}

In this section, we sketch the proof of the main guarantee for \algshort, \cref{thm:main}, with the full proof deferred to \cref{sec:proofs_main}. A central object in the
proof is the \emph{implicit} reward model induced by the \algshort policy
$\pihat$, which we define via
\begin{align}
  \label{eq:rhat}
\rhat(x,a) \ldef{} \beta\link{}\prn*{\frac{\pihat(a\mid{}x)}{\piref(a\mid{}x)}}.
\end{align}
As we will show, this reward model is a natural bridge between
\algshort and the corresponding mixed \chis-regularized RLHF objective in
\cref{sec:framework}, and allows us to view \algshort from a reward-based perspective. 
In particular, note that if we analogously define an induced reward
model class $\cR_{\Pi}\ldef{}\crl{r(x,a) =
  \beta\link{}\prn*{\frac{\pi(a\mid{}x)}{\piref(a\mid{}x)}} :
  \pi\in\Pi}$, then \cref{line:chi_dpo} of \algshort can be viewed as
performing maximum likelihood estimation over this class (in the sense
of \cref{eq:mle}) under the Bradley-Terry model. Under
\cref{ass:realizability}, $\cR_{\Pi}$ realizes the true reward
function $r$ up to an action-independent shift. As a result, if we
define $\rdiff[r](x,a,b) \ldef{} r(x,a) - r(x,b)$, then using a fairly
standard generalization bound for maximum likelihood estimation (e.g.,
\citet{wong1995probability,zhang2006from,Sara00}; see
\cref{lem:clip-dpo-reward}), we can show that
\begin{align}
  \label{eq:vepsstat}
  \vepsstat^2 \ldef{}
\En_{x\sim\rho,a\sim\piref,b\sim\piref}\brk*{\abs*{\rdiff[\rhat](x,a,b) - \rdiff[\rstar](x,a,b)}^2}
  \leq \bigoh\prn*{\Vmax e^{2\Rmax}\cdot\frac{\log(|\Pi|/\delta)}{n}}.
\end{align}
In other words, the estimated reward model $\rhat$ is accurate under
the action distribution induced by $\piref$. However, $\rhat$ may
still be inaccurate for policies that select different actions from
$\piref$, raising concerns of overoptimization. To address this issue,
we use the following lemma, which shows that \chis-divergence bounds the extent to which the
accuracy of a reward model $\rhat$ trained under $\piref$ will
transfer to a downstream policy $\pi$ of interest; this will motivate
our use of \chis-regularization.

\begin{lemma}[Informal version of \cref{lem:clip-dpo-estimation}]
\label{lem:informal-dpo-estimation}
For any policy $\pi:\cX\to\Delta(\cA)$, it holds that 
  \begin{align*}
    \En_{x\sim\rho,a\sim\pi(\cdot\mid{}x),b\sim\piref(\cdot\mid{}x)}\brk*{\abs*{\rdiff[\rhat](x,a,b) - \rdiff[\rstar](x,a,b)}}  
    \lesssim \sqrt{(1+\Dchis{\pi}{\piref})\cdot \vepsstat^2} .
  \end{align*}
\end{lemma}
Going forward, let us abbreviate
$\En_{\pi,\piref}\brk*{\cdot}=\En_{x\sim\rho,a\sim\pi(\cdot\mid{}x),b\sim\piref(\cdot\mid{}x)}\brk*{\cdot}$. Let
$\pistar$ be an arbitrary policy. Noting
that $\Cone[\pi] = 1+2\Dchis{\pi}{\piref}$ and that  
\begin{align*}
J(\pistar)-J(\pihat)
\approxleq{}\En_{\pistar,\piref}\brk*{\abs*{\rdiff[\rhat](x,a,b) -
    \rdiff[\rstar](x,a,b)}}
+\En_{\pihat,\piref}\brk*{\abs*{\rdiff[\rhat](x,a,b) -
    \rdiff[\rstar](x,a,b)}},
\end{align*}
it follows immediately
from \cref{lem:informal-dpo-estimation} that \algshort obtains a crude
guarantee scaling
with all-policy concentrability,
i.e. $J(\pistar)-J(\pihat)\approxleq{}\sqrt{(\cC^{\pistar}+\cC^{\pihat})\vepsstat^2}
\leq \sqrt{(\cC^{\pistar}+\max_{\pi\in\Pi}\cC^{\pi})\vepsstat^2}$. 
This inequality is tight for
non-pessimistic algorithms like \dpo, which
reflects their sensitivity to overoptimization. To
obtain the improved guarantee for \algshort in \cref{thm:main}, which scales only with
\emph{single-policy concentrability} $\cC^{\pistar}$,
the crux of the remaining proof will be to show that \algshort
implicitly implements pessimism via mixed \chis-regularization. For
this, we appeal to the following central technical lemma, which we
expect to find broader use.
\begin{lemma}
[Informal version of \cref{lem:general-reward-to-policy}]
\label{lem:informal-general-reward-to-policy}
  Let $f$ be a convex function with $\dom(f) = \bbR_{+}$ 
  that is differentiable over its domain. 
  Given any parameter $\beta > 0$ 
  and policy $\bar\pi:\cX \rightarrow \Delta(\cA)$ 
  with $\bar\pi(a\mid{}x) \in \dom(f')$ for all $x,a$, 
  define the reward model $\bar r(x,a) = \beta f'\prn*{\frac{\pibar(a|x)}{\piref(a|x)}}$. Then  
  \begin{align*}
    \bar\pi \in \argmax_\pi{}
    \En_\pi\brk*{\bar r(x,a)} - \beta\cdot\Df{\pi}{\piref}. 
  \end{align*}   
\end{lemma}
Under \cref{ass:vmax} we have $\pihat \in \dom(\fmix')$. 
Then recalling that 
$\rhat(x,a) \ldef{}
\beta\link{}\prn*{\frac{\pihat(a\mid{}x)}{\piref(a\mid{}x)}}=\beta\fmix'\prn*{\frac{\pihat(a\mid{}x)}{\piref(a\mid{}x)}}$
and that $\fmix$ is convex,  
\cref{lem:informal-general-reward-to-policy} implies that the policy $\pihat$
produced by \algshort satisfies
\begin{align}
\label{eq:rlhfb}
  \pihat \in \argmax_{\pi\in\Pi} \Jmixr[\rhat](\pi) \ldef{} \En_\pi\brk*{\rhat} - \beta \Dchis{\pi}{\piref} - \beta \Dkl{\pi}{\piref}. 
\end{align}
In other words,
\begin{center}
  \mbox{\emph{The \algshort policy $\pihat$ optimizes the mixed \chis-regularized RLHF
      objective %
      under
      its own implicit reward model.}}
\end{center}
This formally justifies the claim that \algshort implicitly implements
pessimism via \chis-regularization. With this result in hand, we are
now ready to prove \cref{thm:main}.
Let $\pistar$ be an arbitrary policy. Since $\Jmixr[\rhat](\pihat) \ge
\Jmixr[\rhat](\pistar)$ by \cref{eq:rlhfb}, we can decompose the
regret $J(\pistar)-J(\pihat)$ as
\begin{align*}
  J(\pistar) - J(\pihat) 
  \le&~ 
  J(\pistar) - \Jmixr[\rhat](\pistar)
  + \Jmixr[\rhat](\pihat) - J(\pihat)
  \\
  =&~ 
  \underbrace{
  J(\pistar) - J(\piref)
  - \Jmixr[\rhat](\pistar) + \Jmixr[\rhat](\piref)
  }_{\I} 
  + \underbrace{
  \Jmixr[\rhat](\pihat) - \Jmixr[\rhat](\piref)
  - J(\pihat) + J(\piref)
  }_{\II}. 
\end{align*}
In the second line, 
we have added or subtracted 
the baselines 
$J(\piref)$ and $\Jmixr[\rhat](\piref)$ 
to center the objectives 
with the performance of the reference policy.  
Up to statistical errors, the first term $\I$ corresponds to error from how much
$\Jmixr[\rhat](\pistar)$ \emph{underestimates} the return of
$\pistar$ (bias), and the second term $\II$ corresponds to error from how much
$\Jmixr[\rhat](\pihat)$ \emph{overestimates} the return of
$\pihat$ (overoptimization). As we will see shortly, these two sources
of error are directly controlled (in opposing ways) by the strength of
the regularization parameter $\beta$ in \cref{eq:rlhfb}.\loose

First, expanding the definition of $\Jmixr[\rhat](\pistar)$ and centering
the returns using the reference policies, we have
\begin{align*}
  \I 
  &=   
  J(\pistar) - \Jmixr[\rhat](\pistar)
  - J(\piref) + \Jmixr[\rhat](\piref)
  \\
  &= \En_{\pistar}\brk*{\rstar(x,a)}
     - \En_{\pistar}\brk*{\rhat(x,a)}
  + \beta\Dchis{\pistar}{\piref} + \beta\Dkl{\pistar}{\piref}
  - \En_{\pihat}[\rstar(x,a)] + \En_{\piref}[\rhat(x,a)]
  \\
  &= \En_{\pistar,\piref}\brk{\rdiff[\rstar](x,a,b) - \rdiff[\rhat](x,a,b)} 
      + \beta\Dchis{\pistar}{\piref} + \beta\Dkl{\pistar}{\piref}
      \\
    &\leq \sqrt{(1+\Dchis{\pistar}{\piref})\cdot\vepsstat^2} 
  +{}\underbrace{\beta \cdot \Dchis{\pistar}{\piref}}_{\text{bias}} .
\end{align*}
Above, we have used that $\Dkl{\pi}{\piref} \le \Dchis{\pi}{\piref}$
for any policy $\pi$, along with the bound on reward estimation error
from \cref{lem:informal-dpo-estimation}. Next, expanding
$\Jmixr[\rhat](\pihat)$ and centering the returns in a similar
fashion,
\begin{align*}
  \II &=   \Jmixr[\rhat](\pihat) - J(\pihat)
  - \Jmixr[\rhat](\piref) + J(\piref)
  \\
  &= \En_{\pihat,\piref}\brk{\rdiff[\rhat](x,a,b) - \rdiff[\rstar](x,a,b)} 
    - \beta\Dchis{\pihat}{\piref} - \beta\Dkl{\pihat}{\piref}\\
      &\leq 
        \sqrt{(1+\Dchis{\pihat}{\piref})\cdot\vepsstat^2}-
        \beta\cdot\Dchis{\pihat}{\piref}\\
      &\approxleq{} \vepsstat +
        \underbrace{\beta^{-1}\vepsstat^2}_{\textrm{overoptimization error}}.
\end{align*}
Above, the first inequality uses $\Dkl{\pi}{\piref}\geq{}0$ and
\cref{lem:informal-dpo-estimation}, while the second inequality uses
AM-GM. Critically, by using \chis-regularization, we are able to
cancel the on-policy error term $\sqrt{(1+\Dchis{\pihat}{\piref})\cdot\vepsstat^2}$
that arises from change-of-measure, leading to a modest
$\beta^{-1}\vepsstat^2$ penalty for overoptimization.

Combining these results, and recalling that $\Cone[\pi] =
1+2\Dchis{\pi}{\piref}$, we conclude that
\[
  J(\pistar) - J(\pihat)
  \approxleq{} \sqrt{\cC^{\pistar}\cdot\vepsstat^2} +
  \underbrace{\beta\cdot\cC^{\pistar}}_{\textrm{bias}} +
  \underbrace{\beta^{-1}\cdot\vepsstat^2}_{\textrm{overoptimization error}}.
\]

The bias and overoptimization errors above arise from how well our
chosen uncertainty quantifier, $\beta\Dchis{\pi}{\piref}$, accounts
for the on-policy statistical error
$\sqrt{(1+\Dchis{\pi}{\piref})\cdot\vepsstat^2}$ arising from \cref{lem:informal-dpo-estimation}; this is controlled by the magnitude of the regularization parameter $\beta$. 
When $\beta$ is too large, the uncertainty quantifier is overly
pessimistic about the quality of the reward model $\rhat$ under
$\pistar$%
, which increases the \emph{bias} of \algshort. In contrast, the \emph{overoptimization error} increases when $\beta$ is too small. 
In this regime, $\pihat$ overfits to $\rhat$ because the regularizer under-evaluates the statistical error of the learned policy. 
In order to obtain tight statistical rates, the choice of
regularization parameter $\beta$ must carefully balance its opposing
effects on bias and overoptimization error. For a fixed $\pistar$, choosing $\beta \propto (\vepsstat^2/\Cone[\pistar])^{1/2}$ results in the second claim in \cref{thm:main}.

\renewcommand{\rstar}{r^{\star}}

}

\arxiv{
  \section{Experiments in Offline Language Model Alignment}
  \label{sec:experiments}
  
We perform preliminary evaluations of \algshort for offline language model alignment on the TL;DR dataset \citep{stiennon2020learning}, using \dpo as our comparison baseline. 
The reference policy $\piref$ is the Pythia-1b model \citep{biderman2023pythia} pre-trained on SFT data (\href{https://huggingface.co/cleanrl/EleutherAI_pythia-1b-deduped__sft__tldr}{\texttt{cleanrl/EleutherAI\_pythia-1b-deduped\_\_sft\_\_tldr}} from \cite{huang2022cleanrl}), and performance is measured via winrate against a baseline, as judged by GPT-4o. All parameters that are not algorithm-specific, such as the learning rate, are shared by both \algshort and \dpo in order to ensure a fair comparison (see \cref{sec:experiment-details} for details).

\begin{table}[h]
    \centering
    \caption{Winrate on TL;DR Summarization for models learned by \algshort and \dpo, for several choices of the number of training epochs and the regularization parameter $\beta$. Standard error over 3 seeds is also reported. }
    \label{tab:final-winrate}
    \begin{tabular}{cccc} %
        \toprule
        $\beta$ & Epochs & \algshort winrate (\%) & \dpo winrate (\%) \\ %
        \midrule
        \multirow{3}{*}{0.05} 
            & 1 & $\mathbf{56.5\pm 1.3}$ & $55.8 \pm 2.1$ \\ 
            & 2 & $\mathbf{56.1 \pm 0.6}$ & $50.3 \pm 0.8$ \\
            & 4 & $\mathbf{48.0 \pm 1.6}$ & $38.0 \pm 0.7$ \\
        \midrule
        \multirow{3}{*}{0.005} 
            & 1 & $\mathbf{50.6 \pm 1.6}$ & $14.7 \pm 3.9$ \\ 
            & 2 & $\mathbf{52.8 \pm 2.3}$ & $3.4 \pm 1.5$ \\
            & 4 & $\mathbf{51.6 \pm 0.8}$ & $0.5 \pm 0.2$ \\
        \bottomrule 
    \end{tabular}
\end{table}

In \cref{tab:final-winrate} we display the winrates of \algshort and \dpo over several choices of training epochs, as well as regularization parameter $\beta$. The winrate corresponds to the final checkpoint learned by each algorithm for each set of hyperparameters. We consider $\beta = 0.05$ and 1 epoch of training to be a standard setup for $\dpo$ \citep{gao2024rebel,guo2024direct,rafailov2024scaling}, and, 
as we are particularly concerned with regimes where overoptimization is of concern, we additionally analyze performance when epochs are increased, and/or $\beta$ is decreased (corresponding to less regularization).

\begin{figure}[t]
    \centering
    \label{fig:tldr-epochs-2}
    \includegraphics[width=0.47\textwidth]{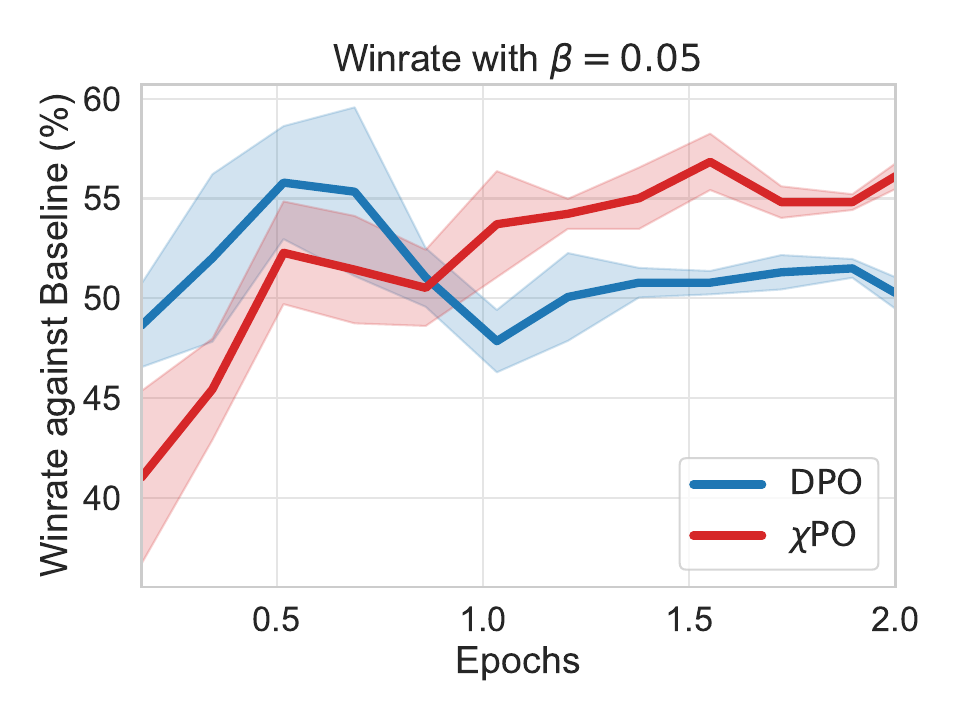}
    \includegraphics[width=0.47\textwidth]{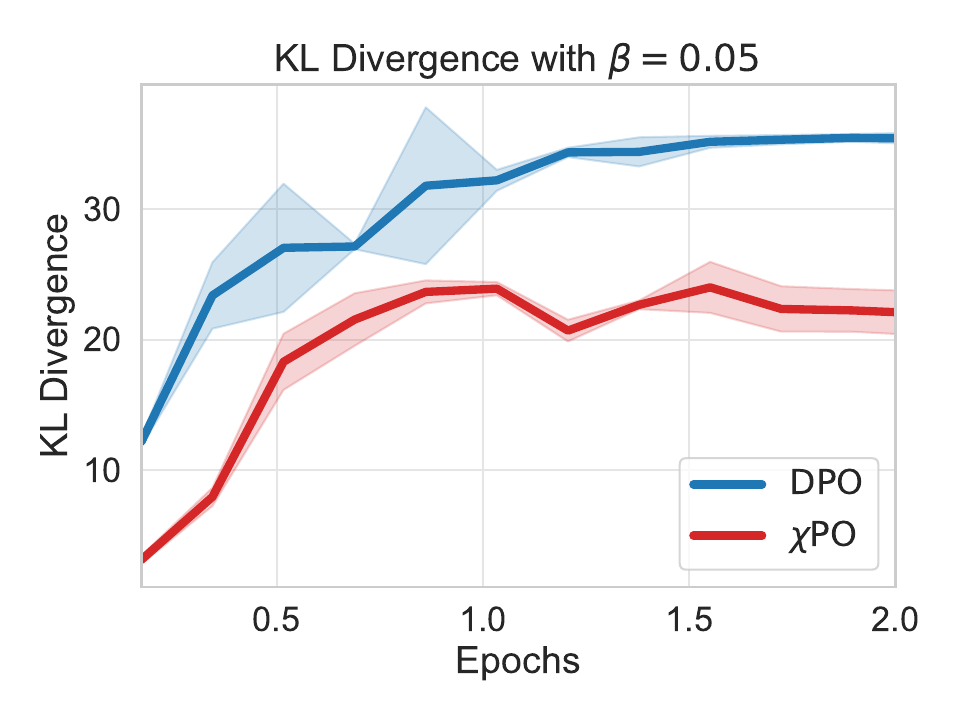}
    \caption{(Left) TL;DR Summarization winrate recorded longitudinally over 2 epochs of training every 250 steps. Shaded area displays $\pm 1$ standard error over 3 seeds. At 1 epoch \algshort already obtains better performance, and continues to improve over the course of training, while \dpo degrades over time. (Right) KL divergence $\Dkl{\pihat}{\piref}$ averaged over 2 of the seeds. For the same $\beta$, \algshort constrains the learned policy to be significantly closer to $\piref$, thereby striking a better bias-variance tradeoff. }
\end{figure}

Over all choices of $\beta$ and epochs, \algshort achieves a higher average winrate than $\dpo$. While the difference is not significant for $\beta = 0.05$ and 1 epoch, the performance gap grows significantly as the number of epochs increases, demonstrating the robustness of \algshort to overoptimization. Further, while \dpo degrades completely for $\beta = 0.005$, \algshort is robust over two orders of magnitude of $\beta$, reinforcing trends seen earlier in \cref{fig:regret} and the more favorable bias-overoptimization tradeoff from our theoretical analysis.

In addition, \algshort exhibits better performance and robustness longitudinally throughout training, as shown in \cref{fig:tldr-epochs-2}. While \dpo peaks early with high variance around 0.5 epochs and degrades thereafter, \algshort continues to improve smoothly then plateaus over the last epoch. Further, for the same regularization parameter $\beta$, the \algshort policy has significantly lower KL-divergence relative to $\piref$, demonstrating that the \chis-regularization is both a stronger regularizer and one that effectively mitigates overoptimization.

}

\arxiv{
\section{\algshortb for General Preference Models}
\label{sec:general_preference}

\newcommand{\Ac}{\mathcal{A}}
\newcommand{\Bc}{\mathcal{B}}
\newcommand{\Cc}{\mathcal{C}}
\newcommand{\Dc}{\mathcal{D}}
\newcommand{\Ec}{\mathcal{E}}
\newcommand{\Fc}{\mathcal{F}}
\newcommand{\Gc}{\mathcal{G}}
\newcommand{\Hc}{\mathcal{H}}
\newcommand{\Ic}{\mathcal{I}}
\newcommand{\Jc}{\mathcal{J}}
\newcommand{\Kc}{\mathcal{K}}
\newcommand{\Lc}{\mathcal{L}}
\newcommand{\Mc}{\mathcal{M}}
\newcommand{\Nc}{\mathcal{N}}
\newcommand{\Oc}{\mathcal{O}}
\newcommand{\Pc}{\mathcal{P}}
\newcommand{\Qc}{\mathcal{Q}}
\newcommand{\Rc}{\mathcal{R}}
\newcommand{\Sc}{\mathcal{S}}
\newcommand{\Tc}{\mathcal{T}}
\newcommand{\Uc}{\mathcal{U}}
\newcommand{\Vc}{\mathcal{V}}
\newcommand{\Wc}{\mathcal{W}}
\newcommand{\Xc}{\mathcal{X}}
\newcommand{\Yc}{\mathcal{Y}}
\newcommand{\Zc}{\mathcal{Z}}

\newcommand{\Ab}{\mathbb{A}}
\newcommand{\Bb}{\mathbb{B}}
\newcommand{\Cb}{\mathbb{C}}
\newcommand{\Db}{\mathbb{D}}
\newcommand{\Eb}{\mathbb{E}}
\newcommand{\Fb}{\mathbb{F}}
\newcommand{\Gb}{\mathbb{G}}
\newcommand{\Hb}{\mathbb{H}}
\newcommand{\Ib}{\mathbb{I}}
\newcommand{\Kb}{\mathbb{K}}
\newcommand{\Lb}{\mathbb{L}}
\newcommand{\Mb}{\mathbb{M}}
\newcommand{\Nb}{\mathbb{N}}
\newcommand{\Ob}{\mathbb{O}}
\newcommand{\Pb}{\mathbb{P}}
\newcommand{\Qb}{\mathbb{Q}}
\newcommand{\Rb}{\mathbb{R}}
\newcommand{\Sb}{\mathbb{S}}
\newcommand{\Tb}{\mathbb{T}}
\newcommand{\Ub}{\mathbb{U}}
\newcommand{\Vb}{\mathbb{V}}
\newcommand{\Wb}{\mathbb{W}}
\newcommand{\Xb}{\mathbb{X}}
\newcommand{\Yb}{\mathbb{Y}}
\newcommand{\Zb}{\mathbb{Z}}

\newcommand{\hpi}{\widehat{\pi}}
\newcommand{\DR}{\cDpref}
\newcommand{\hl}{\widehat{\ell}}
\newcommand{\ox}{\overline{x}}
\newcommand{\oa}{\overline{a}}
\newcommand{\ob}{\overline{b}}
\newcommand{\ovr}{\overline{r}}
\newcommand{\hr}{\widehat{r}}
\newcommand{\Clip}{\clip}
\newcommand{\bpi}{\overline{\pi}}
\newcommand{\suboptun}{\mathsf{subopt}}
\newcommand{\tpi}{\widetilde{\pi}}
\newcommand{\eop}{\eps_{\mathsf{md}}}
\newcommand{\esg}{\eps_{\mathsf{general}}}

\newcommand{\Ccon}{C_{\mathsf{con}}}
\addauthor{whz}{blue}
\newcommand{\whz}[1]{\whzcomment{#1}}
\renewcommand{\eps}{\varepsilon}
\newcommand{\lss}{\ell^{\star}}

\newcommand{\piu}{\pimw}
\newcommand{\CS}{\cC^{\piu}}
\newcommand{\cPstar}{\cP^\star}

All of our results so far concern the Bradley-Terry model
(\cref{eq:bt}), which, as highlighted in prior work, is somewhat restrictive. 
Thus, in this section, we turn our attention to offline alignment
under a \emph{general preference model} which does not assume transitivity  \citep{munos2023nash,wang2023rlhf,swamy2024minimaximalist,rosset2024direct,ye2024theoretical}. The setup is the same as \cref{sec:background},
but we assume that for a given context $x$ and pair of actions
$(a,b)$, the preference $y\in\crl{0,1}$ is generated via a Bernoulli Distribution
\begin{align}
\label{eq:general_pref}
  y\sim{}\Ber\prn*{\cPstar(a\psdgt b \mid{}x)},
\end{align}
where $\cPstar(a\psdgt b\mid{}x)\in\brk{0,1}$ is a general preference
distribution. For a pair of policies $\pi,\pi'$, let $\cPstar(\pi\psdgt
\pi')\ldef{}\En_{x\sim\rho}\brk*{\cPstar(\pi(x)\psdgt{}\pi'(x)\mid{}x)}$. Following
\citet{wang2023rlhf,munos2023nash,swamy2024minimaximalist}, we consider the \emph{minimax
  winner}
\citep{kreweras1965aggregation,simpson1969defining,kramer1973class,fishburn1984probabilistic}
or \emph{von Neumann winner} \citep{dudik2015contextual} as a solution concept:
\begin{align*}
  \pimw\ldef{}\argmax_{\pi\in\Pi}\min_{\pi'\in\Pi}\cPstar(\pi\psdgt{}\pi').
\end{align*}

It will be useful to slightly reparameterize this
formulation by introducing the preference function $\lss(x,a,b)\ldef{}2\cPstar(a\psdgt b\mid{}x)-1$.
  Note that for any well-defined preference model, we have $\cPstar(a\psdgt b \mid{}x) + \cPstar(b\psdgt a \mid{}x) = 1$ for all $x,a,b$, which indicates that $\lss$ satisfies skew symmetry:
\begin{align*}
\lss(x,a,a) = 0,\qquad \lss(x,a,b) + \lss(x,b,a) = 0, \qquad\forall x\in\Xc, a,b\in\Ac.
\end{align*}
Furthermore, the minimax winner above is equivalent to
\begin{align}
\label{eq:mw}
\pimw\ldef{}\argmax_{\pi\in\Pi}\min_{\pi'\in\Pi} \lss(\pi,\pi'),
\end{align}
where $\lss(\pi,\pi'):=\Eb_{x\sim\rho,a\sim\pi(x),b\sim\pi'(x)}[\lss(x,a,b)]$. Concretely, our goal is to use the logged preference data
$\cDpref=\crl*{(x,\ap,\am)}$ (with $(\ap,\am)$ labeled according to
\cref{eq:general_pref}) to compute a policy $\pihat$ that is an
$\veps$-approximate minimax winner, in the sense that
\begin{align}
  \label{eq:duality_gap}
  \DG(\hpi)
  \ldef{} \max_{\pi\in\Pi}\lss(\pi,\hpi)
  - \min_{\pi\in\Pi}\lss(\hpi,\pi) \leq \veps.
\end{align}

\subsection{Impossibility of Single-Policy Concentrability under
  General Preferences}

While the general preference framework above is more powerful than the
Bradley-Terry model, we now show that there is a statistical cost for
this generality. In particular, our first result in this
section shows that in contrast to the Bradley-Terry model, it is not
possible to achieve sample complexity guarantees that scale with
single-policy concentrability under general preferences, even when the
learner has access to a small class of preference models $\scrP$ that
contains the true preference model $\cP$ (i.e., $\cPstar\in\scrP$).

\begin{theorem}[Impossibility of single-policy concentrability under
  general preferences]
  \label{thm:general_lower}
There exists
two problem instances $\theta_1 = (\rho,\cPstar_1,\Pi)$ and $\theta_2 = (\rho, \cPstar_2,\Pi)$ differing only in their ground truth preference model, 
a data collection policy $\piref$, and a preference model class $\scrP = \{\cPstar_1,\cPstar_2\}$ with $|\scrP|=2$
  such that the following hold:\loose
  \begin{enumerate}
  \item For both instances, the single-policy $L_{\infty}$-concentrability coefficient for a minimax
    winner is bounded:
      $\min_{\pimw}\cC_{\infty}^{\pimw}\leq{}2$.\footnote{In general, the minimax winner may not
      be unique. We compete against the minimax winner with the
      \emph{best} possible single-policy concentrability coefficient.}
    \item For any $n \in \bbN$ and any algorithm $\mathsf{Alg}$ which derives a policy $\pihat$ from a dataset $\cD_\pref$ of $n$ samples, there exists an instance $\theta \in \{\theta_1,\theta_2\}$ such that $\piref$ incurs constant suboptimality:
    \begin{align*}
    \min_{\mathsf{Alg}} \max_{i \in \{1,2\}} \En_{\cD_\pref\sim\theta_i}\brk*{\DG(\mathsf{Alg}(\cD_\pref);\theta_i)} \geq{} \frac{1}{8},
    \end{align*}
    where $\DG(\pi;\theta)$ is the duality gap for policy $\pi$ on instance $\theta$. 
  \end{enumerate}
\end{theorem}

This lower bound is inspired by similar results in the literature on offline RL in
               two-player zero-sum Markov games
               \citep{cui2022offline}. However, the lower bound constructions in \citet{cui2022offline} cannot be directly applied
                as-is, because they do not satisfy the skew-symmetry
                property required by the general preference alignment
                framework. Our lower bound highlights that even under
                skew-symmetry, it is impossible to achieve
                single-policy concentrability for offline learning in two-player zero-sum
                games.

\subsection{Iterative \algshortb for General Preferences}

In spite of the hardness in the prequel, we now show that an iterative
variant of \algshort---based on self-play---can learn a near-optimal minimax winner under the general preference
model under a new local coverage condition---a condition that is
stronger than the single policy concentrability but much weaker than
global/all-policy concentrability and the notion of unilateral concentrability introduced by \cite{cui2022offline}.

Our algorithm, Iterative \algshort, is described in
\cref{alg:general}, and consists of two main steps. 

\paragraph{Preference model estimation via least squares regression on
  $\DR$} We first (\cref{line:genest}) learn a preference model from
the offline preference dataset $\DR$. We assume access to a preference function class $\Lc$ which is realizable in the sense that $\lss \in \Lc$ and where all $\ell \in \Lc$ satisfy skew-symmetryc, and we will estimate $\lss$ rather than $\cPstar$. 
 We perform least-squares regression on $\DR$ with $\Lc$ to learn $\lss$:
\begin{align*}
\hl= \argmin_{\ell\in \Lc} \sum_{(x,\ap,\am)\in\DR}\left( \ell(x, \ap, \am) - 1  \right)^2.
\end{align*}

\paragraph{Policy optimization with iterative \algshort update}
Given the estimated model $\hl$, we compute an approximate minimax
winner using an iterative regression scheme inspired by
\citet{gao2024rebel}. We proceed in $T$ iterations
(\cref{line:genloop}), where at each iteration $t$, we define an
iteration-dependent reward function $\ovr^t(x,a)$ based on the current
policy $\pi\ind{t}$ as
\begin{align*}
\ovr^t(x,a) = \Eb_{b\sim\pi^t(x)}[\hl(x,a,b)],\qquad\forall x\in\Xc, a\in\Ac.
\end{align*}
Then, for all $\pi,\pi'\in\Pi$, we define a policy-dependent predictor
$f^{\beta,\eta}_{\pi,\pi'}(x,a,b)$, whose motivation will be described
in detail momentarily, as follows:
\begin{align}
f^{\beta,\eta}_{\pi,\pi'}(x,a,b)&:= \left(1+\frac{1}{\eta}\right)\cdot\left(\beta\phi\left(\frac{\pi\left(a\mid x\right)}{\piref\left(a\mid x\right)}\right)-\beta\phi\left(\frac{\pi\left( b\mid x\right)}{\piref\left( b\mid x\right)}\right)\right) \notag\\
& ~~~~~~ -\frac{1}{\eta}\left(\beta\phi\left(\frac{\pi'\left( a\mid x\right)}{\piref\left( a\mid x\right)}\right)-\beta\phi\left(\frac{\pi'\left( b\mid x\right)}{\piref\left( b\mid x\right)}\right)\right)
\label{eq:predictor_general_pref}
\end{align}
Using $f^{\beta,\eta}_{\pi,\pi^t}(x,a,b)$ as a policy-parameterized
regression function, we (\cref{line:genreg}) compute the next policy
$\pi\ind{t+1}$ by solving a least-squares regression problem in which the Bayes optimal solution is the
\emph{relative reward} $\ovr^t(x,a) - \ovr^t(x,b)$ for iteration $t$.

\begin{algorithm}[tp]
	\caption{Iterative \algshort for General Preferences}
	\label{alg:general}
	\begin{algorithmic}[1]
		\State \textbf{Input}: labeled preference dataset $\DR$, preference model class $\Lc$, regularization coefficient $\beta$, stepsize $\eta$, total number of iterations $T$.
		\State \textbf{Initialize}: $\pi^1=\piref$. 
		\State Learn a preference model $\hl$ via least-squares regression: $$\hl= \argmin_{\ls\in \Lc}
                \sum_{(x,\ap,\am)\in\DR}\left( \ls(x, \ap, \am) - 1
                \right)^2.$$ \label{line:genest}
	  \State Collect $m$ samples $\cDnopref=\{(x,a,b)\}$ where each sample is drawn i.i.d. from $x\sim\rho,a\sim\piref(x), b\sim\piref(x)$.
		\For{$t=1,\cdots,T$} 	 \label{line:genloop}
        \State Sample $b_t\sim\pi^t(x)$ and let $\hr^t(x,a)=\hl(x,a,b_t)$ for all $x\in\Xc, a\in\Ac$. 
		\State Compute
		\begin{align}\pi^{t+1}=\argmin_{\pi\in\Pi}\sum_{(x,a,b)\in\cDnopref}\left(\Clip_{4}\left(f^{\beta,\eta}_{\pi,\pi^t}(x,a,b)\right)-(\hr^t(x,a)-\hr^t(x,b))\right)^2,\label{eq:regression}\end{align}
        $\qquad$ where $f^{\beta,\eta}_{\pi,\pi^t}(x,a,b)$ is defined in \cref{eq:predictor_general_pref}.\label{line:genreg}
		\EndFor
		\State \textbf{Output:} $\hpi=\unif(\{\pi^t\}_{t=1}^T)$.
              \end{algorithmic}
            \end{algorithm}

Let us now explain the intuition behind the the predictor
$f^{\beta,\eta}_{\pi,\pi'}(x,a,b)$. Suppose that the regression step
in \cref{line:genreg} learns a predictor that can perfectly model the relative reward, i.e., 
\begin{align*}
\forall x, a, b, \quad f^{\beta,\eta}_{\pi^{t+1},\pi^t}(x,a,b) = \ovr^t(x,a) - \ovr^t(x,b),
\end{align*}
In this case, we can show that the returned policy $\pi^{t+1}$ is
  the optimal policy for the following mixed \chis-regularized RL objective:
\begin{align}
\label{eq:regularized_rl_iterate}
\pi\ind{t+1}(x) =\argmax_{p\in\Delta(\Xc)}\crl*{\En_{a\sim{}p}\brk*{\ovr^t(x,a)} - \beta\Dfmix{p}{\piref(x)} - \frac{\beta}{\eta}B_x(p,\pi^t)},\qquad\forall x\in\Xc,
\end{align}  where $B_x(p,\pi^t)$ is the Bregman divergence induced by
the regularizer $p\mapsto{}\Dfmix{p}{\piref(x)}$, i.e., 
\begin{align*}
B_x(p,q) := \Dfmix{p}{\piref(x)} - \Dfmix{q}{\piref(x)} - \left\langle\nabla \Dfmix{q}{\piref(x)}, p-q \right\rangle, \qquad\forall x\in\Xc.
\end{align*}
Thus, the algorithm can be understood as running mirror descent on the
iteration-dependent loss function $-\ovr^t $, with $p\mapsto{}\Dfmix{p}{\piref(x)}$ as a
per-context regularizer. 
This technique draws inspiration from~\citet{chang2024dataset}, in which the authors apply a similar regularized mirror descent algorithm to learn the optimal policy for the \emph{reward-based} setting. 
The motivation for using mixed-\chis{} regularization is exactly the same as in \algshort: we
want to ensure that $\frac{\pi^{t+1}(a|x)}{\piref(a|x)} \leq 1
    +\frac{1}{\beta}$, thereby mitigating overoptimization.

\subsection{Theoretical Analysis of Iterative \algshortb}

We now present our main theoretical guarantees for Iterative
\algshortb. We begin by stating a number of statistical
assumptions. We first assume that the preference model class contains the ground truth preference function $\lss$.\loose
\begin{assumption}[Preference function realizability]
\label{ass:pre}
The model class $\Lc$ satisfies $\lss\in\Lc$ where $\lss$ is the ground truth preference function.\loose
\end{assumption}
In addition, since \cref{alg:general} iteratively applies an \algshort update, we require that a policy realizability assumption
analogous to \cref{ass:realizability} holds for each of the
sub-problems in \cref{eq:regularized_rl_iterate}. Concretely, we make
the following assumption.
\begin{assumption}[Policy realizability for general preferences]
\label{ass:md}
For any policy $\pi\in\Pi$ and $\ell\in\Lc$, the policy class $\Pi$ contains the minimizer of the following regularized RL objective:
\begin{align*}
\pibar(x;\ell,\pi)\ldef{} \argmax_{p\in\Delta(\Xc)}\crl*{\En_{a\sim{}p,b\sim\pi(x)}\brk*{\ell(x,a,b)}  - \beta\Dfmix{p}{\piref(x)} -\frac{\beta}{\eta}B_x(p,\pi)},\qquad\forall x\in\Xc.
\end{align*}
\end{assumption}
Finally, we require that the implicit reward functions in
\cref{eq:regression} are bounded, analogous to \cref{ass:vmax}.
\begin{assumption}[Bounded implicit rewards for general preferences]
\label{ass:vmax-g}
For a parameter $\Vmax\geq2$, it holds that for all $\pi,\pi'\in\Pi$, $x\in\Xc$, and $a,b\in\Ac$,
\begin{equation}
\left|f^{\beta,\eta}_{\pi,\pi'}(x,a,b)\right| \leq \Vmax.
\end{equation}
\end{assumption}

Our main guarantee for \cref{alg:general} is as follows.
\begin{theorem}
\label{thm:rdr}
Fix any $\delta\in(0,1]$. Suppose \cref{alg:general} is invoked with 
$T = \frac{mn}{n\Vmax^2 + m}$, $\beta = \frac{1}{\sqrt{T}}$, and $\eta = \frac{1}{T}$.
Then under \cref{ass:pre}, \cref{ass:md} and \cref{ass:vmax-g}, we
have that probability at least $1-\delta$,
\begin{align*}
	\DG(\hpi)\lesssim\min_{C\geq 1}\left\{\suboptun(\hpi,C)+ C\left(\frac{\Vmax\log(|\Pi|/\delta)}{\sqrt{m}} + \frac{\log(|\Pi||\Lc|/\delta)}{\sqrt{n}}\right)\right\},
\end{align*}
where $\suboptun(\hpi,C):=\max_{\pi\in\Pi}\lss(\pi,\hpi) -
\max_{\pi\in\Pi_{C}}\lss(\pi,\hpi)$ and $\Pi_C:=\{\pi:\max_{x\in\Xc}
\Dchis{\pi(x)}{\piref(x)}\leq C\}$. In particular, if we define the
\textit{unilateral concentrability coefficient} as %
\begin{align*}
C_{\mathsf{uni}}:=\max_{\pi\in\Pi,x\in\Xc,a,b\in\Ac}\frac{\pi(a \mid{} x)\pimw(b\mid{}x)}{\piref(a\mid{}x)\piref(b\mid{}x)},
\end{align*} then the bound above implies that
\begin{align*}
	\DG(\hpi)\lesssim  C_{\mathsf{uni}}\cdot\left(\frac{\Vmax\log(|\Pi|/\delta)}{\sqrt{m}} + \frac{\log(|\Pi||\Lc|/\delta)}{\sqrt{n}}\right).
\end{align*}
\end{theorem}
The first result gives a tradeoff between the statistical error and
the approximation error $\suboptun(\hpi,C)$, which is modulated by the
parameter $C$. This tradeoff is analogous to, but more subtle, than
the one for \algshort in the reward-based setting. In the reward-based
setting, \algshort has low regret to the best policy covered
$\piref$. In the general preference setting,~\cref{alg:general} has
small duality gap if, for any policy, there is an approximate best
response that is covered by $\piref$ (this implies that
$\suboptun(\hpi,C)$ is small for small
$C$). Crucially,~\cref{alg:general} does not require that all policies
are covered by $\piref$, which is a distinctive feature of mixed
\chis-regularization and reflects the algorithms robustness to
overoptimization.

The second result concerns the setting where all policies are covered by $\piref$ and is easier to interpret. Indeed, 
if all
$\pi\in\Pi$ satisfy $\Dchis{\pi}{\piref}\leq{}C^{\star}$, then
$\suboptun(\hpi,C^{\star})=0$, which implies that we can learn an
$\veps$-approximate minimizer using $\bigoht(C^{\star}/\veps^2)$
samples. 
Thus, we obtain a guarantee based on unilateral concentrability~\citep{cui2022offline}, which is a stronger condition, i.e., we always have $\max_\pi \Dchis{\pi}{\piref} \leq C_{\mathsf{uni}}$.
However, per the above discussion, the first part of~\cref{thm:rdr} is stronger than results based on unilateral concentrability and hints at a new notion of coverage for general preferences.
Lastly, we remark that the parameter $\Vmax$ only
affects $\sqrt{1/m}$ term in \cref{thm:rdr}, so dependence on this
parameter can be mitigated using unlabeled data.\loose

\cref{thm:rdr} is closely related to recent work of
\citet{ye2024theoretical}, which uses pessimism to learn a regularized minimax
winner, and achieves polynomial sample complexity with a
concentrability assumption similar to \cref{thm:rdr}. However, there
are two key differences. First, their learning objective is the
KL-regularized minimax winner, while we study the unregularized
objective and use \chis-regularization. More importantly, their theoretical algorithm
is computationally inefficient
as it constructs an explicit confidence set for the preference model and performs max-min-style policy optimization.  
In contrast, our
algorithm only requires solving standard supervised learning problems.

}

\section{Discussion}
\label{sec:discussion}
Our work gives the first \arxiv{practical, }general-purpose algorithm for offline alignment with provable robustness to overoptimization and sample complexity guarantees based on single-policy
concentrability. Conceptually, our results contribute to a growing
body of research that highlights the statistical benefits of \chis-divergence for
reinforcement learning \citep{wang2024oracle,gabbianelli2024importance,amortila2024scalable},
and offer an example of fruitful interplay between reinforcement
learning theory and language modeling.
\arxiv{From this perspective, we
expect that our analysis techniques and algorithm design ideas will find broader use.}

Natural technical directions raised by our paper include 
(i) developing a tight understanding of minimax sample complexity and instance-optimality for
offline alignment with general policy classes; 
(ii) understanding the tightest
possible problem-dependent sample complexity guarantees for offline alignment with
\emph{general preference models} (in light of lower bounds in
\cref{sec:general_preference}); and (iii) extending our techniques to
reinforcement learning settings beyond offline alignment (e.g., general MDPs).
We look forward to studying these questions in future work.\loose

\arxiv{
\subsection*{Acknowledgements}
We thank Qinghua Liu, Zhaolin Gao, and Yuda Song for several helpful discussions. 
WS acknowledges funding support from NSF IIS-2154711,
NSF CAREER 2339395, DARPA LANCER: LeArning Network CybERagents. 
}

\clearpage

\bibliography{refs}

\begin{thebibliography}{103}
\providecommand{\natexlab}[1]{#1}
\providecommand{\url}[1]{\texttt{#1}}
\expandafter\ifx\csname urlstyle\endcsname\relax
  \providecommand{\doi}[1]{doi: #1}\else
  \providecommand{\doi}{doi: \begingroup \urlstyle{rm}\Url}\fi

\bibitem[Agarwal et~al.(2019)Agarwal, Jiang, and Kakade]{agarwal2019reinforcement}
Alekh Agarwal, Nan Jiang, and Sham~M Kakade.
\newblock Reinforcement learning: Theory and algorithms.
\newblock \url{https://rltheorybook.github.io/}, 2019.
\newblock Version: January 31, 2022.

\bibitem[Agarwal et~al.(2020)Agarwal, Kakade, Krishnamurthy, and Sun]{agarwal2020flambe}
Alekh Agarwal, Sham Kakade, Akshay Krishnamurthy, and Wen Sun.
\newblock {FLAMBE}: Structural complexity and representation learning of low rank {MDP}s.
\newblock \emph{Advances in Neural Information Processing Systems}, 2020.

\bibitem[Amortila et~al.(2024)Amortila, Foster, and Krishnamurthy]{amortila2024scalable}
Philip Amortila, Dylan~J Foster, and Akshay Krishnamurthy.
\newblock Scalable online exploration via coverability.
\newblock \emph{International Conference on Machine Learning}, 2024.

\bibitem[Athey and Wager(2021)]{athey2021policy}
Susan Athey and Stefan Wager.
\newblock Policy learning with observational data.
\newblock \emph{Econometrica}, 2021.

\bibitem[Azar et~al.(2024)Azar, Guo, Piot, Munos, Rowland, Valko, and Calandriello]{azar2024general}
Mohammad~Gheshlaghi Azar, Zhaohan~Daniel Guo, Bilal Piot, Remi Munos, Mark Rowland, Michal Valko, and Daniele Calandriello.
\newblock A general theoretical paradigm to understand learning from human preferences.
\newblock In \emph{International Conference on Artificial Intelligence and Statistics}, 2024.

\bibitem[Bai et~al.(2022)Bai, Jones, Ndousse, Askell, Chen, DasSarma, Drain, Fort, Ganguli, Henighan, Joseph, Kadavath, Kernion, Conerly, El-Showk, Elhage, Hatfield-Dodds, Hernandez, Hume, Johnston, Kravec, Lovitt, Nanda, Olsson, Amodei, Brown, Clark, McCandlish, Olah, Mann, and Kaplan]{bai2022training}
Yuntao Bai, Andy Jones, Kamal Ndousse, Amanda Askell, Anna Chen, Nova DasSarma, Dawn Drain, Stanislav Fort, Deep Ganguli, Tom Henighan, Nicholas Joseph, Saurav Kadavath, Jackson Kernion, Tom Conerly, Sheer El-Showk, Nelson Elhage, Zac Hatfield-Dodds, Danny Hernandez, Tristan Hume, Scott Johnston, Shauna Kravec, Liane Lovitt, Neel Nanda, Catherine Olsson, Dario Amodei, Tom Brown, Jack Clark, Sam McCandlish, Chris Olah, Ben Mann, and Jared Kaplan.
\newblock Training a helpful and harmless assistant with reinforcement learning from human feedback.
\newblock \emph{arXiv:2204.05862}, 2022.

\bibitem[Biderman et~al.(2023)Biderman, Schoelkopf, Anthony, Bradley, O’Brien, Hallahan, Khan, Purohit, Prashanth, Raff, Skowron, Sutawika, and van~der Wal]{biderman2023pythia}
Stella Biderman, Hailey Schoelkopf, Quentin~Gregory Anthony, Herbie Bradley, Kyle O’Brien, Eric Hallahan, Mohammad~Aflah Khan, Shivanshu Purohit, USVSN~Sai Prashanth, Edward Raff, Aviya Skowron, Lintang Sutawika, and Oskar van~der Wal.
\newblock Pythia: A suite for analyzing large language models across training and scaling.
\newblock In \emph{International Conference on Machine Learning}, 2023.

\bibitem[Bradley and Terry(1952)]{bradley1952rank}
Ralph~Allan Bradley and Milton~E Terry.
\newblock Rank analysis of incomplete block designs: I. {T}he method of paired comparisons.
\newblock \emph{Biometrika}, 1952.

\bibitem[Brown et~al.(2020)Brown, Mann, Ryder, Subbiah, Kaplan, Dhariwal, Neelakantan, Shyam, Sastry, Askell, Agarwal, Herbert-Voss, Krueger, Henighan, Child, Ramesh, Ziegler, Wu, Winter, Hesse, Chen, Sigler, Litwin, Gray, Chess, Clark, Berner, McCandlish, Radford, Sutskever, and Amodei]{brown2020language}
Tom Brown, Benjamin Mann, Nick Ryder, Melanie Subbiah, Jared~D Kaplan, Prafulla Dhariwal, Arvind Neelakantan, Pranav Shyam, Girish Sastry, Amanda Askell, Sandhini Agarwal, Ariel Herbert-Voss, Gretchen Krueger, Tom Henighan, Rewon Child, Aditya Ramesh, Daniel Ziegler, Jeffrey Wu, Clemens Winter, Chris Hesse, Mark Chen, Eric Sigler, Mateusz Litwin, Scott Gray, Benjamin Chess, Jack Clark, Christopher Berner, Sam McCandlish, Alec Radford, Ilya Sutskever, and Dario Amodei.
\newblock Language models are few-shot learners.
\newblock In \emph{Advances in Neural Information Processing Systems}, 2020.

\bibitem[Cen et~al.(2024)Cen, Mei, Goshvadi, Dai, Yang, Yang, Schuurmans, Chi, and Dai]{cen2024value}
Shicong Cen, Jincheng Mei, Katayoon Goshvadi, Hanjun Dai, Tong Yang, Sherry Yang, Dale Schuurmans, Yuejie Chi, and Bo~Dai.
\newblock Value-incentivized preference optimization: A unified approach to online and offline {RLHF}.
\newblock \emph{arXiv:2405.19320}, 2024.

\bibitem[Cesa-Bianchi et~al.(2017)Cesa-Bianchi, Gentile, Lugosi, and Neu]{cesa2017boltzmann}
Nicol{\`o} Cesa-Bianchi, Claudio Gentile, G{\'a}bor Lugosi, and Gergely Neu.
\newblock Boltzmann exploration done right.
\newblock \emph{Advances in Neural Information Processing Systems}, 2017.

\bibitem[Chang et~al.(2024)Chang, Shan, Oertell, Brantley, Misra, Lee, and Sun]{chang2024dataset}
Jonathan~D Chang, Wenhao Shan, Owen Oertell, Kiant{\'e} Brantley, Dipendra Misra, Jason~D Lee, and Wen Sun.
\newblock Dataset reset policy optimization for {RLHF}.
\newblock \emph{arXiv:2404.08495}, 2024.

\bibitem[Chen and Jiang(2022)]{chen2022offline}
Jinglin Chen and Nan Jiang.
\newblock Offline reinforcement learning under value and density-ratio realizability: The power of gaps.
\newblock In \emph{Uncertainty in Artificial Intelligence}, 2022.

\bibitem[Chen et~al.(2022)Chen, Zhong, Yang, Wang, and Wang]{chen2022human}
Xiaoyu Chen, Han Zhong, Zhuoran Yang, Zhaoran Wang, and Liwei Wang.
\newblock Human-in-the-loop: Provably efficient preference-based reinforcement learning with general function approximation.
\newblock In \emph{International Conference on Machine Learning}, 2022.

\bibitem[Chen et~al.(2024)Chen, Deng, Yuan, Ji, and Gu]{chen2024self}
Zixiang Chen, Yihe Deng, Huizhuo Yuan, Kaixuan Ji, and Quanquan Gu.
\newblock Self-play fine-tuning converts weak language models to strong language models.
\newblock \emph{arXiv:2401.01335}, 2024.

\bibitem[Chernozhukov et~al.(2019)Chernozhukov, Demirer, Lewis, and Syrgkanis]{chernozhukov2019semi}
Victor Chernozhukov, Mert Demirer, Greg Lewis, and Vasilis Syrgkanis.
\newblock Semi-parametric efficient policy learning with continuous actions.
\newblock \emph{Advances in Neural Information Processing Systems}, 2019.

\bibitem[Christiano et~al.(2017)Christiano, Leike, Brown, Martic, Legg, and Amodei]{christiano2017deep}
Paul~F Christiano, Jan Leike, Tom Brown, Miljan Martic, Shane Legg, and Dario Amodei.
\newblock Deep reinforcement learning from human preferences.
\newblock \emph{Advances in Neural Information Processing Systems}, 2017.

\bibitem[Corless et~al.(1996)Corless, Gonnet, Hare, Jeffrey, and Knuth]{corless1996lambert}
Robert~M Corless, Gaston~H Gonnet, David~EG Hare, David~J Jeffrey, and Donald~E Knuth.
\newblock On the {L}ambert {W} function.
\newblock \emph{Advances in Computational Mathematics}, 1996.

\bibitem[Coste et~al.(2023)Coste, Anwar, Kirk, and Krueger]{coste2023reward}
Thomas Coste, Usman Anwar, Robert Kirk, and David Krueger.
\newblock Reward model ensembles help mitigate overoptimization.
\newblock \emph{arXiv:2310.02743}, 2023.

\bibitem[Cui and Du(2022)]{cui2022offline}
Qiwen Cui and Simon~S Du.
\newblock When are offline two-player zero-sum {M}arkov games solvable?
\newblock \emph{Advances in Neural Information Processing Systems}, 2022.

\bibitem[Das et~al.(2024)Das, Chakraborty, Pacchiano, and Chowdhury]{das2024provably}
Nirjhar Das, Souradip Chakraborty, Aldo Pacchiano, and Sayak~Ray Chowdhury.
\newblock Provably sample efficient {RLHF} via active preference optimization.
\newblock \emph{arXiv:2402.10500}, 2024.

\bibitem[de~Geer(2000)]{Sara00}
Sara A.~Van de~Geer.
\newblock \emph{Empirical Processes in {M}-{E}stimation.}
\newblock Cambridge University Press, 2000.

\bibitem[Dong et~al.(2023)Dong, Xiong, Goyal, Zhang, Chow, Pan, Diao, Zhang, Shum, and Zhang]{dong2023raft}
Hanze Dong, Wei Xiong, Deepanshu Goyal, Yihan Zhang, Winnie Chow, Rui Pan, Shizhe Diao, Jipeng Zhang, Kashun Shum, and Tong Zhang.
\newblock Raft: Reward ranked finetuning for generative foundation model alignment.
\newblock \emph{arXiv:2304.06767}, 2023.

\bibitem[Dong et~al.(2024)Dong, Xiong, Pang, Wang, Zhao, Zhou, Jiang, Sahoo, Xiong, and Zhang]{dong2024rlhf}
Hanze Dong, Wei Xiong, Bo~Pang, Haoxiang Wang, Han Zhao, Yingbo Zhou, Nan Jiang, Doyen Sahoo, Caiming Xiong, and Tong Zhang.
\newblock {RLHF} workflow: From reward modeling to online {RLHF}.
\newblock \emph{arXiv:2405.07863}, 2024.

\bibitem[Du et~al.(2024)Du, Winnicki, Dalal, Mannor, and Srikant]{du2024exploration}
Yihan Du, Anna Winnicki, Gal Dalal, Shie Mannor, and R~Srikant.
\newblock Exploration-driven policy optimization in {RLHF}: Theoretical insights on efficient data utilization.
\newblock \emph{arXiv:2402.10342}, 2024.

\bibitem[Duan et~al.(2020)Duan, Jia, and Wang]{duan2020minimax}
Yaqi Duan, Zeyu Jia, and Mengdi Wang.
\newblock Minimax-optimal off-policy evaluation with linear function approximation.
\newblock In \emph{International Conference on Machine Learning}, 2020.

\bibitem[Duchi and Namkoong(2019)]{duchi2019variance}
John Duchi and Hongseok Namkoong.
\newblock Variance-based regularization with convex objectives.
\newblock \emph{Journal of Machine Learning Research}, 2019.

\bibitem[Dud{\'\i}k et~al.(2015)Dud{\'\i}k, Hofmann, Schapire, Slivkins, and Zoghi]{dudik2015contextual}
Miroslav Dud{\'\i}k, Katja Hofmann, Robert~E Schapire, Aleksandrs Slivkins, and Masrour Zoghi.
\newblock Contextual dueling bandits.
\newblock In \emph{Conference on Learning Theory}, 2015.

\bibitem[Eisenstein et~al.(2023)Eisenstein, Nagpal, Agarwal, Beirami, D'Amour, Dvijotham, Fisch, Heller, Pfohl, Ramachandran, Shaw, and Berant]{eisenstein2023helping}
Jacob Eisenstein, Chirag Nagpal, Alekh Agarwal, Ahmad Beirami, Alex D'Amour, DJ~Dvijotham, Adam Fisch, Katherine Heller, Stephen Pfohl, Deepak Ramachandran, Peter Shaw, and Jonathan Berant.
\newblock Helping or herding? reward model ensembles mitigate but do not eliminate reward hacking.
\newblock \emph{arXiv:2312.09244}, 2023.

\bibitem[Farahmand et~al.(2010)Farahmand, Szepesv{\'a}ri, and Munos]{farahmand2010error}
Amir-massoud Farahmand, Csaba Szepesv{\'a}ri, and R{\'e}mi Munos.
\newblock Error propagation for approximate policy and value iteration.
\newblock \emph{Advances in Neural Information Processing Systems}, 2010.

\bibitem[Fisch et~al.(2024)Fisch, Eisenstein, Zayats, Agarwal, Beirami, Nagpal, Shaw, and Berant]{fisch2024robust}
Adam Fisch, Jacob Eisenstein, Vicky Zayats, Alekh Agarwal, Ahmad Beirami, Chirag Nagpal, Pete Shaw, and Jonathan Berant.
\newblock Robust preference optimization through reward model distillation.
\newblock \emph{arXiv:2405.19316}, 2024.

\bibitem[Fishburn(1984)]{fishburn1984probabilistic}
Peter~C Fishburn.
\newblock Probabilistic social choice based on simple voting comparisons.
\newblock \emph{The Review of Economic Studies}, 1984.

\bibitem[Foster and Rakhlin(2023)]{foster2023foundations}
Dylan~J Foster and Alexander Rakhlin.
\newblock Foundations of reinforcement learning and interactive decision making.
\newblock \emph{arXiv:2312.16730}, 2023.

\bibitem[Gabbianelli et~al.(2024)Gabbianelli, Neu, and Papini]{gabbianelli2024importance}
Germano Gabbianelli, Gergely Neu, and Matteo Papini.
\newblock Importance-weighted offline learning done right.
\newblock In \emph{International Conference on Algorithmic Learning Theory}, 2024.

\bibitem[Gao et~al.(2023)Gao, Schulman, and Hilton]{gao2023scaling}
Leo Gao, John Schulman, and Jacob Hilton.
\newblock Scaling laws for reward model overoptimization.
\newblock In \emph{International Conference on Machine Learning}, 2023.

\bibitem[Gao et~al.(2024)Gao, Chang, Zhan, Oertell, Swamy, Brantley, Joachims, Bagnell, Lee, and Sun]{gao2024rebel}
Zhaolin Gao, Jonathan~D Chang, Wenhao Zhan, Owen Oertell, Gokul Swamy, Kiant{\'e} Brantley, Thorsten Joachims, J~Andrew Bagnell, Jason~D Lee, and Wen Sun.
\newblock {REBEL}: Reinforcement learning via regressing relative rewards.
\newblock \emph{arXiv:2404.16767}, 2024.

\bibitem[Google(2023)]{anil2023palm}
Google.
\newblock Palm 2 technical report.
\newblock \emph{arXiv:2305.10403}, 2023.

\bibitem[Guo et~al.(2024)Guo, Zhang, Liu, Liu, Khalman, Llinares, Rame, Mesnard, Zhao, Piot, Ferret, and Blondel]{guo2024direct}
Shangmin Guo, Biao Zhang, Tianlin Liu, Tianqi Liu, Misha Khalman, Felipe Llinares, Alexandre Rame, Thomas Mesnard, Yao Zhao, Bilal Piot, Johan Ferret, and Mathieu Blondel.
\newblock Direct language model alignment from online {AI} feedback.
\newblock \emph{arXiv:2402.04792}, 2024.

\bibitem[Huang et~al.(2022)Huang, Dossa, Ye, Braga, Chakraborty, Mehta, and Araújo]{huang2022cleanrl}
Shengyi Huang, Rousslan Fernand~Julien Dossa, Chang Ye, Jeff Braga, Dipam Chakraborty, Kinal Mehta, and João~G.M. Araújo.
\newblock Cleanrl: High-quality single-file implementations of deep reinforcement learning algorithms.
\newblock \emph{Journal of Machine Learning Research}, 2022.

\bibitem[Ji et~al.(2024)Ji, Kulkarni, Wang, and Xie]{ji2024selfplay}
Xiang Ji, Sanjeev Kulkarni, Mengdi Wang, and Tengyang Xie.
\newblock Self-play with adversarial critic: Provable and scalable offline alignment for language models.
\newblock \emph{arXiv:2406.04274}, 2024.

\bibitem[Jin et~al.(2021)Jin, Yang, and Wang]{jin2021pessimism}
Ying Jin, Zhuoran Yang, and Zhaoran Wang.
\newblock Is pessimism provably efficient for offline {RL}?
\newblock In \emph{International Conference on Machine Learning}, 2021.

\bibitem[Kallus and Uehara(2020)]{kallus2020double}
Nathan Kallus and Masatoshi Uehara.
\newblock Double reinforcement learning for efficient off-policy evaluation in markov decision processes.
\newblock \emph{Journal of Machine Learning Research}, 2020.

\bibitem[Kramer(1973)]{kramer1973class}
Gerald~H Kramer.
\newblock On a class of equilibrium conditions for majority rule.
\newblock \emph{Econometrica: Journal of the Econometric Society}, 1973.

\bibitem[Kreweras(1965)]{kreweras1965aggregation}
Germain Kreweras.
\newblock Aggregation of preference orderings.
\newblock In \emph{Mathematics and Social Sciences I: Proceedings of the seminars of Menthon-Saint-Bernard, France and of G{\"o}sing, Austria}, 1965.

\bibitem[Lattimore and Szepesv{\'a}ri(2020)]{lattimore2020bandit}
Tor Lattimore and Csaba Szepesv{\'a}ri.
\newblock \emph{Bandit algorithms}.
\newblock Cambridge University Press, 2020.

\bibitem[Lee et~al.(2021)Lee, Jeon, Lee, Pineau, and Kim]{lee2021optidice}
Jongmin Lee, Wonseok Jeon, Byungjun Lee, Joelle Pineau, and Kee-Eung Kim.
\newblock Optidice: Offline policy optimization via stationary distribution correction estimation.
\newblock In \emph{International Conference on Machine Learning}, 2021.

\bibitem[Li et~al.(2023)Li, Yang, and Wang]{li2023reinforcement}
Zihao Li, Zhuoran Yang, and Mengdi Wang.
\newblock Reinforcement learning with human feedback: Learning dynamic choices via pessimism.
\newblock \emph{arXiv:2305.18438}, 2023.

\bibitem[Liu et~al.(2023)Liu, Zhao, Joshi, Khalman, Saleh, Liu, and Liu]{liu2023statistical}
Tianqi Liu, Yao Zhao, Rishabh Joshi, Misha Khalman, Mohammad Saleh, Peter~J Liu, and Jialu Liu.
\newblock Statistical rejection sampling improves preference optimization.
\newblock \emph{arXiv:2309.06657}, 2023.

\bibitem[Liu et~al.(2020)Liu, Swaminathan, Agarwal, and Brunskill]{liu2020provably}
Yao Liu, Adith Swaminathan, Alekh Agarwal, and Emma Brunskill.
\newblock Provably good batch off-policy reinforcement learning without great exploration.
\newblock \emph{Advances in Neural Information Processing Systems}, 2020.

\bibitem[Liu et~al.(2024)Liu, Lu, Zhang, Liu, Guo, Yang, Blanchet, and Wang]{liu2024provably}
Zhihan Liu, Miao Lu, Shenao Zhang, Boyi Liu, Hongyi Guo, Yingxiang Yang, Jose Blanchet, and Zhaoran Wang.
\newblock Provably mitigating overoptimization in {RLHF}: Your {SFT} loss is implicitly an adversarial regularizer.
\newblock \emph{arXiv:2405.16436}, 2024.

\bibitem[Ma et~al.(2022{\natexlab{a}})Ma, Yan, Jayaraman, and Bastani]{ma2022offline}
Jason~Yecheng Ma, Jason Yan, Dinesh Jayaraman, and Osbert Bastani.
\newblock Offline goal-conditioned reinforcement learning via $ f $-advantage regression.
\newblock \emph{Advances in Neural Information Processing Systems}, 2022{\natexlab{a}}.

\bibitem[Ma et~al.(2022{\natexlab{b}})Ma, Shen, Jayaraman, and Bastani]{ma2022smodice}
Yecheng~Jason Ma, Andrew Shen, Dinesh Jayaraman, and Osbert Bastani.
\newblock Smodice: Versatile offline imitation learning via state occupancy matching.
\newblock \emph{arXiv:2202.02433}, 2022{\natexlab{b}}.

\bibitem[Michaud et~al.(2020)Michaud, Gleave, and Russell]{michaud2020understanding}
Eric~J Michaud, Adam Gleave, and Stuart Russell.
\newblock Understanding learned reward functions.
\newblock \emph{arXiv:2012.05862}, 2020.

\bibitem[Moskovitz et~al.(2023)Moskovitz, Singh, Strouse, Sandholm, Salakhutdinov, Dragan, and McAleer]{moskovitz2023confronting}
Ted Moskovitz, Aaditya~K Singh, DJ~Strouse, Tuomas Sandholm, Ruslan Salakhutdinov, Anca~D Dragan, and Stephen McAleer.
\newblock Confronting reward model overoptimization with constrained {RLHF}.
\newblock \emph{arXiv:2310.04373}, 2023.

\bibitem[Munos et~al.(2023)Munos, Valko, Calandriello, Azar, Rowland, Guo, Tang, Geist, Mesnard, Michi, Selvi, Girgin, Momchev, Bachem, Mankowitz, Precup, and Piot]{munos2023nash}
R{\'e}mi Munos, Michal Valko, Daniele Calandriello, Mohammad~Gheshlaghi Azar, Mark Rowland, Zhaohan~Daniel Guo, Yunhao Tang, Matthieu Geist, Thomas Mesnard, Andrea Michi, Marco Selvi, Sertan Girgin, Nikola Momchev, Olivier Bachem, Daniel~J. Mankowitz, Doina Precup, and Bilal Piot.
\newblock Nash learning from human feedback.
\newblock \emph{arXiv:2312.00886}, 2023.

\bibitem[Novoseller et~al.(2020)Novoseller, Wei, Sui, Yue, and Burdick]{novoseller2020dueling}
Ellen Novoseller, Yibing Wei, Yanan Sui, Yisong Yue, and Joel Burdick.
\newblock Dueling posterior sampling for preference-based reinforcement learning.
\newblock In \emph{Conference on Uncertainty in Artificial Intelligence}, 2020.

\bibitem[OpenAI(2023)]{achiam2023gpt}
OpenAI.
\newblock Gpt-4 technical report.
\newblock \emph{arXiv:2303.08774}, 2023.

\bibitem[Ouyang et~al.(2022)Ouyang, Wu, Jiang, Almeida, Wainwright, Mishkin, Zhang, Agarwal, Slama, Ray, Schulman, Hilton, Kelton, Miller, Simens, Askell, Welinder, Christiano, Leike, and Lowe]{ouyang2022training}
Long Ouyang, Jeffrey Wu, Xu~Jiang, Diogo Almeida, Carroll Wainwright, Pamela Mishkin, Chong Zhang, Sandhini Agarwal, Katarina Slama, Alex Ray, John Schulman, Jacob Hilton, Fraser Kelton, Luke Miller, Maddie Simens, Amanda Askell, Peter Welinder, Paul Christiano, Jan Leike, and Ryan Lowe.
\newblock Training language models to follow instructions with human feedback.
\newblock \emph{Advances in Neural Information Processing Systems}, 2022.

\bibitem[Pacchiano et~al.(2021)Pacchiano, Saha, and Lee]{pacchiano2021dueling}
Aldo Pacchiano, Aadirupa Saha, and Jonathan Lee.
\newblock Dueling {RL}: Reinforcement learning with trajectory preferences.
\newblock \emph{arXiv:2111.04850}, 2021.

\bibitem[Pal et~al.(2024)Pal, Karkhanis, Dooley, Roberts, Naidu, and White]{pal2024smaug}
Arka Pal, Deep Karkhanis, Samuel Dooley, Manley Roberts, Siddartha Naidu, and Colin White.
\newblock Smaug: Fixing failure modes of preference optimisation with {DPO}-positive.
\newblock \emph{arXiv:2402.13228}, 2024.

\bibitem[Rafailov et~al.(2023)Rafailov, Sharma, Mitchell, Manning, Ermon, and Finn]{rafailov2024direct}
Rafael Rafailov, Archit Sharma, Eric Mitchell, Christopher~D Manning, Stefano Ermon, and Chelsea Finn.
\newblock Direct preference optimization: Your language model is secretly a reward model.
\newblock \emph{Advances in Neural Information Processing Systems}, 2023.

\bibitem[Rafailov et~al.(2024{\natexlab{a}})Rafailov, Chittepu, Park, Sikchi, Hejna, Knox, Finn, and Niekum]{rafailov2024scaling}
Rafael Rafailov, Yaswanth Chittepu, Ryan Park, Harshit Sikchi, Joey Hejna, Bradley Knox, Chelsea Finn, and Scott Niekum.
\newblock Scaling laws for reward model overoptimization in direct alignment algorithms.
\newblock \emph{arXiv:2406.02900}, 2024{\natexlab{a}}.

\bibitem[Rafailov et~al.(2024{\natexlab{b}})Rafailov, Hejna, Park, and Finn]{rafailov2024r}
Rafael Rafailov, Joey Hejna, Ryan Park, and Chelsea Finn.
\newblock {From $r$ to $Q^{\star}$}: Your language model is secretly a {Q}-function.
\newblock \emph{arXiv:2404.12358}, 2024{\natexlab{b}}.

\bibitem[Rashidinejad et~al.(2021)Rashidinejad, Zhu, Ma, Jiao, and Russell]{rashidinejad2021bridging}
Paria Rashidinejad, Banghua Zhu, Cong Ma, Jiantao Jiao, and Stuart Russell.
\newblock Bridging offline reinforcement learning and imitation learning: A tale of pessimism.
\newblock \emph{Advances in Neural Information Processing Systems}, 2021.

\bibitem[Rita et~al.(2024)Rita, Strub, Chaabouni, Michel, Dupoux, and Pietquin]{rita2024countering}
Mathieu Rita, Florian Strub, Rahma Chaabouni, Paul Michel, Emmanuel Dupoux, and Olivier Pietquin.
\newblock Countering reward over-optimization in {LLM} with demonstration-guided reinforcement learning.
\newblock \emph{arXiv:2404.19409}, 2024.

\bibitem[Rosset et~al.(2024)Rosset, Cheng, Mitra, Santacroce, Awadallah, and Xie]{rosset2024direct}
Corby Rosset, Ching-An Cheng, Arindam Mitra, Michael Santacroce, Ahmed Awadallah, and Tengyang Xie.
\newblock {Direct Nash Optimization}: Teaching language models to self-improve with general preferences.
\newblock \emph{arXiv:2404.03715}, 2024.

\bibitem[Shah et~al.(2015)Shah, Balakrishnan, Bradley, Parekh, Ramchandran, and Wainwright]{shah2015estimation}
Nihar Shah, Sivaraman Balakrishnan, Joseph Bradley, Abhay Parekh, Kannan Ramchandran, and Martin Wainwright.
\newblock {Estimation from Pairwise Comparisons: Sharp Minimax Bounds with Topology Dependence}.
\newblock In \emph{International Conference on Artificial Intelligence and Statistics}, 2015.

\bibitem[Simpson(1969)]{simpson1969defining}
Paul~B Simpson.
\newblock On defining areas of voter choice: Professor tullock on stable voting.
\newblock \emph{The Quarterly Journal of Economics}, 1969.

\bibitem[Song et~al.(2022)Song, Zhou, Sekhari, Bagnell, Krishnamurthy, and Sun]{song2022hybrid}
Yuda Song, Yifei Zhou, Ayush Sekhari, J~Andrew Bagnell, Akshay Krishnamurthy, and Wen Sun.
\newblock Hybrid {RL}: Using both offline and online data can make {RL} efficient.
\newblock \emph{arXiv:2210.06718}, 2022.

\bibitem[Song et~al.(2024)Song, Swamy, Singh, Bagnell, and Sun]{song2024understanding}
Yuda Song, Gokul Swamy, Aarti Singh, J~Andrew Bagnell, and Wen Sun.
\newblock Understanding preference fine-tuning through the lens of coverage.
\newblock \emph{arXiv:2406.01462}, 2024.

\bibitem[Stiennon et~al.(2020)Stiennon, Ouyang, Wu, Ziegler, Lowe, Voss, Radford, Amodei, and Christiano]{stiennon2020learning}
Nisan Stiennon, Long Ouyang, Jeffrey Wu, Daniel Ziegler, Ryan Lowe, Chelsea Voss, Alec Radford, Dario Amodei, and Paul~F Christiano.
\newblock Learning to summarize with human feedback.
\newblock \emph{Advances in Neural Information Processing Systems}, 33, 2020.

\bibitem[Swamy et~al.(2024)Swamy, Dann, Kidambi, Wu, and Agarwal]{swamy2024minimaximalist}
Gokul Swamy, Christoph Dann, Rahul Kidambi, Zhiwei~Steven Wu, and Alekh Agarwal.
\newblock A minimaximalist approach to reinforcement learning from human feedback.
\newblock \emph{arXiv:2401.04056}, 2024.

\bibitem[Tajwar et~al.(2024)Tajwar, Singh, Sharma, Rafailov, Schneider, Xie, Ermon, Finn, and Kumar]{tajwar2024preference}
Fahim Tajwar, Anikait Singh, Archit Sharma, Rafael Rafailov, Jeff Schneider, Tengyang Xie, Stefano Ermon, Chelsea Finn, and Aviral Kumar.
\newblock Preference fine-tuning of {LLMs} should leverage suboptimal, on-policy data.
\newblock \emph{arXiv:2404.14367}, 2024.

\bibitem[Tang et~al.(2024)Tang, Guo, Zheng, Calandriello, Munos, Rowland, Richemond, Valko, Pires, and Piot]{tang2024generalized}
Yunhao Tang, Zhaohan~Daniel Guo, Zeyu Zheng, Daniele Calandriello, R{\'e}mi Munos, Mark Rowland, Pierre~Harvey Richemond, Michal Valko, Bernardo~{\'A}vila Pires, and Bilal Piot.
\newblock Generalized preference optimization: A unified approach to offline alignment.
\newblock \emph{arXiv:2402.05749}, 2024.

\bibitem[Tien et~al.(2022)Tien, He, Erickson, Dragan, and Brown]{tien2022causal}
Jeremy Tien, Jerry Zhi-Yang He, Zackory Erickson, Anca Dragan, and Daniel~S Brown.
\newblock Causal confusion and reward misidentification in preference-based reward learning.
\newblock In \emph{International Conference on Learning Representations}, 2022.

\bibitem[Touvron et~al.(2023)Touvron, Martin, Stone, Albert, Almahairi, Babaei, Bashlykov, Batra, Bhargava, Bhosale, Bikel, Blecher, Ferrer, Chen, Cucurull, Esiobu, Fernandes, Fu, Fu, Fuller, Gao, Goswami, Goyal, Hartshorn, Hosseini, Hou, Inan, Kardas, Kerkez, Khabsa, Kloumann, Korenev, Koura, Lachaux, Lavril, Lee, Liskovich, Lu, Mao, Martinet, Mihaylov, Mishra, Molybog, Nie, Poulton, Reizenstein, Rungta, Saladi, Schelten, Silva, Smith, Subramanian, Tan, Tang, Taylor, Williams, Kuan, Xu, Yan, Zarov, Zhang, Fan, Kambadur, Narang, Rodriguez, Stojnic, Edunov, and Scialom]{touvron2023llama}
Hugo Touvron, Louis Martin, Kevin Stone, Peter Albert, Amjad Almahairi, Yasmine Babaei, Nikolay Bashlykov, Soumya Batra, Prajjwal Bhargava, Shruti Bhosale, Dan Bikel, Lukas Blecher, Cristian~Canton Ferrer, Moya Chen, Guillem Cucurull, David Esiobu, Jude Fernandes, Jeremy Fu, Wenyin Fu, Brian Fuller, Cynthia Gao, Vedanuj Goswami, Naman Goyal, Anthony Hartshorn, Saghar Hosseini, Rui Hou, Hakan Inan, Marcin Kardas, Viktor Kerkez, Madian Khabsa, Isabel Kloumann, Artem Korenev, Punit~Singh Koura, Marie-Anne Lachaux, Thibaut Lavril, Jenya Lee, Diana Liskovich, Yinghai Lu, Yuning Mao, Xavier Martinet, Todor Mihaylov, Pushkar Mishra, Igor Molybog, Yixin Nie, Andrew Poulton, Jeremy Reizenstein, Rashi Rungta, Kalyan Saladi, Alan Schelten, Ruan Silva, Eric~Michael Smith, Ranjan Subramanian, Xiaoqing~Ellen Tan, Binh Tang, Ross Taylor, Adina Williams, Jian~Xiang Kuan, Puxin Xu, Zheng Yan, Iliyan Zarov, Yuchen Zhang, Angela Fan, Melanie Kambadur, Sharan Narang, Aurelien Rodriguez, Robert Stojnic, Sergey Edunov, and Thomas Scialom.
\newblock Llama 2: Open foundation and fine-tuned chat models.
\newblock \emph{arXiv:2307.09288}, 2023.

\bibitem[Tsybakov(2008)]{tsybakov2008introduction}
Alexandre~B Tsybakov.
\newblock \emph{Introduction to Nonparametric Estimation}.
\newblock Springer, 2008.

\bibitem[Uehara and Sun(2021)]{uehara2021pessimistic}
Masatoshi Uehara and Wen Sun.
\newblock Pessimistic model-based offline reinforcement learning under partial coverage.
\newblock \emph{arXiv:2107.06226}, 2021.

\bibitem[Van~Erven and Harremos(2014)]{van2014renyi}
Tim Van~Erven and Peter Harremos.
\newblock R{\'e}nyi divergence and kullback-leibler divergence.
\newblock \emph{IEEE Transactions on Information Theory}, 60\penalty0 (7), 2014.

\bibitem[von Werra et~al.(2020)von Werra, Belkada, Tunstall, Beeching, Thrush, Lambert, and Huang]{vonwerra2022trl}
Leandro von Werra, Younes Belkada, Lewis Tunstall, Edward Beeching, Tristan Thrush, Nathan Lambert, and Shengyi Huang.
\newblock Trl: Transformer reinforcement learning.
\newblock \url{https://github.com/huggingface/trl}, 2020.

\bibitem[Wang et~al.(2023{\natexlab{a}})Wang, Jiang, Yang, Liu, and Chen]{wang2023beyond}
Chaoqi Wang, Yibo Jiang, Chenghao Yang, Han Liu, and Yuxin Chen.
\newblock Beyond reverse {KL}: Generalizing direct preference optimization with diverse divergence constraints.
\newblock \emph{arXiv:2309.16240}, 2023{\natexlab{a}}.

\bibitem[Wang et~al.(2024)Wang, Krishnamurthy, and Slivkins]{wang2024oracle}
Lequn Wang, Akshay Krishnamurthy, and Alex Slivkins.
\newblock Oracle-efficient pessimism: Offline policy optimization in contextual bandits.
\newblock In \emph{International Conference on Artificial Intelligence and Statistics}, 2024.

\bibitem[Wang et~al.(2023{\natexlab{b}})Wang, Liu, and Jin]{wang2023rlhf}
Yuanhao Wang, Qinghua Liu, and Chi Jin.
\newblock Is {RLHF} more difficult than standard {RL}?
\newblock \emph{arXiv:2306.14111}, 2023{\natexlab{b}}.

\bibitem[Wong and Shen(1995)]{wong1995probability}
Wing~Hung Wong and Xiaotong Shen.
\newblock Probability inequalities for likelihood ratios and convergence rates of sieve mles.
\newblock \emph{The Annals of Statistics}, 1995.

\bibitem[Wu and Sun(2023)]{wu2023making}
Runzhe Wu and Wen Sun.
\newblock Making {RL} with preference-based feedback efficient via randomization.
\newblock \emph{arXiv:2310.14554}, 2023.

\bibitem[Wu et~al.(2024)Wu, Sun, Yuan, Ji, Yang, and Gu]{wu2024self}
Yue Wu, Zhiqing Sun, Huizhuo Yuan, Kaixuan Ji, Yiming Yang, and Quanquan Gu.
\newblock Self-play preference optimization for language model alignment.
\newblock \emph{arXiv:2405.00675}, 2024.

\bibitem[Xie and Jiang(2020)]{xie2020q}
Tengyang Xie and Nan Jiang.
\newblock Q* approximation schemes for batch reinforcement learning: A theoretical comparison.
\newblock In \emph{Conference on Uncertainty in Artificial Intelligence}, 2020.

\bibitem[Xie et~al.(2021)Xie, Cheng, Jiang, Mineiro, and Agarwal]{xie2021bellman}
Tengyang Xie, Ching-An Cheng, Nan Jiang, Paul Mineiro, and Alekh Agarwal.
\newblock Bellman-consistent pessimism for offline reinforcement learning.
\newblock \emph{Advances in Neural Information Processing Systems}, 2021.

\bibitem[Xie et~al.(2024)Xie, Foster, Krishnamurthy, Rosset, Awadallah, and Rakhlin]{xie2024exploratory}
Tengyang Xie, Dylan~J Foster, Akshay Krishnamurthy, Corby Rosset, Ahmed Awadallah, and Alexander Rakhlin.
\newblock Exploratory preference optimization: Harnessing implicit {Q}*-approximation for sample-efficient rlhf.
\newblock \emph{arXiv:2405.21046}, 2024.

\bibitem[Xiong et~al.(2023)Xiong, Dong, Ye, Zhong, Jiang, and Zhang]{xiong2023gibbs}
Wei Xiong, Hanze Dong, Chenlu Ye, Han Zhong, Nan Jiang, and Tong Zhang.
\newblock Gibbs sampling from human feedback: A provable {KL}-constrained framework for {RLHF}.
\newblock \emph{arXiv:2312.11456}, 2023.

\bibitem[Xu et~al.(2020)Xu, Wang, Yang, Singh, and Dubrawski]{xu2020preference}
Yichong Xu, Ruosong Wang, Lin Yang, Aarti Singh, and Artur Dubrawski.
\newblock Preference-based reinforcement learning with finite-time guarantees.
\newblock \emph{Advances in Neural Information Processing Systems}, 2020.

\bibitem[Ye et~al.(2024)Ye, Xiong, Zhang, Jiang, and Zhang]{ye2024theoretical}
Chenlu Ye, Wei Xiong, Yuheng Zhang, Nan Jiang, and Tong Zhang.
\newblock A theoretical analysis of {Nash} learning from human feedback under general {KL}-regularized preference.
\newblock \emph{arXiv:2402.07314}, 2024.

\bibitem[Yuan et~al.(2024)Yuan, Cui, Wang, Ding, Wang, Deng, Shan, Chen, Xie, Lin, Liu, Zhou, Peng, Liu, and Sun]{yuan2024advancing}
Lifan Yuan, Ganqu Cui, Hanbin Wang, Ning Ding, Xingyao Wang, Jia Deng, Boji Shan, Huimin Chen, Ruobing Xie, Yankai Lin, Zhenghao Liu, Bowen Zhou, Hao Peng, Zhiyuan Liu, and Maosong Sun.
\newblock Advancing llm reasoning generalists with preference trees.
\newblock \emph{arXiv:2404.02078}, 2024.

\bibitem[Zanette et~al.(2021)Zanette, Wainwright, and Brunskill]{zanette2021provable}
Andrea Zanette, Martin~J Wainwright, and Emma Brunskill.
\newblock Provable benefits of actor-critic methods for offline reinforcement learning.
\newblock \emph{Advances in Neural Information Processing Systems}, 2021.

\bibitem[Zhan et~al.(2022)Zhan, Huang, Huang, Jiang, and Lee]{zhan2022offline}
Wenhao Zhan, Baihe Huang, Audrey Huang, Nan Jiang, and Jason Lee.
\newblock Offline reinforcement learning with realizability and single-policy concentrability.
\newblock In \emph{Conference on Learning Theory}, 2022.

\bibitem[Zhan et~al.(2023{\natexlab{a}})Zhan, Uehara, Kallus, Lee, and Sun]{zhan2023provable}
Wenhao Zhan, Masatoshi Uehara, Nathan Kallus, Jason~D Lee, and Wen Sun.
\newblock Provable offline preference-based reinforcement learning.
\newblock In \emph{International Conference on Learning Representations}, 2023{\natexlab{a}}.

\bibitem[Zhan et~al.(2023{\natexlab{b}})Zhan, Uehara, Sun, and Lee]{zhan2023query}
Wenhao Zhan, Masatoshi Uehara, Wen Sun, and Jason~D Lee.
\newblock Provable reward-agnostic preference-based reinforcement learning.
\newblock \emph{arXiv:2305.18505}, 2023{\natexlab{b}}.

\bibitem[Zhang(2006)]{zhang2006from}
Tong Zhang.
\newblock From $\epsilon$-entropy to {KL}-entropy: Analysis of minimum information complexity density estimation.
\newblock \emph{The Annals of Statistics}, 2006.

\bibitem[Zhang et~al.(2024)Zhang, Ton, Shen, Wang, and Liu]{zhang2024overcoming}
Xiaoying Zhang, Jean-Francois Ton, Wei Shen, Hongning Wang, and Yang Liu.
\newblock Overcoming reward overoptimization via adversarial policy optimization with lightweight uncertainty estimation.
\newblock \emph{arXiv:2403.05171}, 2024.

\bibitem[Zhu et~al.(2023)Zhu, Jordan, and Jiao]{zhu2023principled}
Banghua Zhu, Michael Jordan, and Jiantao Jiao.
\newblock Principled reinforcement learning with human feedback from pairwise or k-wise comparisons.
\newblock In \emph{International Conference on Machine Learning}, 2023.

\bibitem[Zhu et~al.(2024)Zhu, Jordan, and Jiao]{zhu2024iterative}
Banghua Zhu, Michael~I Jordan, and Jiantao Jiao.
\newblock Iterative data smoothing: Mitigating reward overfitting and overoptimization in {RLHF}.
\newblock \emph{arXiv:2401.16335}, 2024.

\bibitem[Zhu and Zhang(2024)]{zhu2024provably}
Hanlin Zhu and Amy Zhang.
\newblock Provably efficient offline goal-conditioned reinforcement learning with general function approximation and single-policy concentrability.
\newblock \emph{Advances in Neural Information Processing Systems}, 2024.

\bibitem[Zhu et~al.(2020)Zhu, Lin, Dai, and Zhou]{zhu2020off}
Zhuangdi Zhu, Kaixiang Lin, Bo~Dai, and Jiayu Zhou.
\newblock Off-policy imitation learning from observations.
\newblock \emph{Advances in Neural Information Processing Systems}, 2020.

\end{thebibliography}

\clearpage

\appendix

\renewcommand{\contentsname}{Contents of Appendix}
\addtocontents{toc}{\protect\setcounter{tocdepth}{2}}
{
  \hypersetup{hidelinks}
  \tableofcontents
}

\part{Additional Results}

\section{Additional Related Work}
\label{sec:related}

\paragraph{Theoretical algorithms for offline alignment}
Much of prior theoretical work on offline alignment considers algorithms
that are tailored to linearly parameterized policies
\citep{zhu2023principled,li2023reinforcement,xiong2023gibbs}, 
while others are not efficiently implementable, e.g., as they require solving min-max problems over a version space \citep{zhan2023provable}.
For general policy classes, \citet{ye2024theoretical} provide an
algorithm that achieves sample complexity guarantees based on
single-policy concentrability, but the algorithm requires computation
of an uncertainty bonus which cannot be implemented faithfully for
large language models. \citet{ji2024selfplay} provide an algorithm
that achieves single-policy concentrability using self-play, but their
approach requires the non-standard realizability assumption that for
all $\pi\in\Pi$, there exists $\pi'\in\Pi$ such that $r(x,a) =
\beta\log\frac{\pi(a\mid{}x)}{\pi'(a\mid{}x)}-Z_{\pi,\pi'}(x)$ for some
function $Z_{\pi,\pi'}(x)$ that depends on $x$, but not the
action $a$. In addition, their algorithm is iterative, and requires solving a
DPO-like objective many times (roughly $1/\veps^{2}$ iterations are
required to achieve accuracy $\veps$). Most
relevant to our work,
\citet{liu2024provably,cen2024value,fisch2024robust} propose solving
the appealingly simple DPO + SFT objective in \cref{eq:rpo}. As we
discuss in detail in \cref{sec:rpo_lower}, this objective fails to
achieve single-policy concentrability unless non-standard convexity assumptions
on the policy class or reward model class hold.

A number of other works consider the \emph{hybrid} setting for alignment
where---in addition to offline preference data from $\piref$, the
algorithm has access to online feedback
\citep{xiong2023gibbs,gao2024rebel,chang2024dataset,song2024understanding}. While it is
straightforward to achieve guarantees based on single-policy concentrability in this setting,
this is a stronger feedback model than what we consider, and is not
always realistic. Our work is
also complementary to fully online alignment, which dispenses with coverage
conditions entirely but requires active exploration
\citep{xu2020preference,novoseller2020dueling,pacchiano2021dueling,wu2023making,zhan2023query,chen2022human,wang2023rlhf,du2024exploration,das2024provably,ye2024theoretical,xie2024exploratory,cen2024value}.

\paragraph{Generalizations of \dpo}
\citet{wang2023beyond} provide a generalization of the \dpo reparameterization
trick which supports general $f$-divergences that satisfy certain
regularity conditions. Their work does not provide sample
complexity guarantees or theoretical guidance on which choices of
$f$-divergence are preferable, but our main algorithm \algshort, can be derived as
a special case of their technique with a novel choice of
$f$-divergence. \citet{tang2024generalized} also provide a general
framework for deriving \dpo variants with general loss functions, but
our algorithm does not appear to be a special case of their framework.

\paragraph{Offline reinforcement learning theory}
The theory of \emph{offline
  reinforcement learning} addresses challenges similar to
overoptimization, which is typically
describes through the language of \emph{distribution shift}. Many of
these works, using pessimism and related algorithmic techniques, provide guarantees
that are robust to partial coverage of the data collection policy
$\piref$, which is reflected in sample complexity guarantees based on
single-policy concentrability and similar coverage conditions. While
this line of work provides efficient algorithms for simple (e.g.,
tabular or linear) settings
\citep{liu2020provably,jin2021pessimism,rashidinejad2021bridging},
existing approaches that support general function approximation
\citep{xie2021bellman,uehara2021pessimistic,zhan2022offline,chen2022offline}
cannot be implemented efficiently for language models without
non-trivial modifications. See also closely related research on policy
optimization and evaluation in statistics and econometrics
\citep{athey2021policy,chernozhukov2019semi,kallus2020double}.

\paragraph{\chis-divergence in reinforcement learning}
Our work contributes to a growing body of research that uses
\chis-divergence to derive reinforcement learning algorithms with novel statistical
guarantees.\footnote{More classically, \chis-divergence is known to play a
fundamental role in asymptotic statistics
\citep{tsybakov2008introduction,duchi2019variance}.} Notably, our work is inspired by
\citet{wang2024oracle} (see also \citet{gabbianelli2024importance}),
who use a regularizer similar to \chis-divergence to derive
single-policy concentrability guarantees for contextual
bandits. Compared to the \chis-regularizer $\cC^{\pi} =
\En_{\pi}\brk*{\frac{\pi(a\mid{}x)}{\piref(a\mid{}x)}}$ we use, their
regularizer takes the form
$\En_{\pi}\brk*{\frac{1}{\piref(a\mid{}x)}}$, which is always larger. As a result of this
diference, their regularizer is not suitable for large action
spaces. By addressing this shortcoming, we expect our
\chis-regularization approach to find further use in offline RL.

Other related works include (i) \citet{duan2020minimax} show that
\chis-divergence plays a fundamental role in offline RL with linear
function approximation; (ii) \cite{zhan2022offline} use
\chis-regularization to provide guarantees based on single-policy
concentrability for an offline RL method based on weight function
learning; and (iii) \citet{amortila2024scalable} provide online RL
algorithms that explore by directly minimizing an exploration
objective based on \chis-divergence. We mention in passing that a
number of recent empirical works
apply \chis-regularization \citep{zhu2020off,lee2021optidice,
  ma2022offline,ma2022smodice,zhu2024provably} to
reinforcement learning in embodied domains. 
Lastly, \citet{cesa2017boltzmann} prove lower bounds against 
the softmax policy distribution, but in the context of online exploration
for online RL. While this is different problem setting than ours, 
their construction may be in similar in spirit to our lower bound against
KL-regularization in offline reinforcement learning (\cref{prop:rpo_lower}).

\paragraph{Empirical research on offline alignment}
Our work uses \dpo \citep{rafailov2024direct} as a starting
point. Many prior works have built upon \dpo with the aim
of addressing specific shortcomings, including
\citet{liu2023statistical,tang2024generalized,azar2024general,rosset2024direct,chen2024self,wu2024self,tajwar2024preference}. Closely
related, there is a large body of research that attempts to understand and
mitigate overoptimization in offline alignment from a purely empirical perspective \citep{michaud2020understanding,tien2022causal,coste2023reward,dong2023raft,eisenstein2023helping,gao2023scaling,moskovitz2023confronting,pal2024smaug,rita2024countering,rafailov2024scaling,zhang2024overcoming}.

\subsection{Detailed Comparison to DPO + SFT}
\label{sec:rpo_lower}

\newcommand{\piminmax}{\pihat_{\normalfont\textsf{min-max}}}
\newcommand{\rminmax}{\rhat_{\normalfont\textsf{min-max}}}
\newcommand{\pimaxmin}{\pihat_{\normalfont\textsf{max-min}}}
\newcommand{\pidposft}{\pihat_{\normalfont\textsf{DPO+SFT}}}

In this section, we give additional background on the suboptimality of
the DPO + SFT objective in
\cref{eq:rpo}. Let $\beta>0$ be the KL-regularization parameter and
$\alpha>0$ be an optimism parameter. Consider the setting in which
$\Pi=\crl*{\pi_r(a\mid{}x)=\piref(a \mid{} x)\exp(\beta^{-1}(r(x,a)-Z_r(x)))\mid{}r\in\cR}$ 
for a reward class
$\cR\subset(\cX\times\cA\to\bbR)$. \citet{liu2024provably,cen2024value,fisch2024robust} propose
solving (variants of) the objective
\begin{align}
  \label{eq:maxmin}
  \pimaxmin
  =
\argmax_{\pi}\min_{r\in\cR}\crl*{\alpha\prn*{\En_{x\sim\rho{},a\sim\pi(\cdot\mid{}x),b\sim\piref(\cdot\mid{}x)}\brk*{r(a)-r(b)
    }-\beta\Dkl{\pi}{\piref}}
  +\cL(r)},
\end{align}
where the max ranges over the space of all policies, and where $\cL(r)\ldef{}-\frac{1}{n}\sum_{(x,\ap,\am)\in\cD_\pref}\log\sigma\brk*{r(x,\ap)-r(x,\am)}$
is the negative log-likelihood under the Bradley-Terry
model. \citet{liu2024provably} show that for general policy classes,
this algorithm attains sample complexity guarantees scaling with
single-policy concentrability; \citet{cen2024value} provide similar
results for the special case of linearly parameterized policies.

The objective in \cref{eq:maxmin} is non-trivial to implement for
language models. To derive the DPO + SFT objective in \cref{eq:rpo},
\citet{liu2024provably} observe that if $\cR$ is convex, the minimax
theorem implies that the objective value in \cref{eq:maxmin} is
equivalent to the value for the min-max objective
\begin{align}
  \label{eq:minmax}
  \min_{r\in\cR}\max_{\pi}\crl*{\alpha\prn*{\En_{x\sim\rho{},a\sim{}\pi(\cdot\mid
  x),b\sim\piref(\cdot\mid x)}\brk*{r(a)-r(b)
    }-\beta\Dkl{\pi}{\piref}}
  +\cL(r)}.
\end{align}
This leads to a natural algorithmic strategy adopted by~\citep{liu2024provably,cen2024value,fisch2024robust}: Let $\rminmax$ be the minimizing reward function in~\cref{eq:minmax} and let $\pi_{\rminmax}$---the optimal policy in the KL-regularized MDP with reward function $\rminmax$---be the final policy returned by the algorithm. 
After standard manipulations, one can then show that $\pi_{\rminmax}$ is
equivalent to
\begin{align}
  \label{eq:rpo2}
  \argmax_{\pi\in\Pi}\crl*{
  \alpha\cdot\En_{\piref}\brk*{\beta\log\pi(a\mid x)}
  + 
\frac{1}{n}\sum_{(x,\ap,\am) \in \Dcal_\pref} \log\left[\sigma\left(
  \beta\log\frac{\pi(\ap\mid x)}{\piref(\ap \mid x)} -
  \beta\log\frac{\pi(\am \mid x)}{\piref(\am \mid x)} \right) \right]
  }.
\end{align}

We call this policy $\pihat_{\textsf{DPO+SFT}}$. The sample complexity
analyses for the $\pihat_{\textsf{DPO+SFT}}$ policy (\cref{eq:rpo2})
in \citep{liu2024provably,cen2024value} rely on
showing that the objective value in \cref{eq:minmax} is equivalent to
the value in \cref{eq:maxmin}, which is not guaranteed to hold if
$\cR$ is non-convex (e.g., if $\cR$ is a class of neural
networks).\footnote{Precisely,~\citet{liu2024provably} provide
guarantees for $\pimaxmin$ with general reward class $\cR$
and establish equivalence of $\pimaxmin$ and $\piminmax$ when $\cR$ is
convex, while~\citet{cen2024value} consider linear function approximation, which yields the required convexity.} Indeed, the following proposition shows that, for non-convex
reward classes $\cR$, the DPO + SFT objective in \cref{eq:rpo2} fails to achieve a statistical guarantee based on single-policy concentrability, even when
\cref{eq:maxmin} succeeds.
\begin{proposition}
  \label{prop:rpo_lower}
  Let $n \in \bbN$ with $n \geq 2$ be given. There exists a reward class
  $\cR$ with $|\cR| = 2$, a problem instance $(\rho,r)$ satisfying
  realizability ($r\in\cR$) and $r\in\brk{0,1}$, a data collection policy $\piref$,
  and universal constants $c_1 \in (0,1)$ and $c_2,c_3>0$ 
  such that the following hold:
  \begin{enumerate}
    \item There exists a policy $\widetilde{\pi}$ such that $\|
      \widetilde{\pi}/\piref \|_{\infty} \leq 2$; yet
    \item For any $\beta \leq (2\log(n))^{-1}$ and $\alpha\geq{}0$,
      the minimax policy $\piminmax$ (\cref{eq:minmax}) and DPO+SFT
      policy $\pidposft$ (\cref{eq:rpo2}) derived from a
      dataset $\cD_\pref$ of $n$ samples from $\piref$ incur suboptimality
      \begin{align*}
        J(\widetilde{\pi}) - J(\pidposft) =  J(\widetilde{\pi}) - J(\piminmax) \geq c_2,
      \end{align*}
      with probability at least $c_1$.
    \item For any $\beta \geq (2\log(n))^{-1}$ and $\alpha \geq 0$, the minimax policy $\piminmax$ (\cref{eq:minmax}) and DPO+SFT
      policy $\pidposft$ (\cref{eq:rpo2}) derived from a
      dataset $\cD_\pref$ of $n$ samples from $\piref$ incur suboptimality
      \begin{align*}
        J(\widetilde{\pi}) - J(\pidposft) =  J(\widetilde{\pi}) - J(\piminmax) \geq \frac{c_3}{\log(n)},
      \end{align*}
      with probability at least $c_1$. 
  \end{enumerate}
\end{proposition}
On the other hand, we observe that for the instance in
\cref{prop:rpo_lower}, \algshort (via \cref{thm:main}) with
$\beta\propto1/\sqrt{n}$ and the class $\Pi=\crl*{\pi(a\mid{}x)=\piref(a\mid{}x)\cdot\link^{-1}(\beta^{-1}(r(x,a)-Z_r(x)))\mid{}r\in\cR}$
achieves
\[
  J(\pitil) - J(\pihat) \approxleq \sqrt{\frac{(\cC^{\pitil})^2}{n}}\approxleq\sqrt{\frac{1}{n}},
\]
highlighting the fact that \algshort meaningfully adapts to
single-policy concentrability even when the technical conditions
required by DPO+SFT do not hold; see also \cref{sec:understanding}. We
find this conclusion to be somewhat surprising, as
\citet{xie2024exploratory} show that an \emph{optimistic} counterpart
to \cref{eq:rpo2}, which negates the SFT term, enjoys strong
guarantees for online alignment with general policy classes without
requiring convexity.

Although our construction does not establish inconsistency in the
$\beta\geq (2\log(n))^{-1}$ regime, in general, DPO+SFT will incur
$\bigoh(\beta)$ bias if one aims to compete with the optimal
policy. Due to restriction that $\beta$ must be rather large, this
results in an exponentially slower rate of convergence than
$\algshort$.

\begin{proof}[\pfref{prop:rpo_lower}]
    Let $n\in\bbN$ with $n\geq{}2$ be given. We consider a problem
    instance with $\cX=\crl{x_1,x_2}$ and $\cA=\crl*{a_0,a_1,a_2,a_3}$,
    so that $\abs*{\cA}=4$. We define a reward class with two reward
    functions $\cR:= \{r_1,r_2\}$ as follows. For $i\in\crl{1,2}$: 
    \begin{align*}
      &r_i(x_1,a_0) = \zeta , \mathand r_i(x_1,a_1) = r_i(x_1,a_2) = r_i(x_1,a_3) = 0\\
      &r_i(x_2,a_0) = 1/2, \quad r_i(x_2,a_i) = 1, \mathand r_i(x_2,a_j) = 0\;\; \forall j \ne i.
    \end{align*}
    Here $\zeta \in [0,1]$ will be chosen at the end of the proof.
    The context distribution is $\rho = \unif(\cX)$, and we define
    $\piref$ for each $x_i\in\crl{x_1,x_2}$ via
    \begin{align*}
      \piref(a_0 \mid x_i) = 1/2, \quad \piref(a_1 \mid x_i) = \piref(a_2 \mid x_i) = 1/(2n), \mathand \piref(a_3 \mid x_i) = (n-2)/(2n).
    \end{align*}
    Let $r_1$ be the true reward function. Recall that $\cD_{\pref}=\crl*{(x,\ap,\am)}$
consists of $n$ tuples $(x,\ap,\am)$ obtained by sampling $x \sim \rho$ and a pair of actions
$(a,b)\sim\piref$ and labeling them as $(\ap,\am)$ via the
Bradley-Terry model in \cref{eq:bt} with reward $r_1$. Define a
``bad'' event under this process:
\begin{align*}
  \cE \ldef{} \crl*{\text{No tuples in $\cD_\pref$ contain $a_1$ or $a_2$}}.
\end{align*}
We can lower bound the probability of $\cE$ as follows:
\begin{align*}
  \bbP[\cE^{\setcomp}] &\leq \bbP[\text{$a_1$ in $\cD_\pref$}] +
                         \bbP[\text{$a_2$ in $\cD_\pref$}] \\
                       &= 2(1 - (1-1/2n)^n)
  \leq{}  2(1 - e^{-1/2}(1 - 1/(4n))) \leq 2 ( 1 - 7e^{-1/2}/8) \leq 0.94,
\end{align*}
where the first inequality uses that $(1 - x/n)^n\geq e^{-x}(1 -
x^2/n)$ for $n\geq 1$ and $|x| < n$. We conclude that
\begin{align*}
  \bbP[\cE] \geq 0.06 =: c_1.
\end{align*}
Let $\cL(r; \cD_\pref)\ldef{}-\frac{1}{n}\sum_{(x,\ap,\am)\in\cD_\pref}\log\sigma\brk*{r(x,\ap)-r(x,\am)}$
denote the \dpo loss. Observe that conditioned on $\cE$, we have that $\cL(r_1;\cD_\pref) = \cL(r_2;\cD_\pref)$. Noting that
\begin{align*}
  \max_{\pi}\crl*{\bbE_{\pi}[r] - \bbE_{\piref}[r] - \beta \Dkl{\pi}{\piref}}
=   \bbE_{\pi_r}[r] - \bbE_{\piref}[r] - \beta \Dkl{\pi_r}{\piref},
\end{align*}
is the same for both $r \in \cR$, we see that both $r_{1}$ and $r_2$
optimize the minimax objective in \cref{eq:minmax}. Thus, breaking
ties adversarially, we can choose $\piminmax=\pi_{r_2}$ under $\cE$
for all values of $\beta>0$ and $\alpha\geq{}0$. By the equivalence
between the minimax objective in \cref{eq:minmax} and the DPO+SFT
objective
in~\cref{eq:rpo2}~\citep{liu2024provably,cen2024value,fisch2024robust},
for $\Pi=\crl*{\pi_{r_1},\pi_{r_2}}$, we can choose
$\pidposft=\pi_{r_2}$ in \cref{eq:rpo2} under $\cE$. Indeed, under
$\cE$, the DPO+SFT objective is equivalent to
$\argmax_{\pi\in\Pi}\En_{\piref}\brk*{\log\pi(a)}$, and $\pi_{r_1}$
and $\pi_{r_2}$ have the same value for this objective.

To conclude we choose $\pitil(\cdot)=a_0$, which has
$\nrm*{\pitil/\piref}_{\infty}=2$. It remains to calculate the
suboptimality gap. 
\[
J(\pitil) - J(\pidposft) =   J(\pitil) - J(\piminmax)
  =   J(\pitil) - J(\pi_{r_2})
\]
under $\cE$. Note that $J(\pitil) = \zeta/2 + 1/4$. We decompose the reward for $\pi_{r_2}$ on instance $r_1$ into two components, corresponding to the two contexts $x_1,x_2$:
\begin{align*}
J(\pi_{r_2})  &= \frac{1}{2}\prn*{\bbE_{a\sim\pi_{r_2}}[ r_1(x_1,a) ] + \bbE_{a\sim\pi_{r_2}}[ r_1(x_2,a) ] } \rdef \frac{1}{2}\prn*{ J_1(\beta) + J_2(\beta) }\\
J_1(\beta) & = \frac{r_1(x_1,a_0)\piref(a_0 \mid x_1) \exp( r_2(x_1,a_0)/\beta)}{Z(r_2,x_1)} = \frac{\zeta /2 \exp(\zeta /\beta)}{1/2 \exp(\zeta/\beta) + 1/2}\\
J_2(\beta) & =\frac{r_{1}(x_2, a_0) \piref(a_0\mid x_2) \exp( r_2(x_2, a_0)/\beta) + r_{1}(x_1, a_{1})\piref(a_{1}\mid x_2) \exp(r_2(x_2, a_{1})/\beta))}{Z(r_2,x_2)}\\
  & ~~~~~~~~~~ = \frac{1/4 e^{1/2\beta} + 1/(2n)}{ 1/2 e^{1/2\beta} + e^{1/\beta}/(2n) + (n-1)/(2n)},
\end{align*}
where $Z(r_2,x)\ldef{}\sum_{a\in\cA}\piref(a\mid x)\exp(r_2(x,a)/\beta)$. 

We first consider the small $\beta$ regime. Here we use the upper
bound $J_1(\beta) \leq \zeta$ and focus on $J_2(\beta)$. Note that
$J_2(\beta)$ is increasing with $\beta$ for $\beta \leq 1/(2\log(n))$. In particular, if we
consider $\beta = 1/(c\log(n))$ for $c\geq 2$, then the expression
above is equal to
\begin{align*}
  J_2(\beta) = \frac{n^{c/2}/4 + 1/(2n)}{n^{c/2}/2 + n^{c-1}/2 + (n-1)/(2n)} \leq
  \frac{n^{c/2}/4 +  1/(2n)}{n^{c/2} + (n-1)/(2n)} \leq 1/4 +
  \frac{1}{2n^{c/2+1}} \leq 3/8,
\end{align*}
where the last inequality holds when $c\geq{}2$ and $n \geq{}2$.  We
set $c=2$, so that as long as $n\geq{}2$, $J(\pi_{r_2})
\leq\frac{3}{8}$. Thus, the suboptimality is
\begin{align*}
J(\pitil)-J(\pi_{r_2}) \geq{} \frac{\zeta}{2} + \frac{1}{4} - \prn*{\frac{\zeta}{2} + \frac{3}{16}} \geq \frac{1}{16}\rdef{} c_2.
\end{align*}

Next consider the regime where $\beta \geq 1/(2\log(n))$. 
Analogously to before, note that $J_2(\beta) \leq 1/2$. On the other hand, $J_1(\beta)$ is monotonically decreasing with $\beta$, so using $\beta \geq 1/(2\log(n))$ we obtain the bound
\begin{align*}
  J_1(\beta) \leq \frac{\zeta\exp(2\zeta\log(n))}{\exp(2\zeta\log(n)) + 1 } = \zeta \cdot \frac{n^{2\zeta}}{n^{2\zeta} + 1}.
\end{align*}
So in this case, the suboptimality is
\begin{align*}
  J(\pitil)-J(\pi_{r_2}) \geq{} \frac{\zeta}{2}\cdot\prn*{ 1 - \frac{n^{2\zeta}}{n^{2\zeta}+1}} \geq \frac{\zeta}{4} \cdot \frac{1}{n^{2\zeta}} = \frac{\log(2)}{16\log(n)},
\end{align*}
if we set $\zeta = \log(2)/(2\log(n))$ which is in $[0,1]$ under the
assumption that $n \geq 2$. 
\end{proof}

\section{Sample Complexity Guarantees for \chisb-RLHF}
\label{sec:rlhf}

The \chis-regularization framework we consider (\cref{sec:framework}) can be used to derive
algorithms beyond just \algshort, and we expect it to find broader use. 
To highlight this, in this section we analyze the algorithm that
directly optimizes a variant of the \chis-regularized RLHF objective
in \cref{eq:chis_reward}; this can be accomplished via policy optimization methods
such as PPO, in the vein of classical RLHF approaches to offline alignment
\citep{christiano2017deep,bai2022training,ouyang2022training,vonwerra2022trl}. As we
will show, a benefit of directly optimizing the RLHF objective is that
it allows us to provide guarantees that avoid dependence
on the $\Vmax$ parameter in \cref{thm:main},  
which may lead to improvement when $\Pi$ includes policies with very
large or very small density ratios $\frac{\pi}{\piref}$.

\paragraph{Algorithm}
Our algorithm, \rlhfalg is displayed in \cref{alg:rlhf}. At the
population level, the algorithm aims to optimize a variant of
\cref{eq:rlhf} that incorporates a small but important modification
that allows us to avoid dependencies on $\frac{\pi}{\piref}$. 
Given \emph{smoothing parameter} $\eta > 0$, define the \emph{smoothed
  \chis-divergence} $\Dsmthchis{\pi}{\piref} \ldef{} \En_\pi\brk[\Big]{\frac{\pi(a\mid{}x)}{\piref(a\mid{}x) + \eta\pi(a\mid{}x)}}$. We aim to find 
\begin{align}
\label{eq:rlhf-smooth}
  \argmax_{\pi} \Jsmth(\pi) 
  \ldef&~
  \En_{\pi} \left[ \rstar(x,a)
  \right]-\beta\Dsmthchis{\pi}{\piref}
  \\
  =&~ \argmax_{\pi}\E_{\pi} \left[ \rstar(x,a) -\beta\frac{\pi(a\mid{}x)}{\piref(a\mid{}x) + \eta\pi(a\mid{}x)}\right].
  \nonumber  
\end{align} 
The smoothing parameter $\eta$ effectively clips the policy ratio in
$\Dsmthchis{\pi}{\piref}$ where $\piref(a|x) \ll \eta\pi(a|x)$;
$\Dchis{\cdot}{\cdot}$ corresponds to the special (non-clipped) case
where $\eta = 0$. In particular, clipping ensures a uniform bound of
the form $\Dsmthchis{\pi}{\piref} \le \eta^{-1}$, whereas the best
bound we can hope for with the unclipped \chis-divergence is
$\Dchis{\pi}{\piref} = \En_\pi\brk[\Big]{\frac{\pi(a|x)}{\piref(a|x)}}
\le \Cinf$. For this reason, smoothing will allow us to obtain guarantees that
avoid dependence on all-policy concentrability or parameters similar
to $\Vmax$.

\begin{algorithm}[tp]
\caption{\rlhfalg}
\label{alg:rlhf}
\begin{adjustbox}{max width=\textwidth}
\begin{minipage}{\linewidth}
\begin{algorithmic}[1]
  \Statex[0] \multiline{{\bfseries input:}
    Reference policy $\piref$, preference dataset $\cDpref$, unlabeled
    context
    dataset $\cDnopref$,
    \chis-regularization coefficient $\beta>0$,
    smoothing parameter $\eta \ge  0$.
    }
  \vspace{3pt}
    \State \textbf{Estimate reward model via maximum likelihood: }
    \begin{align}
      \label{eq:reward-estimation}
        \rhat \leftarrow \argmax_{r \in \Rcal}
      \sum_{(x,\ap,\am) \in
        \cDpref}\log\left[\sigma\left(
        r(x,\ap) - r(x,\am) \right) \right].
      \end{align}
    \label{line:reward-estimation}
    \State Define \chis-regularized RLHF objective:
    \[
      \Jhatsmth(\pi) \ldef{} \frac{1}{n_\nopref}\sum_{x \in
            \cDnopref}\prn*{\En_{a\sim\pi(\cdot|x)}[\wh r(x,a)] - \beta\sum_a \frac{\pi^2(a|x)}{\piref(a|x) + \eta\pi(a|x)}}.
          \]
     \State \textbf{Policy optimization:} Compute $\pihat \in \Pi$ such that 
    \begin{align*}
      \Jhatsmth(\pihat) \ge \max_{\pi\in\Pi}\Jhatsmth(\pi) - \vepsopt.
    \end{align*}
    \vspace{-15pt}
    \label{line:chis-rlhf}
    \State \textbf{return:}
    $\pihat$.
\end{algorithmic}
\end{minipage}
\end{adjustbox}
\end{algorithm}

To optimize \cref{eq:rlhf-smooth}, \cref{alg:rlhf} takes two datasets
as input, along with a user-specified reward model class $\cR$
  and policy class $\Pi$. The first dataset, $\cDpref$, is labeled with human preferences, and is used to learn a reward model $\rhat$ via maximum likelihood estimation in \cref{line:reward-estimation}.
The second, $\cDnopref$, contains \emph{only unlabeled contexts} sampled from $\rho$, 
and is utilized in \cref{line:chis-rlhf} to learn a policy that
approximately maximizes an empirical version of
\cref{eq:rlhf-smooth}. Importantly, because \cref{line:chis-rlhf}
involves an empirical expectation over only contexts, it is a purely
computational problem that we can solve using algorithms like PPO; we
allow for tolerance $\vepsopt$ in \cref{line:chis-rlhf} to accommodate
optimization error from such algorithms.
By using unlabeled contexts in \cref{line:chis-rlhf}, we can obtain
tighter guarantees when $\cDnopref$ is large. This is often the case
in practice, where unlabeled contexts are cheap to obtain, but
preferences can be expensive to query.

\paragraph{Theoretical guarantees}
To analyze \rlhfalg, we make similar assumptions to those utilized in \cref{thm:main} for \algshort. 
Since \rlhfalg utilizes separate reward and policy classes, we require realizability conditions for both. 
Namely, $\Rcal$ must be able to express the true reward function
$\rstar$, and $\Pi$ must include the optimal policy for the
regularized RLHF objective in \cref{eq:rlhf-smooth}.

\begin{assumption}
  \label{ass:reward-realizability} 
  The reward function class satisfies $\rstar \in \Rcal$, and is bounded so that $r(x,a) \in [0, \Rmax]$ for all $r \in \Rcal$ and $(x,a) \in \Xcal\times\Acal$.   
\end{assumption}

\begin{assumption}
\label{ass:rlhf-policy-realizability} 
The policy class $\Pi$ satisfies $\pistarsmth \in \Pi$, where $\pistarsmth$ is the optimal policy for \cref{eq:rlhf-smooth}.     
\end{assumption}

Below is our main sample complexity guarantee for \rlhfalg. While it
is stated for a fixed, $\beta$-dependent smoothing parameter for
compactness, the general version of this result (\cref{thm:rlhf-sample}) allows for general $\eta$. 
\begin{theorem}
  \label{thm:rlhf}
  Let $\beta>0$ be given, and suppose
  \cref{ass:reward-realizability,ass:rlhf-policy-realizability} hold
  any $\eta \in \brk*{0,\frac{\beta}{8\Rmax}}$. With probability at
least $1-\delta$, \rlhfalg (\cref{alg:rlhf}) produces a policy $\wh\pi$ such
that for all policies $\pistar$ simultaneously, we have
\begin{align*}
  &J(\pistar) - J(\pihat) \\
&  \approxleq 
  \Rmax e^{2\Rmax}\cdot\sqrt{\frac{\Cone[\pistar]\log(|\cR|/\delta)}{n}} 
  + \beta\cdot\Cone[\pistar] 
  + \beta^{-1}\cdot\frac{\Rmax^2 e^{4\Rmax}\log(|\cR|/\delta)}{n} 
  + \Rmax \sqrt{\frac{\log(|\Pi|/\delta)}{n_\nopref}}
  + \vepsopt.  
\end{align*} 
In particular, given any comparator policy $\pistar$, we can choose
the regularization parameter $\beta$ to achieve 
\begin{align}
  \label{eq:rlhf_tuned}
  J(\pistar) - J(\pihat) \approxleq \Rmax e^{2\Rmax}\cdot\sqrt{\frac{\Cone[\pistar]\log(|\Rcal|/\delta)}{n}} + \Rmax\sqrt{\frac{\log(|\Pi|/\delta)}{n_\nopref}} + \vepsopt. 
\end{align}
\end{theorem}
Above, we see that \rlhfalg, like \algshort, has sample complexity
that scales only with the single-policy concentrability coefficient
$\Cone[\pistar]$, and holds for all comparator policies $\pistar$
simultaneously. Since the choice of $\beta$ induces a similar
bias-overoptimization tradeoff in the first statement of
\cref{thm:rlhf} as it did in \cref{thm:main} for \algshort, we focus
our discussion on the guarantee for a tuned choice of $\beta$
(\cref{eq:rlhf_tuned}). The first term in \cref{eq:rlhf_tuned} accounts for
the reward estimation error (\cref{line:reward-estimation}) and scales
with $\Cone[\pistar]$; as before, this accounts for how well rewards
estimated from $\piref$ transfer to other candidate policies. The
second term in \cref{eq:rlhf_tuned} accounts for the statistical error
from sampled contexts used in \cref{line:chis-rlhf} for policy
optimization.  In particular, it is possible to drive this term to be much smaller than the first by using a larger unlabeled context dataset, which is typically far cheaper to acquire.

\paragraph{Computationally efficiency}
\cref{thm:rlhf} bounds the sample complexity of \rlhfalg under the
assumption that we can solve \cref{line:chis-rlhf} up to
$\vepsopt$-accuracy. This is a purely computational problem, and in
practice it can be solved using policy gradient methods such
as PPO.

\paragraph{Comparison to \algshortb}
Unlike \algshort (\cref{thm:main}), \cref{thm:rlhf} has no dependence
on the parameter $\Vmax$ or quantities such as $\frac{\pi}{\piref} \le
\max_\pi \Cinf$. We primarily attribute this to the fact that \rlhfalg uses an explicit reward function class $\Rcal$, and normalizing or clipping it to the reward range $\Rmax$ is both natural and routinely done in practice \citep{shah2015estimation,christiano2017deep, ouyang2022training}.
In comparison, the implicit reward models induced by the policy class
$\Pi$ in \algshort can have larger range, and clipping the policy
class in \algshort directly, e.g., so that $|\beta\link(\frac{\pi}{\piref})|$ is bounded, is misguided, because the policy class may lose realizability (\cref{ass:realizability}). 
This is because $\rstar(x,a) = \beta\link\prn[\Big]{\frac{\pistarb(a|x)}{\piref(a|x)}} + \Zr[\rstar](x)$, and the normalization factor $\Zr[\rstar]$ cannot be reasonably accounted for when clipping $\Pi$.
While the $\Vmax$ (\cref{ass:vmax}) parameter involves pairs of action
probabilities, and thereby sidesteps the normalization constant issue,
it may not always be practical to modify $\Pi$ so that $\Vmax$ is
bounded, since this would require checking all pairs of each policy's
action probabilities.

However, using an explicit reward function class alone is not
enough. As discussed previously, when we move from implicit to
explicit \chis-regularization, incorporating the smoothing parameter $\eta$ in \cref{eq:rlhf-smooth}
is essential to avoid statistical errors due to policies with
large density ratios when we approximate the \chis-regularizer with
empirical data. A careful
choice of $\eta = \beta/\Rmax$ in \cref{thm:rlhf} balances the
benefits of clipping against the bias it introduces. Without smoothing
(i.e., $\eta = 0$), a guarantee that depends on $\max_\pi \Cinf[\pi]$
for \rlhfalg would be unavoidable, since the sample complexity must scale with the range
of the problem, which grows with the magnitude of the regularizer.
See \cref{cor:rlhf-nosmooth} in \cref{sec:proofs_rlhf} for a guarantee in the case where
$\eta=0$, which highlights this.

\arxiv{
  \section{Experiment details}
  \label{sec:experiment-details}
}

\paragraph{Dataset and models} For training, we use \href{https://huggingface.co/datasets/trl-internal-testing/tldr-preference-trl-style}{\texttt{trl-internal-testing/tldr-preference-trl-style}}, with 92.9K train samples and 83.8K validation samples. The reference policy $\piref$ is the Pythia-1b model \citep{biderman2023pythia} pre-trained on SFT data (\href{https://huggingface.co/cleanrl/EleutherAI_pythia-1b-deduped__sft__tldr}{\texttt{cleanrl/EleutherAI\_pythia-1b-deduped\_\_sft\_\_tldr}} from \cite{huang2022cleanrl}), and performance is measured via winrate against a baseline, as judged by GPT-4o. All parameters that are not algorithm-specific, such as the learning rate, are shared by both \algshort and \dpo in order to ensure a fair comparison.

\paragraph{Training details} Our implementation of \algshort is built upon the \dpo trainer from Transformer Reinforcement Learning (TRL) \citep{vonwerra2022trl}. \algshort comes with strong robustness and theoretical properties, but the policy ratios can sometimes introduce instability in training. In practice, we have observed that better stability and performance can be achieved by utilizing the (more general form) link function $\wt\phi(z) \ldef{} \exp\prn*{\clip_{[-88,20]}\prn*{\alpha \cdot \log z}} + \gamma \cdot \log z$ in \cref{alg:main}, and performing a small grid search over additional parameters $\alpha = \{\frac{1}{4}, 1\}$ and $\gamma = \{0.1, 1\}$ for a fixed $\beta$. 

We briefly discuss each parameter in turn. The mixing parameter $\gamma$ controls the relative ratios of KL- and \chis-regularization, our analysis in \cref{sec:main-mix} shows that \cref{thm:main} holds more generally for $\gamma \in (0, 1]$ (see \cref{thm:main-mix}). Next, ignoring clipping, $\alpha \in (0, 1]$ in $\wt\phi$ implements regularization with the $(1+\alpha)$-divergence (or Renyi divergence), which is an $f$-divergence that is stronger than KL-regularization but weaker than \chis-regularization \citep{van2014renyi}, and also carries single-policy concentrability guarantees (although with a slower-rate dependence on sample size $n$). 
For example, $\alpha = \frac{1}{4}$ corresponds to the link function 
$\phi(z) = \prn{z}^{1/4} + \gamma\log z$, 
which is easier to optimize than the link function 
$\phi(z) = z + \gamma \log z$ (corresponding to $\alpha = 1$) 
induced by \chis-regularization, 
given the potentially large magnitude of $z = \frac{\pi}{\piref}$. 
Though we do not write out the analysis here, the methods used to prove 
the sample complexity of \algshort (\cref{thm:main})
can be used to prove analogous guarantees for regularization with $\alpha$-divergences,
which will have slightly worse statistical rates. 

Lastly, we provide some additional explanation for the clipping operation. We observed that \texttt{torch.exp} is prone to underflow when $\log\frac{\pi}{\piref}$ is very negative, and clipping the upper range to 20 can help reduce numerical instabilities. Clipping in such a manner is supported by our analysis in \arxiv{\cref{prop:conc_bounds}}, which shows that $\frac{\pistar}{\piref} \le 1 + \frac{\Rmax}{\beta}$ (though technically we do not know $\Rmax$). 
The parameters for all experiments are displayed in \cref{tab:parameters}. 

\begin{table}[h]
    \centering
    \caption{Parameter settings in TL;DR summarizion}
    \begin{tabular}{@{}ll@{}}
        \toprule
        \textbf{Algorithm} & \textbf{Parameters} \\ \midrule
        \dpo 
            & batch size: 64 \\
            & learning rate: 1e-6 \\
            & scheduler: cosine \\
            & optimizer: adamw \\
        \midrule
        \algshort 
            & batch size: 64 \\
            & clip range: [-88, 20] \\
            & learning rate: 1e-6 \\
            & scheduler: cosine \\
            & optimizer: adamw \\
        $\beta=0.05$, 1 epoch & $\alpha: 1.25, \gamma: 1.0$ \\
        $\beta=0.05$, 2 epochs & $\alpha: 2.00, \gamma: 1.0$ \\
        $\beta=0.05$, 4 epochs & $\alpha: 1.25, \gamma: 0.1$ \\
        $\beta=0.005$, all epochs & $\alpha: 1.25, \gamma: 0.1$ \\
        \bottomrule
    \end{tabular}
    \label{tab:parameters}
\end{table}

\paragraph{Generation details} For winrate evaluation, we use greedy, temperature 0, decoding. For computation of the KL divergence, we sample from the model with temperature 1. The maximum prompt length is 512, and the maximum response length is 200. We use the standard generation prompt ``TL;DR:'' \citep{gao2024rebel}.

\paragraph{Evaluation of performance} The performance of each algorithm is measured via winrate against responses in the SFT dataset, as measured by GPT-4o (global standard). The winrate is computed on a subset of 512 prompts from the SFT validation set  (\href{https://huggingface.co/datasets/trl-internal-testing/tldr-preference-sft-trl-style}{\texttt{trl-internal-testing/tldr-preference-sft-trl-style}}), and the order of the model and reference responses are randomized each round.

\arxiv{

\section{Applying \algshortb to the Token-Level MDP}
\label{sec:token}

  We formalize the offline alignment problem as a (preference-based) contextual bandit
  problem. Other works \citep{rafailov2024r,xie2024exploratory}
  instead adopt a \emph{token-level MDP} formulation for alignment. In the
token-level MDP with horizon $H$, the
initial state $s_1\sim\rho$
represents a prompt, each action $a_h$ represents a token (with $\cA$
representing the vocabulary), and the state
$s_h=(s_1,a_1,\ldots,a_{h-1})$ is the prompt and sequence of tokens so
far. The language model policy $\pi$ maps
the current state $s_h=(s_1,a_1,\ldots,a_{h-1})$ to a distribution
over the next token $a_h\sim\pi(s_h)$, and the final trajectory
$\tau=(s_1,a_1),\ldots,(s_H,a_H)$ produced by this process represents the language model's response to the prompt $s_1$.

To apply \algshort to the token-level MDP, we assume access to a dataset
of labeled responses $\cD_\pref = \crl*{(s_1,\taup,\taum)}$ which is
labeled according to the Bradley-Terry model
\begin{align}
\label{eq:bt_traj}
\bbP(\tau\psdgt\tautil\mid{}s_1) = \frac{\exp\prn*{r(\tau\mid{}s_1)}}{\exp\prn*{r(\tau\mid{}s_1)} + \exp\prn*{r(\wt{\tau}\mid{}s_1)}}
\end{align}
for an unknown trajectory-level reward function
$r(\tau\mid{}s_1)$. Defining
$\pi(\tau\mid{}s_1)=\prod_{h=1}^{H}\pi(a_h\mid{}s_h)$, the \algshort
objective takes the form
    \begin{align}
      \label{eq:chi_dpo_token}
        \pihat \leftarrow \argmax_{\pi \in \Pi}
        \sum_{(x,\taup,\taum) \in
        \Dcal_\pref}\log\left[\sigma\left(
        \clip_{2\Rmax}\brk*{\beta\link\prn*{\frac{\pi(\taup\mid s_1)}{\piref(\taup\mid s_1)}} -
        \beta\link\prn*{\frac{\pi(\taum\mid s_1)}{\piref(\taum\mid s_1)}}} \right) \right],
    \end{align}
    which can be derived by 
    reparameterizing the objective
    \[
      \Jmix(\pi) =
      \En_{s_1\sim\rho,\tau\sim\pi\mid{}s_1}\brk*{r(\tau\mid{}s_1)}
      -\beta\cdot{}\Dchis{\pi}{\piref}-\beta\cdot\Dkl{\pi}{\piref},
    \]
    where
    $\Dchis{\pi}{\piref}=\frac{1}{2}\En_{s_1\sim\rho,\tau\sim\piref\mid{}s_1}\brk[\big]{\prn[\big]{\frac{\pi(\tau\mid{}s_1)}{\piref(\tau\mid{}s_1)}-1}^2}$
    and
    $\Dkl{\pi}{\piref}=\En_{s_1\sim\rho,\tau\sim\pi\mid{}s_1}\brk[\big]{\log\frac{\pi(\tau\mid{}s_1)}{\piref(\tau\mid{}s_1)}}$
    are the trajectory-level \chis- and KL-divergence.

From a statistical perspective, the token-level MDP formulation is
identical to the contextual bandit formulation, treating the
trajectory $\tau$ as a composite action, and \cref{eq:chi_dpo_token}
coincides with \cref{eq:chi_dpo} under this
interpretation. Consequently, \cref{thm:main} applies as-is to the
token-level \algshort objective in \cref{eq:chi_dpo_token}. In
particular, as long as $\pistarb\in\Pi$, where $\pistarb$ is the policy that
satisfies
\[
r(\tau\mid{}s_1) = \beta\link\prn*{\frac{\pistarb(\tau\mid{}s_1)}{\piref(\tau\mid{}s_1)}} + \Zklr[r](s_1),
\]
token-level \algshort ensures that with probability at least $1-\delta$, for all
$\pistar\in\Pi$,
    \begin{align}
    J(\pistar) - J(\pihat) \approxleq \Vmax e^{2\Rmax}\cdot\sqrt{\frac{\Cone[\pistar]\log(|\Pi|/\delta)}{n}} + \beta\cdot\Cone[\pistar] + \beta^{-1}\cdot\frac{\Vmax^2 e^{4\Rmax}\log(|\Pi|/\delta)}{n},
    \end{align}
    where $\Cone[\pi]\ldef{}\En_{s_1\sim\rho,\tau\sim{}\pi\mid{}s_1}\brk*{\frac{\pi(\tau\mid{}s_1)}{\piref(\tau\mid{}s_1)}}$.

  }

\part{Proofs}

\section{Preliminaries}
\label{sec:technical}

Recall that for a pair of probability measures $\bbP$ and $\bbQ$ with
a common dominating measure $\omega$, Hellinger distance is defined
via
\begin{align}
\Dhels{\bbP}{\bbQ}=\int\prn*{\sqrt{\frac{\mathrm{d}\bbP}{\mathrm{d}\omega}}-\sqrt{\frac{\mathrm{d}\bbQ}{\mathrm{d}\omega}}}^2\mathrm{d}\omega.
\end{align}

\begin{lemma}[MLE for conditional density estimation (e.g., \citet{wong1995probability,Sara00,zhang2006from,agarwal2020flambe})]
  \label{lem:mle}
  Consider a conditional density $p^\star : \cX \rightarrow \Delta(\cY)$, where $\cX$ is the instance space and $\cY$ is the target space. 
  Let $\cD = \{(x\ind{i},y\ind{i})\}_{i=1}^n$ be a dataset in which $(x\ind{i},y\ind{i})$ are drawn i.i.d. as $x\ind{i} \sim \rho \in \Delta(\cX)$ and $y\ind{i} \sim p^\star(y\mid{}x)$. Suppose we have a finite function class $\cP$ such that $p^\star \in \cP$, where $p(\cdot\mid{}x) \in \Delta(\cY)$ for all $p \in \cP$ and $x \in \cX$. Define the maximum likelihood estimator
  \begin{align*}
    \wh p \ldef \argmax_{p \in \cP} \sum_{(x,y) \in \cD} \log p(y\mid{}x).
  \end{align*}
  Then with probability at least $1-\delta$, 
  \begin{align*}
    \En_{x \sim \rho}\brk*{\Dhels{\wh p(\cdot\mid{}x)}{p^\star(\cdot\mid{}x)}} \le \frac{2\log(|\cP|\delta^{-1})}{n}.
  \end{align*}    
\end{lemma}

\section{Proofs for \creftitle{sec:main}}
\label{sec:proofs_main}

This section is organized as follows. 
First, in \cref{sec:main-mix}, 
we analyze a more general version of \algshort 
that mixes KL-regularization with \chis-regularization using a mixing
parameter $\gamma \in (0,1]$, 
and present its sample complexity guarantee in \cref{thm:main-mix}.
\algshort is a special case with $\gamma = 1$, 
and \cref{sec:proof-main} shows (with a one-line proof) 
that \cref{thm:main} follows directly from \cref{thm:main-mix} 
with this parameter choice. 

\subsection{General Version of
  \creftitle{thm:main}}\label{sec:main-mix}

As previously described at the end of \cref{sec:theoretical}, 
\algshort can be applied in a more general form where the
KL-regularization is mixed with \chis-regularization using a weight
parameter $\gamma \in (0,1]$. In this section, we analyze the sample complexity for this form of the algorithm, of which \algshort is a special case with $\gamma = 1$, which directly leads to the guarantee in \cref{thm:main}.   

Concretely, given regularization parameter $\beta > 0$ and weight
parameter $\gamma \in (0,1]$, we aim to solve the mixed \chis-regularized objective 
\begin{align}
\label{eq:rlhf-mix}
  \argmax_{\pi : \cX\rightarrow\Delta(\cA)} 
  \Jmixg(\pi) 
  \ldef{} 
  \En_{\pi}\brk*{\rstar(x,a)} 
  - \beta\cdot\Dchis{\pi}{\piref}
  - \beta\gamma\cdot\Dkl{\pi}{\piref}.
\end{align}
The regularization term $\Dchis{\pi}{\piref}+\gamma\cdot\Dkl{\pi}{\piref} = \Dfmixg{\pi}{\piref}$ is an $f$-divergence induced by the function $\fmixg(z) \ldef{} \frac{1}{2}(z-1)^2 + \gamma z\log{}z$. 
Correspondingly, we replace the link function $\phi(\cdot)$ in \algshort with 
\[
  \linkg(z) \ldef{} z + \gamma\log(z), 
\]
and output the policy 
\begin{align}
\label{eq:dpo-mix}
    \pihat \leftarrow \argmax_{\pi \in \Pi}
    \sum_{(x,\ap,\am) \in
    \Dcal_\pref}\log\left[\sigma\left(
    \clip_{2\Rmax}\brk*{\beta\linkmix\prn*{\frac{\pi(\ap\mid x)}{\piref(\ap\mid x)}} -
    \beta\linkmix\prn*{\frac{\pi(\am\mid x)}{\piref(\am\mid x)}}} \right) \right].
\end{align}

To give a sample complexity guarantee for \cref{eq:dpo-mix}, we
require that $\Pi$ can express the optimal regularized policy for the
objective $\Jmixg$ in \cref{eq:rlhf-mix}. This generalizes
\cref{ass:realizability} for \algshort, which corresponds to the
special case where $\gamma=1$.
\begin{assumption}[Policy realizability]
  \label{ass:realizability-mix}
  The policy class $\Pi$ satisfies $\pistarg\in\Pi$, where $\pistarg$
  is the optimal policy under mixed \chis-regularization (\cref{eq:chis_opt}).
\end{assumption}

We also assert that, analogous to \cref{ass:vmax}, the ``implicit'' reward models
induced by the policy class $\Pi$ and the link function $\linkg$ have bounded range. 
\begin{assumption}[Bounded implicit rewards]
\label{ass:vmax-mix}
For a parameter $\Vmax\geq{}\Rmax$, it holds that for all $\pi\in\Pi$, $x\in\cX$, and $a,b\in\cA$,
\begin{equation}
\abs*{\beta\linkg\prn*{\frac{\pi(a\mid{}x)}{\piref(a\mid{}x)}}-\beta\linkg\prn*{\frac{\pi(b\mid{}x)}{\piref(b\mid{}x)}}} \leq \Vmax.
\end{equation}
\end{assumption}

We now state the sample complexity guarantee for the policy learned in
\cref{eq:dpo-mix}. The first bound applies to general $\beta > 0$ and
$\gamma \in (0, 1]$, while in the second we obtain a tight statistical
rate by choosing the parameter $\beta$ as a function of the comparator policy $\pistar$.

\begin{theorem}[General version of \cref{thm:main}]
\label{thm:main-mix}
  Suppose \cref{ass:realizability-mix,ass:vmax-mix} hold for some $\beta > 0$ and $\gamma \in (0, 1]$. 
  With probability at least $1-\delta$, 
  the variant of \algshort in \cref{eq:dpo-mix} produces a policy $\pihat$ such that for all policies $\pistar$ simultaneously, we have 
  \begin{align*}
    J(\pistar) - J(\pihat) 
    \leq 
    32\Vmax e^{2\Rmax}\cdot\sqrt{\frac{2\Cone[\pistar]\log(|\Pi|/\delta)}{n}} 
    +\beta(1+\gamma)\cdot\frac{\Cone[\pistar]}{2}
    + \beta^{-1} \cdot \frac{256\Vmax^2 e^{4\Rmax}\log(|\Pi|/\delta)}{n} . 
  \end{align*} 
  In particular, given any comparator policy $\pistar$, we can choose $\beta= 32\Vmax e^{2\Rmax} \sqrt{\frac{2\log(|\Pi|/\delta)}{n\Cone[\pistar]}}$ to achieve 
  \begin{align*}
    J(\pistar) - J(\wh\pi) 
    \leq 
    \prn*{64+4\gamma}\Vmax e^{2\Rmax} \cdot \sqrt{\frac{\Cone[\pistar]\log(|\Pi|/\delta)}{n}}. 
  \end{align*}
\end{theorem}

The bias-overoptimization tradeoffs induced by the choice of $\beta$ in \cref{thm:main-mix} 
are identical to those for \cref{thm:main} (and described there). Let
us briefly discuss the influence of $\gamma$ on the sample
complexity. We first observe that choice of $\gamma \in (0, 1]$
changes the bound by only a small multiplicative factor, which implies that $\gamma$ can be arbitrarily small as long as it is positive. 
For the analysis, this is natural because the KL-divergence is dominated by the \chis-divergence, 
and, as discussed in \cref{sec:algorithm}, KL-regularization is only
needed to enable the DPO-style reparameterization trick for
\cref{eq:dpo-mix} (in particular, the \chis-RLHF algorithm in
\cref{sec:rlhf}, which does not require reparameterization, obtains similar guarantees using pure
\chis-regularization). It is worth noting, however, 
that the $\gamma$ parameter can implicitly influence the magnitude of
$\Vmax$, as well as the policy realizability condition. As such,
practical consequences of this hyperparameter choice may not be fully captured by \cref{thm:main-mix}.  

\begin{proof}[\pfref{thm:main-mix}]
  Recall that the link function $\linkg$ induces a correspondence
between policies in the class $\Pi$ and the implicit reward functions
they induce (or, equivalently, 
between policies and the Bradley-Terry preference models they express).
Our proof centers around the implicit reward model 
induced by the learned policy $\pihat$, 
\begin{align*}
  \rhat(x,a) \ldef{} \beta\cdot\linkg\prn*{\frac{\pihat(a\mid{}x)}{\piref(a\mid{}x)}}, 
\end{align*} 
which will allow us to move between the \algshort objective (\cref{eq:dpo-mix}) 
and the RLHF objective (\cref{eq:rlhf-mix}). 
In particular, we establish two key facts,
which together show 
that \cref{eq:dpo-mix} implicitly solves \cref{eq:rlhf-mix}:  
\begin{enumerate}
  \item 
  (\cref{lem:clip-dpo-estimation})
  The reward model $\rhat$ is an accurate estimate of $\rstar$ on the distribution of $\piref$.
  Moreover, we can transfer this guarantee to the distribution of any
  policy $\pi$ by paying a multiplicative $(1+2\Dchis{\pi}{\piref})$-factor.
  \item 
  (\cref{lem:general-reward-to-policy})
  $\pihat$ maximizes the RLHF objective in \cref{eq:rlhf-mix} with reward model $\rhat$, namely,  
  \begin{align}
  \label{eq:informal-rhat-rlhf-mix}
    \pihat 
    = \argmax_{\pi\in\Pi} 
    \En_\pi\brk{\rhat(x,a)}
    - \beta\cdot\Dchis{\pi}{\piref}
    - \beta\gamma\cdot\Dkl{\pi}{\piref}.
  \end{align}   
\end{enumerate}
Establishing these relationships enables us 
to analyze the \algshort policy $\pihat$ defined in \cref{eq:dpo-mix} through the RLHF formulation in \cref{eq:informal-rhat-rlhf-mix}, 
allowing us to appeal to pessimism-based arguments to show that
\algshort is insensitive to overoptimization error that might
otherwise be encountered when learning a
policy from off-policy data.

\paragraph{Implicit reward model $\rhat$} 
The \algshort objective in \cref{eq:dpo-mix} is equivalent to maximum likelihood estimation with the Bradley-Terry preference model over the induced reward function class 
\begin{align*}
  \cR_\Pi \ldef \crl*{r(x,a) = \beta\cdot\linkg\prn*{\frac{\pi(a\mid{}x)}{\piref(a\mid{}x)}} : \pi\in\Pi}. 
\end{align*} 
Then, since $\pihat$ is the maximizer in \cref{eq:dpo-mix}, we can equivalently write 
\begin{align}
  \label{eq:rhat-mle}
  \rhat = \argmax_{r \in \cR_\Pi} \sum_{(x,\ap,\am)\in\cDpref}\log\sigma\prn*{\clip_{2\Rmax}\brk*{r(\ap\mid{}x)-r(\am\mid{}x)}}. 
\end{align} 

The following lemma, which builds on a standard MLE generalization
bound (\cref{lem:mle}) bounds the error of $\rhat$ under the action
distribution induced by $\piref$. Recall that we use
$\En_{\pi,\pi'}\brk{\cdot}$ as shorthand for
$\En_{x\sim\rho,a\sim\pi(\cdot\mid{}x),b\sim\pi'(\cdot\mid x)}\brk{\cdot}$.
\begin{lemma}
\label{lem:clip-dpo-reward}  
  Suppose \cref{ass:realizability-mix} holds. Then with probability at least $1-\delta$, the policy $\pihat$ output by \cref{eq:dpo-mix} satisfies 
  \begin{align*}
    \vepsstat^2 \rdef{} \En_{\piref,\piref}
    \brk*{\prn*{\clip_{2\Rmax}\brk*{\rhat(x,a) - \rhat(x,b)} 
    - \clip_{2\Rmax}\brk*{\rstar(x,a) - \rstar(x,b)}}^2} 
    \le 
    \frac{128\Rmax^2 e^{4\Rmax} \log(|\Pi|/\delta)}{n}.    
  \end{align*}
\end{lemma}
\cref{lem:clip-dpo-reward}, along with all further supporting lemmas,
is proven in the sequel. This result measures the error of $\rhat$ using the
clipped differences of rewards for pairs of actions $(x,a,b)$ drawn
from $\piref$. Clipping the range of the implicit/explicit reward
functions to $2\Rmax$ ensures that the statistical error does not
depend on $\Vmax$.  One minor but important detail in the proof is showing that \cref{ass:realizability-mix} implies $\cR_\Pi$ includes the true reward function $\rstar$ up to an action-independent shift, so that the true preference model is realizable.  

\paragraph{Implicit RLHF policy optimization}
Having established the accuracy of $\rhat$, we now show that \cref{eq:dpo-mix} finds the optimal policy to the RLHF objective in \cref{eq:informal-rhat-rlhf-mix} when $\rhat$ is used as the reward model, i.e.,  
\begin{align}
\label{eq:rhat-rlhf-mixg}
  \pihat 
  = \argmax_{\pi\in\Pi} 
  \Jmixgr[\rhat](\pi) 
  \ldef{}
  \En_\pi\brk{\rhat(x,a)}
  - \beta\cdot\Dchis{\pi}{\piref}
  - \beta\gamma\cdot\Dkl{\pi}{\piref}. 
\end{align} 
This is a direct consequence of the result in \cref{lem:general-reward-to-policy}, 
which shows that an analogous property holds for general $f$-divergences.
In particular, for any convex function $f$ and policy $\pi$,
the policy $\pi$ is itself the optimal solution 
to the $f$-divergence-regularized RLHF objective
under the implicit reward model induced by $\pi$ 
with the link function $f'$. 
\begin{lemma}
\label{lem:general-reward-to-policy}
  Let $f : (0,\infty) \rightarrow \RR$ be 
  a convex function with $f(1) = 0$. 
  Further, 
  $f$ is differentiable almost everywhere 
  and $0 \notin\dom(f')$, 
  where we define   
  $f'(0) \ldef{} \lim_{x\downarrow 0} \frac{f(x) - f(0)}{x}$
  and    
  $f(0) \ldef{} \lim_{x\downarrow 0} f(x)$. 
  Given any parameter $\beta > 0$
  and valid policy 
  $\bar\pi:\cX\rightarrow\Delta(\cA)$, with $\pi(a\mid{}x) \in \dom(f')$ for all $(x,a)$,
  let  
  $\bar{r}(x,a) = \beta f'\prn*{\frac{\bar\pi(a\mid{}x)}{\piref(a\mid{}x)}}$
  be the implicit reward model.
  Then %
  \begin{align*}
    \bar\pi 
    \in  
    \argmax_{\pi:\cX \rightarrow \Delta(\cA)}
    \En_\pi\brk*{\bar{r}(x,a)} 
    - \beta \Df{\pi}{\piref}. 
  \end{align*}   
\end{lemma}

Since $\fmixg' = \linkg = x + \gamma\log x$ for $\gamma > 0$, clearly $0 \not\in \dom(\linkg)$. 
Further, under \cref{ass:vmax-mix}, 
$\pi(a\mid{x}) > 0$ for all $\pi\in\Pi$ (otherwise $\Vmax$ would be undefined), 
thus $\pi(a\mid{x}) \in \dom(\linkg)$ for all $(x,a)$. 
The claim in \cref{eq:rhat-rlhf-mixg} then directly follows.

\paragraph{Estimation error translation}
To proceed, we will use condition on \cref{lem:clip-dpo-reward} and
use the event in this lemma to relate the
estimated RLHF objective in \cref{eq:informal-rhat-rlhf-mix} to the
``true'' RLHF objective that replaces $\rhat$ with $\rstar$. An
immediate challenge is that the RLHF objective in \cref{eq:informal-rhat-rlhf-mix} 
must evaluate $\En_\pi[\rhat(x,a)]$ for all $\pi\in\Pi$, and accuracy under $\piref$ does not immediately imply that $\rhat$ is accurate for other policies. 
The following bound quantifies the effects of this distribution shift using the \chis-divergence, 
and expresses how the estimation guarantee for $\rhat$ in
\cref{lem:clip-dpo-reward} transfers to other policies $\pi$ of interest.
\begin{lemma}
  \label{lem:clip-dpo-estimation}
  Suppose \cref{ass:realizability} holds. Then for any $\pi: \cX
  \rightarrow \Delta(\cA)$, under the event in
  \cref{lem:clip-dpo-reward}, we have
  \begin{align*}
     \En_{\pi,\piref}\brk*{\abs*{\rhat(x,a) - \rhat(x,b) 
     - \prn*{\rstar(x,a) - \rstar(x,b)}}} 
     \leq 
     \frac{2\Vmax}{\Rmax}\cdot 
     \sqrt{\prn*{1+2\Dchis{\pi}{\piref}}\cdot\vepsstat^2}, 
  \end{align*} 
  where $\vepsstat^2$ is the off-policy estimation error defined in \cref{lem:clip-dpo-reward}. 
\end{lemma}
It is worth noting that \cref{lem:clip-dpo-estimation} bounds the \emph{unclipped} on-policy estimation error (on the LHS) 
in terms of the \emph{clipped} off-policy estimation error, 
and in making this translation we pay for $\Vmax$.  
As we will see shortly, working with the unclipped $\rhat$ object is necessary for showing that \cref{eq:dpo-mix} implicitly optimizes \cref{eq:informal-rhat-rlhf-mix}. 

\paragraph{Pessimism-based regret decomposition}
Equipped with the preceding lemmas, we can now bound the regret for
\algshort. We decompose the regret using the RLHF objective
$\Jmixgr[\rhat](\pistar)$ defined in \cref{eq:rhat-rlhf-mixg}. Fixing
an arbitrary comparator policy $\pistar$, we have
 \begin{align*}
  J(\pistar) - J(\wh\pi) 
  =&~ \En_{\pistar}\brk{\rstar(x,a)} - \En_{\pihat}\brk{\rstar(x,a)}
  \\
  =&~ \En_{\pistar}\brk{\rstar(x,a)} - \Jmixgr[\rhat](\pistar) + \Jmixgr[\rhat](\pistar) - \En_{\pihat}\brk{\rstar(x,a)}
  \\
  \le&~ 
  \En_{\pistar}\brk{\rstar(x,a)} - \Jmixgr[\rhat](\pistar) 
  + \Jmixgr[\rhat](\pihat) - \En_{\pihat}\brk{\rstar(x,a)},
 \end{align*}
 where the last inequality uses the optimality of $\pihat$ for \cref{eq:rhat-rlhf-mixg}.

 Expanding the expression for $\Jmixgr[\rhat]$, we can further bound
 this by
\begin{align}
  J(\pistar) - J(\wh\pi)
  \le&~ 
  \En_{\pistar}\brk{\rstar(x,a) - \rhat(x,a)} 
  + \beta\Dchis{\pistar}{\piref}
  + \beta\gamma\Dkl{\pistar}{\piref}\notag
  \\
  &+ 
  \En_{\pihat}\brk{\rhat(x,a) - \rstar(x,a)} 
  - \beta\Dchis{\pihat}{\piref} 
  - \beta\gamma\Dkl{\pihat}{\piref}\notag
  \\
  \le&~ 
  \En_{\pistar}\brk{\rstar(x,a) - \rhat(x,a)} 
  + \beta(1+\gamma)\Dchis{\pistar}{\piref}\notag
  \\
  &+ \En_{\pihat}\brk{\rhat(x,a) - \rstar(x,a)} 
  - \beta\Dchis{\pihat}{\piref}.  \label{eq:step0}
\end{align}
In the last line, we use the fact that $0 \le \Dkl{\pi}{\piref} \le \Dchis{\pi}{\piref}$ for any policy $\pi$ to consolidate the $f$-divergence terms. Specifically, this allows us to eliminate $\Dkl{\pihat}{\piref}$, and combine $\Dkl{\pistar}{\piref}$ and $\Dchis{\pistar}{\piref}$. 

In order to bound the reward estimation error terms in \cref{eq:step0} using the guarantee we have previously established (\cref{lem:clip-dpo-estimation}), 
we first center them using the return under the reference policy:
\begin{align*}
  &\En_{\pistar}\brk{\rstar(x,a) - \rhat(x,a)} + \En_{\pihat}\brk{\rhat(x,a) - \rstar(x,a)}
  \\
  &\quad= 
  \En_{\pistar,\piref}\brk*{\rstar(x,a) - \rhat(x,a) - \rstar(x,b) + \rhat(x,b)} 
  + \En_{\pihat,\piref}\brk*{\rhat(x,a) - \rstar(x,a) - \rhat(x,b) + \rstar(x,b)}
  \\
  &\quad=
  \En_{\pistar,\piref}\brk*{\rstardiff(x,a,b) - \rhatdiff(x,a,b)} 
  + \En_{\pihat,\piref}\brk*{\rhatdiff(x,a,b) - \rstardiff(x,a,b)},
\end{align*}
where $\rstardiff(x,a,b)\ldef{}\rstar(x,a) - \rstar(x,b)$ and $\rhatdiff(x,a,b)\ldef{}\rhat(x,a) - \rhat(x,b)$. 
Substituting this identity back into the regret decomposition in
\cref{eq:step0}, we apply \cref{lem:clip-dpo-estimation} with
$\vepsstat^2 \ldef{} 128 \Rmax^2 e^{4\Rmax}
\frac{\log(|\Pi|/\delta)}{n}$ (from \cref{lem:clip-dpo-reward}) to
obtain 
\begin{align*}
  J(\pistar) - J(\wh\pi)
  \le&~ 
  \En_{\pistar,\piref}\brk*{\rstardiff(x,a,b) - \rhatdiff(x,a,b)} 
  + \beta(1+\gamma)\Dchis{\pistar}{\piref}
  \\
  &+ 
  \En_{\pihat,\piref}\brk*{\rhatdiff(x,a,b) - \rstardiff(x,a,b)}
  - \beta\Dchis{\pihat}{\piref}
  \\
  \le&~ 
  \frac{2\Vmax}{\Rmax}\sqrt{\prn*{1+2\Dchis{\pistar}{\piref}}\cdot\vepsstat^2} 
  + \beta(1+\gamma)\Dchis{\pistar}{\piref}
  \\
  &+ 
  \frac{2\Vmax}{\Rmax}\sqrt{\prn*{1+2\Dchis{\pihat}{\piref}}\cdot\vepsstat^2}
  - \beta\Dchis{\pihat}{\piref}
  \\
  =&~   
  \frac{2\Vmax}{\Rmax}\sqrt{\Cone[\pistar]\cdot\vepsstat^2} 
  + \frac{\beta(1+\gamma)}{2}\cdot\prn*{\Cone[\pistar] - 1}
  + \frac{2\Vmax}{\Rmax}\sqrt{\Cone[\pihat]\cdot\vepsstat^2}
  - \frac{\beta}{2}\cdot\prn*{\Cone[\pihat] - 1}
  \\
  \le&~   
  \frac{2\Vmax}{\Rmax}\sqrt{\Cone[\pistar]\cdot\vepsstat^2} 
  + \frac{\beta(1+\gamma)}{2}\cdot\Cone[\pistar]
  + \frac{2\Vmax}{\Rmax}\sqrt{\Cone[\pihat]\cdot\vepsstat^2}
  - \frac{\beta}{2}\cdot\Cone[\pihat], 
\end{align*}
since $\Cone[\pi] = 1+2\Dchis{\pi}{\piref}$, or equivalently $ \Dchis{\pi}{\piref} = \frac{1}{2}\prn{\Cone[\pi] - 1}$. 
Lastly, we use the AM-GM inequality to upper bound 
\[
  \frac{2\Vmax}{\Rmax}\sqrt{\Cone[\pihat]\cdot\vepsstat^2} 
  \le 
  \frac{2\Vmax^2\vepsstat^2}{\Rmax^2\beta} 
  + \frac{\beta\Cone[\pihat]}{2},
\]
allowing us to conclude that
\begin{align*}
  J(\pistar) - J(\wh\pi)
  \le&~
  \frac{2\Vmax}{\Rmax}\sqrt{\Cone[\pistar]\cdot\vepsstat^2} 
  + \frac{\beta(1+\gamma)}{2}\cdot\Cone[\pistar]
  + 2\beta^{-1}\cdot\frac{\Vmax^2\vepsstat^2}{\Rmax^2}. 
\end{align*} 
Plugging in the expression for $\vepsstat^2$ results in the first statement of \cref{thm:main-mix}. 

\paragraph{Choosing $\beta$ for tight rates}
For the second statement, given a comparator policy $\pistar$, choosing $\beta = \frac{2\Vmax}{\Rmax}\sqrt{\frac{\vepsstat^2}{\Cone[\pistar]}}$ gives  
\begin{align*}
  J(\pistar) - J(\wh\pi)
  \le&~
  \frac{2\Vmax}{\Rmax} \sqrt{\Cone[\pistar]\cdot\vepsstat^2} 
  + (1+\gamma)\frac{\Vmax}{\Rmax}\sqrt{\Cone[\pistar]\cdot\vepsstat^2}
  + \frac{\Vmax}{\Rmax} \sqrt{\Cone[\pistar]\cdot\vepsstat^2}
  \\
  =&~ \prn*{4+\gamma}\frac{\Vmax}{\Rmax}\sqrt{\Cone[\pistar]\cdot\vepsstat^2}. 
\end{align*}
\end{proof}

\subsubsection{Proofs for Supporting Lemmas}
\begin{proof}[\pfref{lem:clip-dpo-reward}]   
  Recall the reward-based MLE objective in \cref{eq:rhat-mle}, 
  \begin{align*}
    \rhat = \argmax_{r \in \cR_\Pi} \sum_{(x,\ap,\am)\in\cDpref}\log\sigma\prn*{\clip_{2\Rmax}\brk*{r(x,\ap)-r(x,\am)}}. 
  \end{align*}  
  To leverage standard generalization bounds for MLE, we re-interpret
  this objective as maximum likelihood over a class of preference distributions under the Bradley-Terry model.  
  For a reward function $r$, define for all $y \in \{+1, -1\}$ and $(x,a,b) \in \Xcal\times\Acal\times\Acal$ its induced preference distribution:
  \[
    P_r(y|x,a,b) 
    = 
    \indic\crl{y=+1}\cdot\sigma\prn*{\clip_{2\Rmax}\brk*{r(x,a) - r(x,b)}} 
    + 
\indic\crl{y=-1}\cdot\sigma\prn*{\clip_{2\Rmax}\brk*{r(x,b) - r(x,a)}}.
  \]  
  Consider the a class of preference models induced by $\cR_\Pi$ under this definition,  
  $
    \cP_\Pi 
    \ldef{}
    \crl*{P_{r} : r \in \cR_\Pi}. 
  $ 
  We can equivalently write that 
  \begin{align*}
    P_{\rhat} = \argmax_{p \in \cP_\Pi} \sum_{(x,\ap,\am)\in\cDpref}\log p(+1 \mid{} x,\ap,\am),
  \end{align*}
or, interpreting each tuple $(x,\ap,\am)$ in $\cDpref$ as being
induced by a tuple $(x,a,\atil,y)$ in which $(\ap,\am)=(a,\atil)$ if
$y=+1$ and $(\ap,\am)=(\atil,a)$ if $y=-1$,
  \begin{align*}
    P_{\rhat} = \argmax_{p \in \cP_\Pi} \sum_{(x,a,\atil,y)\in\cDpref}\log p(y \mid{} x,a,\atil).
  \end{align*}
  
  Next, we show that $P_{\rstar} \in \cP_\Pi$, ie., the induced preference model class realizes the true distribution. 
  For $\pistarg$, define the reward model 
  \[
    \wt{r}^\star(x,a) = \linkg\prn*{\frac{\pistarg(a\mid{}x)}{\piref(a\mid{}x)}},
  \]
  which is equivalent to $\rstar$ up to an action-independent shift, namely, the normalization factor $\lambdag$ in \cref{lem:rlhf-mix-solution}.  
  Since $\pistarg \in \Pi$ under \cref{ass:realizability-mix}, 
  we have $\wt{r}^\star \in \cR_\Pi$, and for all $(x,a,b)\in\cX\times\cA\times\cA$, it holds that 
  \begin{align*}
   \clip_{2\Rmax}\brk*{\wt{r}^\star(x,a)-\wt{r}^\star(x,b)}
    =&~ 
    \clip_{2\Rmax}\brk*{\rstar(x,a)-\rstar(x,b)}
    =
    \rstar(x,a)-\rstar(x,b).
  \end{align*} 
  The first equality is because action-independent shift between $\wt{r}^\star$ and $\rstar$ is cancelled out when taking the difference of rewards, 
  and the second equality is because, by assumption, $\rstar \in [0, \Rmax]$.
  As a result, the reward difference is bounded in the same range and never clipped. 
  
  From this we conclude that $P_{\wt{r}^\star} = P_{\rstar} \in \cP_\Pi$, and realizability is satisfied. 
  Further, it is easy to see that $\cP_\Pi$ contains only valid distributions. 
  Thus, having satisfied the necessary preconditions, 
  we can invoke \cref{lem:mle}, which guarantees that with probability at least $1-\delta$, we have  
  \begin{align*}
    \En_{\piref,\piref}\brk*{\Dhels{P_{\rhat}(\cdot\mid{}x,a,b)}{P_{\rstar}(\cdot\mid{}x,a,b)}}
    \le&~ 
    \frac{2\log(|\Pi|/\delta)}{n}. 
  \end{align*}
  To conclude, we extract a bound on reward estimation error from this
  Hellinger distance bound by using \cref{lem:lipschitz} with $R = V = 2\Rmax$, giving
  \begin{align*}
    &\En_{\piref,\piref}
    \brk*{\prn*{\clip_{2\Rmax}\brk*{\rhat(x,a) - \rhat(x,b)} 
    - \clip_{2\Rmax}\brk*{\rstar(x,a) - \rstar(x,b)}}^2}
    \\
    &\quad\le 
    64e^{4\Rmax} \Rmax^2 \cdot \En_{\piref,\piref}\brk*{\Dhels{P_{\rhat}(\cdot\mid{}x,a,b)}{P_{\rstar}(\cdot\mid{}x,a,b)}}
    \\
    &\quad\le 128e^{4\Rmax} \Rmax^2 \cdot \frac{\log(|\Pi|/\delta)}{n}.
  \end{align*}
\end{proof}

\begin{proof}[\pfref{lem:general-reward-to-policy}]
  
  First we rewrite the objective as a minimization problem, 
  \begin{alignat*}{2}
    \argmin_{\pi}\quad  
    &-\En_\pi\brk*{\bar r(x,a)}
    + \beta\Df{\pi}{\piref}
    \\
    \textrm{s.t.}\quad  
    &\rho(x)\sum_a \pi(a\mid{}x) = \rho(x) 
    &&\forall x,
    \\
    &\rho(x)\pi(a\mid{}x) \ge 0 
    &&\forall x,a. 
  \end{alignat*}
  Here, $\pi$ is the primal variable, 
  and denote the dual variables as $\lambda : \cX\rightarrow \RR$ and $\alpha:\cX\times\cA \rightarrow [0,\infty)$, which correspond to the first and second constraints, respectively. The Lagrangian form is then 
  \begin{align*}
    \cL(\pi,\lambda,\alpha)
    = 
    -\En_\pi\brk{\bar r(x,a)} 
    + \beta\Df{\pi}{\piref} 
    + \sum_x \rho(x) \lambda(x) \prn*{\sum_a \pi(a\mid{}x) - 1}
    - \sum_{x}\rho(x)\sum_a \alpha(x,a)\pi(a\mid{}x).
  \end{align*}
  Slater's condition holds since $\bar\pi$ itself is a strictly feasible solution, and the objective is convex in $\pi(a\mid{}x)$.  
  Then if $(\pi, \lambda,\alpha)$ satisfy the KKT conditions, they are the optimal primal and dual variables, 
  which, overloading notation, 
  we denote as $(\pistar, \lambda^\star, \alpha^\star)$. 
  
  We will demonstrate that setting $\pistar = \bar\pi$, $\lambda^\star = 0$, and $\alpha^\star = 0$ satisfies the KKT conditions. 
  First, we observe that the proposed solutions 
  are primal and dual feasible. 
  Further, we have $\bar\pi > 0$ since
  $0 \notin \dom(f')$ and $\bar\pi(a\mid{}x) \in \dom(f')$. 
  As a result,  
  $\rho(x)\alpha^\star(x,a)\pi(a\mid{}x) = 0$ for all $x,a$, 
  and complementary slackness is satisfied.     
  Lastly, for stationarity, 
  \begin{align*}
    \frac{\partial \cL(\pi,\lambda,\alpha)}{\partial \pi(a\mid{}x)} 
    =&~ \rho(x)\prn*{
    -\bar{r}(x,a)
    + \beta f'\prn*{\frac{\bar\pi(a\mid{}x)}{\piref(a\mid{}x)}} 
    + \lambda^\star(x) 
    - \alpha^\star(x,a)
    }
    \\ 
    =&~
    \rho(x)\prn*{
    -\bar{r}(x,a)
    + \beta f'\prn*{\frac{\bar\pi(a\mid{}x)}{\piref(a\mid{}x)}}
    } 
    \\
    =&~ 
    \rho(x)\prn*{
    -\beta f'\prn*{\frac{\bar\pi(a\mid{}x)}{\piref(a\mid{}x)}}
    + \beta f'\prn*{\frac{\bar\pi(a\mid{}x)}{\piref(a\mid{}x)}}
    }
    \\
    =&~ 0,
  \end{align*} 
  where in the second line we substitute $\lambda^\star = 0$ and $\alpha^\star = 0$, and in third line we have utilized the definition of $\bar {r}(x,a)$ from the lemma statement. 
\end{proof}

\begin{proof}[\pfref{lem:clip-dpo-estimation}]
For a pair of policies $\pi,\pi'$ and $p \ge 1$, we define the norm 
$\nrm{\cdot}_{p, \pi\times\pi'} \ldef{} \prn*{\En_{\rho,a\sim\pi,b\sim\pi'}\brk{|\cdot|^p}}^{1/p}$.  
In addition, for notational compactness, we abbreviate 
$\rhatdiff(x,a,b) \ldef{} \rhat(x,a) - \rhat(x,b)$, 
and 
$\rstardiff(x,a,b) \ldef{} \rstar(x,a) - \rstar(x,b)$. 

Recall that our goal is 
to bound the (unclipped) reward estimation error under $\pi$ 
using the (clipped) reward estimation error $\piref$. 
We begin by decomposing
\begin{align*}
  \nrm*{\rstardiff - \rhatdiff}_{1,\pi\times{\piref}} 
  =&~ 
  \nrm*{\rstardiff 
  - \clip_{2\Rmax}\brk*{\rhatdiff} 
  + \clip_{2\Rmax}\brk*{\rhatdiff} 
  - \rhatdiff}_{1,\pi\times{\piref}}
  \\
  \le&~ 
  \nrm*{\rstardiff - \clip_{2\Rmax}\brk*{\rhatdiff}}_{1,\pi\times{\piref}} 
  + \nrm*{\prn*{\clip_{2\Rmax}\brk*{\rhatdiff} - \rhatdiff}\cdot \indic\brk*{\clip_{2\Rmax}\brk*{\rhatdiff} \neq \rhatdiff}}_{1,\pi\times{\piref}}
  \\
  \le&~ 
  \underbrace{\nrm*{\rstardiff - \clip_{2\Rmax}\brk*{\rhatdiff}}_{1,\pi\times{\piref}}}_{\text{(I) clipped on-policy estimation error}} 
  + \underbrace{\Vmax\cdot\bbP_{\pi,{\piref}}\prn*{\clip_{2\Rmax}\brk*{\rhatdiff} \neq \rhatdiff}}_{\text{(II) bias from clipping}}.
\end{align*}
This splits our bound into two terms. 
The first is the on-policy error of the clipped reward differences, and can be directly bounded by \cref{lem:clip-dpo-reward} using a standard change-of-measure argument. 
The second expresses the error of translating the clipped estimates to the unclipped ones in our target bound. 
For the first term, using Cauchy-Schwarz gives
\begin{align*}
  \text{(I)} = \nrm*{\rstardiff - \clip_{2\Rmax}\brk*{\rhatdiff}}_{1,\pi\times{\piref}} 
  \le&~ \sqrt{\Cone \cdot \nrm*{\rstardiff - \clip_{2\Rmax}\brk*{\rhatdiff}}_{2,{\piref}\times{\piref}}^2}
  \\
  =&~ \sqrt{\Cone \cdot \nrm*{\clip_{2\Rmax}\brk*{\rstardiff} - \clip_{2\Rmax}\brk*{\rhatdiff}}_{2,{\piref}\times{\piref}}^2},
\end{align*}
where the last equality uses that $\rstardiff\in\brk{-\Rmax,\Rmax}$.

Next, for the second term, we again use Cauchy-Schwarz to change measure onto the offline distribution, 
\begin{align*}
  \text{(II)} = \Vmax \cdot \bbP_{\pi\times{\piref}}\prn*{\clip_{2\Rmax}\brk*{\rhatdiff} \neq \rhatdiff} 
  \le \Vmax \cdot \sqrt{\Cone \cdot \bbP_{{\piref},{\piref}}\prn*{\clip_{2\Rmax}\brk*{\rhatdiff} \neq \rhatdiff}}.  
\end{align*}
Further, using Markov's inequality along with the fact that $\rstardiff\in\brk{-\Rmax,\Rmax}$,
\begin{align*}
  \bbP_{{\piref},{\piref}}\prn*{\clip_{2\Rmax}\brk*{\rhatdiff} \neq \rhatdiff}
    \le&~ \bbP_{{\piref},{\piref}}\prn*{\abs*{\clip_{2\Rmax}\brk*{\rhatdiff}} = 2\Rmax}
    \\
    \le&~ \bbP_{{\piref},{\piref}}\prn*{\abs*{\clip_{2\Rmax}\brk*{\rhatdiff} - \clip_{2\Rmax}\brk*{\rstardiff}} \ge  \Rmax}
    \\
    \leq&~ \frac{1}{\Rmax^2} \nrm*{\clip_{2\Rmax}\brk*{\rhatdiff} - \clip_{2\Rmax}\brk*{\rstardiff}}_{2,{\piref}\times{\piref}}^2. 
\end{align*} 
 
Combining inequalities, we obtain 
\begin{align*}
  \nrm*{\rstardiff - \rhatdiff}_{1,\pi\times{\piref}} 
  \le&~ 
  \prn*{1 + \frac{\Vmax}{\Rmax}} 
  \sqrt{\Cone \cdot \nrm*{\clip_{2\Rmax}\brk*{\rhatdiff} - \clip_{2\Rmax}\brk*{\rstardiff}}_{2,{\piref}\times{\piref}}^2}
  \\
  =&~ 
  \prn*{1 + \frac{\Vmax}{\Rmax}} 
  \sqrt{\prn*{1+2\Dchis{\pi}{\piref}} \cdot \vepsstat^2}
  \\
  \le&~ \frac{2\Vmax}{\Rmax} \sqrt{\prn*{1+2\Dchis{\pi}{\piref}} \cdot \vepsstat^2}.
\end{align*}
In the second line we have used $\Cone = 1 + 2\Dchis{\pi}{\piref}$ and
the definition of $\vepsstat^2$ from \cref{lem:clip-dpo-reward}, 
and in the last line we use $\Vmax\ge\Rmax$.    

\end{proof}

\begin{lemma}
  \label{lem:rlhf-mix-solution}
  When $\piref(a\mid{}x)> 0$ for all $x\in\cX$, the optimal policy $\pistarg$ for \cref{eq:rlhf-mix} satisfies 
  \begin{align*}
    \rstar(x,a) 
    = 
    \linkg\prn*{\frac{\pistarg(a\mid{}x)}{\piref(a\mid{}x)}} 
    + \lambdag(x),
  \end{align*}  
  where $\lambdag$ is an optimal dual variable that normalizes $\pistarg$.   
\end{lemma} 
\begin{proof}[\pfref{lem:rlhf-mix-solution}]
It is easy to see that strong duality holds for \cref{eq:rlhf-mix},
since it is convex and strictly feasible (e.g., for the policy
$\piref$). Thus, the KKT conditions give the optimal primal and dual solutions. 

Since \cref{eq:rlhf-mix} is constrained optimization problem (over valid policies), we first define the dual variables. 
Below, $\lambda : \cX \rightarrow \bbR$ corresponds to the equality constraint that $\sum_a \pi(a\mid{}x) = 1$ for all $x \in \cX$, and $\alpha : \cX \times \cA \rightarrow \bbR_{\ge 0}$ corresponds to the inequality constraint that $\pi(a\mid{}x) \ge 0$ for all $(x,a) \in \cX\times\cA$. 
After converting \cref{eq:rlhf-mix} from maximization to minimization, we write \cref{eq:rlhf-mix} in Lagrangian form as  
\begin{align*}
  \cL(\pi,\lambda,\alpha)
  = 
  -\En_\pi\brk{\rstar(x,a)} 
  + \beta\Dfmixg{\pi}{\piref} 
  + \sum_x \rho(x) \lambda(x) \prn*{\sum_a \pi(a\mid{}x) - 1}
  - \sum_{x}\rho(x)\sum_a \alpha(x,a)\pi(a\mid{}x),
\end{align*}
since multiplying each of the solutions by $\rho(x)$ does not affect
the value of the saddle-point problem. 
We denote the optimal primal variable as $\pistarg$, and optimal dual variables as $(\lambdag, \alphag)$. 
 
From stationarity, the optimal primal and dual variables satisfy
\begin{align*}
  \rstar(x,a) 
  = 
  \linkg\prn*{\frac{\pistarg(a\mid{}x)}{\piref(a\mid{}x)}} 
  + \lambdag(x) 
  - \alphag(x,a). 
\end{align*}

Next, for a function $g$ let $g^{-1}$ denote its left inverse, such that $g^{-1}(g(x)) = x$. Because $\linkg$ is injective (see proof of \cref{lem:general-reward-to-policy}), it has a left inverse $(\linkg)^{-1}$, and we can write 
\begin{align*}
  \pistarg(a\mid{}x) 
  = 
  \piref(a\mid{}x)\cdot(\linkg)^{-1}
  \prn*{\rstar(x,a) - \lambdag(x) + \alphag(x,a)}.
\end{align*}  
Because $\linkg(z) = z + \gamma\log(z)$, 
$0 \notin \dom(\linkg)$, 
and therefore $0 \notin \mathrm{range}((\linkg)^{-1})$.
Then from the above expression, 
we observe that $\pistarg(a\mid x) > 0$ 
since $\piref(a\mid x) > 0$. It immediately follows that 
$\alphag(x,a) = 0$ for all $(x,a)$ from complementary slackness, 
which states that the optimal solutions satisfy
$\pistarg(a\mid{}x)\cdot\alphag(x,a) = 0$ for all $x,a$. This allows us to reduce the expression for $\rstar$ to the stated result, that is, 
\begin{align*}
  \rstar(x,a) 
  = 
  \linkg\prn*{\frac{\pistarg(a\mid{}x)}{\piref(a\mid{}x)}} + \lambdag(x).  
\end{align*}
\end{proof}

\begin{lemma}
\label{lem:lipschitz}
  For $z \in [-R, R]$ and $z' \in  [-V, V]$ where $V \ge R \ge 1$, we have 
  \begin{align*}
    |z - z'| 
    \le&~ 4e^{2R} V \cdot \left|\sigma(z) - \sigma(z')\right|.
  \end{align*} 
  Additionally, if we define the distribution $P_{z}(y) = \indic\crl{y=+1}\sigma(z)
  + \indic\crl{y=-1}\sigma(-z)$ for $y \in \{-1, +1\}$ and define $P_{z'}$
  analogously, then 
  \begin{align*}
    |z - z'| 
    \le&~ 4e^{2R} V \cdot \Dhel{P_z}{P_{z'}}. 
  \end{align*}
\end{lemma}

\begin{proof}[\pfref{lem:lipschitz}]
We begin with the first statement, and write  
\begin{align*}
  |z - z'|
  =&~ 
  \frac{|z - z'|}{|\sigma(z) - \sigma(z')|}
  \cdot 
  |\sigma(z) - \sigma(z')|.
\end{align*}

Since $\sigma(z') \in (0,1)$ but $z' \in [-V, V]$, it can be observed that the slope $\frac{|z - z'|}{|\sigma(z) - \sigma(z')|}$ is smallest where $z \approx z'$, and increases as we move away from this region in either direction. To better intuit the scaling of the slope in terms of $V$, we expand $|\sigma(z) - \sigma(z')|$ in the denominator to write 
\begin{align*}
  |z - z'|
  = 
  \frac{|z-z'|(1+e^{z})(1+e^{z'})}{|e^{z} - e^{z'}|}  
  \cdot 
  |\sigma(z) - \sigma(z')|.  
\end{align*}
This indicates that the slope should scale linearly (not exponentially) with the range of $z'$. 
For example, as $z' \rightarrow \infty$,  $(1+e^{z'}) / |e^z- e^{z'}| = O(1)$.

To make this intuition precise, we split into two cases. First, whenever $e^{z'} \ge \frac{e^{{R} + z} + 1}{e^{R} - 1}$ or $e^{z'} \le \frac{e^{R+z} - 1}{e^{R}+1}$ (this constitutes the range where ``$z' \approx z$''), we have $1+e^{z'} \le e^{R}|e^z - e^{z'}|$. Then in this region, 
\begin{align*}
  |z - z'|
  = \frac{|z-z'|(1+e^{z})(1+e^{z'})}{|e^{z} - e^{z'}|}  |\sigma(z) - \sigma(z')| \le 2V (1+e^{R}) e^{R}\cdot|\sigma(z) - \sigma(z')|. 
\end{align*} 

Next, for $e^{z'} \in [\frac{e^{R + z} - 1}{e^{R} + 1}, \frac{e^{R+z} + 1}{e^{R}-1}]$, we apply the mean value theorem. Since $\sigma'(x) = e^{x}(1+e^{-x})^{-2}$,   
\begin{align*}
  \frac{|z - z'|}{|\sigma(z) - \sigma(z')|} 
  \le&~ \sup_{\tilde z \in \brk*{\min\{z, {z'}\}, \max\{z, {z'}\}}} e^{\tilde z}(1+e^{-\tilde z})^{-2}
  \\
  \le&~ \sup_{e^{\tilde z} \in \brk*{\frac{e^{R + z} - 1}{e^{R} + 1}, \frac{e^{R+z} + 1}{e^{R}-1}}} e^{\tilde z}(1+e^{-\tilde z})^{-2}
  \\
  \le&~ 4e^{R}. 
\end{align*}  
In the second inequality, we use the fact that $e^{z'}, e^z \in [\frac{e^{R + z} - 1}{e^R + 1}, \frac{e^{R+z} + 1}{e^R-1}]$, and in the third inequality we use the fact that $\sigma'(x)$ is increasing in $x$, and that $|z| \le R$. 
Combining the inequalities for the two regions of $e^{z'}$ gives the result. 

For the second statement, we use the fact that 
\begin{align*}
  2\Dhels{P_z}{P_{z'}} \ge \sum_{y \in \{+1, -1\}}\frac{(P_z(y) - P_{z'}(y))^2 }{P_z(y) + P_{z'}(y)}.  
\end{align*}
As a result, 
\begin{align*}
  \sum_{y \in \{+1, -1\}}(P_z(y) - P_{z'}(y))^2  \le 4\Dhels{P_z}{P_{z'}}. 
\end{align*}
Since $P_z(y) = 1 - P_z(-y)$ and $P_z(+1) = \sigma(z)$,  
\begin{align*}
  \sum_{y \in \{+1, -1\}}(P_z(y) - P_{z'}(y))^2 = 2(\sigma(z) - \sigma(z'))^2,  
\end{align*}
and therefore $(\sigma(z) - \sigma(z'))^2 \le 2\Dhels{P_z}{P_{z'}}$. The result follows from taking the square root of both sides and combining with the first statement in the lemma. 
\end{proof}

\subsection{Proof of \creftitle{thm:main}}
\label{sec:proof-main}

\begin{proof}[\pfref{thm:main}]
  The policy optimization in \cref{line:chi_dpo} of \cref{alg:main} is a special case of \cref{eq:dpo-mix} with $\gamma = 1$. As a result, \cref{thm:main} follows directly from \cref{thm:main-mix} when instantiated with $\gamma = 1$. 
\end{proof}

\subsection{Proof of \creftitle{cor:reward_model}}
\begin{proof}[\pfref{cor:reward_model}]%
  Recall that for any $\beta>0$, 
  \cref{thm:main} (\cref{eq:main1}) with the policy class $\Pi_{\cR}$
  ensures that with probability at least
  $1-\delta$, for all $\pistar$,
  \begin{align}
    \label{eq:rm0}
    J(\pistar) - J(\pihat) \leq c_1\Rmax e^{2\Rmax}\cdot\sqrt{\frac{\Cone[\pistar]\log(|\cR|/\delta)}{n}} + c_2\beta\Cone[\pistar] + c_3\beta^{-1}\frac{\Rmax^2 e^{4\Rmax}\log(|\cR|/\delta)}{n}
  \end{align}
  for absolute constants $c_1,c_2,c_3>0$. Let us invoke this result with
  \[
    \betastar = \argmax_{\beta>0}\max_{\pistar}\crl*{
      J(\pistar) - c_1\Rmax
      e^{2\Rmax}\cdot\sqrt{\frac{\Cone[\pistar]\log(|\cR|/\delta)}{n}}
      - c_2\beta\Cone[\pistar] - c_3\beta^{-1}\frac{\Rmax^2 e^{4\Rmax}\log(|\cR|/\delta)}{n}
      }.
    \]
    Then \cref{eq:rm0} implies that
    \begin{align*}
\max_{\pistar}\crl*{
      J(\pistar) - c_1\Rmax
      e^{2\Rmax}\cdot\sqrt{\frac{\Cone[\pistar]\log(|\cR|/\delta)}{n}}
      - c_2\betastar\Cone[\pistar] - c_3(\betastar)^{-1}\frac{\Rmax^2 e^{4\Rmax}\log(|\cR|/\delta)}{n}
      } - J(\pihat) \leq{} 0,
    \end{align*}
    so that by the definition of $\betastar$,
    \begin{align*}
      \max_{\beta>0}\max_{\pistar}\crl*{
      J(\pistar) - c_1\Rmax
      e^{2\Rmax}\cdot\sqrt{\frac{\Cone[\pistar]\log(|\cR|/\delta)}{n}}
      - c_2\beta\Cone[\pistar] - c_3\beta^{-1}\frac{\Rmax^2 e^{4\Rmax}\log(|\cR|/\delta)}{n}
      } - J(\pihat) \leq{} 0,
    \end{align*}
    or equivalently
    \begin{align*}
      J(\pistar) - J(\pihat) \leq c_1\Rmax e^{2\Rmax}\cdot\sqrt{\frac{\Cone[\pistar]\log(|\cR|/\delta)}{n}} + c_2\beta\Cone[\pistar] + c_3\beta^{-1}\frac{\Rmax^2 e^{4\Rmax}\log(|\cR|/\delta)}{n}\quad\forall\pistar,\forall\beta>0.
    \end{align*}
    It follows that for all comparator policies $\pistar$, we have
    \begin{align*}
        J(\pistar) - J(\pihat) \approxleq \Rmax e^{2\Rmax}\cdot\sqrt{\frac{\Cone[\pistar]\log(|\cR|/\delta)}{n}}
    \end{align*}
    by choosing $\beta\propto{}\sqrt{\frac{\Rmax^2
        e^{4\Rmax}\log(|\cR|/\delta)}{\Cone[\pistar]n}}$ above.
  
\end{proof}

\section{Proofs for \creftitle{sec:understanding}}
\label{sec:proofs_understanding}
\begin{proof}[\pfref{prop:link}]
  To see that $\phi$ and $\phi^{-1}$ are strictly increasing, we note
  that $\phi'(z) = 1 + \frac{1}{z}>0$ for all $z>0$.

We now bound the inverse function $\phi^{-1}$. We will use the fact
that $z\mapsto{}W_0(z)$ is increasing over $z\geq{}0$ throughout. We first consider the regime where $z\geq{}1$. Since $W_0(\cdot)$ is
increasing, we have that $\link^{-1}(z) = W_0(e^z) \leq{} z$ if and only if
$e^{z}\leq{}ze^{z}$, which is clearly true for $z\geq{}1$. On the
other hand, for $c>0$ we have $\link^{-1}(z) = W_0(e^z)\geq{}c\cdot{}z$ if and only if
$e^z\geq{}cze^{cz}$; setting $c=1/2$ is clearly sufficient.

We now consider the regime where $z\leq{}1$. Here, we see that
$\link^{-1}(z) = W(e^z)\leq{}e^{z}$ if and only if $e^{z}\leq{} e^{z}e^{e^{z}}$, which
holds for all $z\in\bbR$. On the other hand have that
$\link^{-1}(z) = W(e^z)\geq{}e^{-e}e^{z}$ if and only if $e^{z}\geq{}
e^{-e}e^{z}e^{e^{-e}e^{z}}$. Since $z\leq{}1$, we have
\[
  e^{-e}e^{z}e^{e^{-e}e^{z}}
  \leq{}   e^{-e}e^{z}e^{e^{z}}
  \leq{}   e^{-e}e^{z}e^{e} = e^{z},
\]
which establishes the result.

\end{proof}

\begin{proof}[\pfref{prop:conc_bounds}]%
Recall that the optimal policy satisfies  
  \begin{align}
  \label{eq:chis_opt_gamma}
r(x,a) = \beta\link\prn*{\frac{\pistarb(a\mid{}x)}{\piref(a\mid{}x)}} + \Zr[r](x),
  \end{align}
  where $\Zr[r](x)$ is a normalization
constant chosen such that $\pistarb(\cdot\mid{}x)$ is a valid
probability distribution.

We begin by bounding $\Zr[r](x)$. We will use that
$r(x,a)\in\brk{0,\Rmax}$. Let $x\in\cX$ be fixed. By averaging
\cref{eq:chis_opt_gamma} over $a\sim\pistarb(x)$, we have
\begin{align*}
  \En_{a\sim\pistarb(x)}\brk{r(x,a)}
  = \beta
  \En_{a\sim\pistarb(x)}\brk*{\frac{\pistarb(a\mid{}x)}{\piref(a\mid{}x)}}
  + \beta\Dkl{\pistarb}{\piref} + \Zr[r](x)
  \geq{} \Zr[r](x),
\end{align*}
so $\Zr[r](x)\leq{}\Rmax$. On the other hand, averaging over
$a\sim\piref(x)$, we have
\begin{align*}
  \En_{a\sim\pistarb(x)}\brk{r(x,a)}
  &= \beta
  \En_{a\sim\piref(x)}\brk*{\frac{\pistarb(a\mid{}x)}{\piref(a\mid{}x)}}
    - \beta\Dkl{\piref}{\pistarb} + \Zr[r](x)\\
    &\leq{} \beta+ \Zr[r](x),
\end{align*}
so $\Zr[r](x)\geq{}-\beta$.

Having established that $\Zr[r](x)\in\brk*{-\beta,\Rmax}$, we will use that
$\link\prn*{\frac{\pistarb(a\mid{}x)}{\piref(a\mid{}x)}}=\beta^{-1}(r(x,a)-\Zr[r](x))$,
so that our bound on $\Zr[r]$ implies that
\[
-\beta^{-1}\Rmax\leq{}\link\prn*{\frac{\pistarb(a\mid{}x)}{\piref(a\mid{}x)}} \leq{} 1+\beta^{-1}\Rmax,
\]
or, since $\link^{-1}$ is increasing,
\[
e^{-e}\cdot{}e^{-\beta^{-1}\Rmax}\leq{}\link^{-1}(-\beta^{-1}\Rmax)\leq{}\frac{\pistarb(a\mid{}x)}{\piref(a\mid{}x)} \leq{} \link^{-1}(1+\beta^{-1}\Rmax)\leq{}1+\beta^{-1}\Rmax,
\]
where we have used that $\link^{-1}(z)\leq{}z$ for $z\geq{}1$ and
$\link^{-1}(z)\geq{} e^{z-e}$ for $z\leq{}1$  (by \cref{prop:link}).

\end{proof}

\section{Proofs for \creftitle{sec:general_preference}}
\label{sec:proofs_general_preference}
\subsection{Proof of \cref{thm:general_lower}}
  \begin{proof}[\pfref{thm:general_lower}]%
    \renewcommand{\ind}[1]{^{#1}}%
    We consider a family of instances in which there is a single
    context (prompt) $\cX=\{\emptyset\}$ and four actions (responses)
    $\cA=\{a,b,c,d\}$. We consider the reference policy $\piref$ given by
\begin{align*}
\piref(a'\mid{}x)=
\begin{cases}
\frac{1}{C}, &\text{ if $a'=a$ or $a'=b$,}\\
1-\frac{2}{C}, &\text{ if $a'=c$.}\\
\end{cases}
\end{align*}

We consider a preference model class
$\scrP=\crl*{\cP^1,\cP^2}$ in which
\[
\cP\ind{i}(a\ind{0}\psdgt{}a\ind{1}\mid{}x) = (1+\ell\ind{i}(x,a^0,a^1))/2
\]
for a function $\ell\ind{i}(x,a^0, a^1)\in\brk{-1,+1}$.
The functions $\ell\ind{1}$ and $\ell\ind{2}$ are defined as follows
(we omit the dependence on $x$, since there is a single context):
\begin{align*}
&\ell^1(a^0,a^1)=\ell^2(a^0,a^1)=0,\quad\forall a^0\in\cA, a^1\in\{a,b,c\},\\
&\ell^1(a,d)=0,\quad\ell^1(b,d)=-1,\quad\ell^1(c,d)=1\\
&\ell^2(a,d)=-1,\quad\ell^2(b,d)=0,\quad\ell^2(c,d)=-1.
\end{align*}
Note that both functions are skew-symmetric in the sense that
$\ell(x,a',a') = 0$ and $\ell(x,a^0,a^1) + \ell(x, a^1, a^0) =
0$ for all $x\in\cX$ and $a^0,a^1\in\cA$.

It is straightforward to see that the deterministic policies
$\piu^1(x)=a$ and $\piu^2(x)=b$ are minimax winners for $\ell^1$ and
$\ell^2$ respectively. Observe that for both policies, we have
\[
  \cC^{\pimw\ind{1}}_{\infty}
  = \cC^{\pimw\ind{2}}_{\infty} = C.
\]
To proceed, we compute duality gap an arbitrary policy $\pi$ under
$\cP^1$ and $\cP^2$. Let $\DG(\pi;\cP)$ denote the value of $\DG(\pi)$
when $\cP$ is the true preference model. Then we have: 
\begin{align*}
\max_{q\in\Delta(\Ac)} l(q,\pi) &= \max_{q\in\Delta(\Ac)} -q(b)\pi(d)+q(c)\pi(d) + q(d)\pi(b)-q(d)\pi(c),\\
\min_{q\in\Delta(\Ac)} l(\pi,q) &= \min_{q\in\Delta(\Ac)} -\pi(b)q(d)+\pi(c)q(d) + \pi(d)q(b)-\pi(d)q(c),\\
&=-\max_{q\in\Delta(\Ac)} -q(b)\pi(d)+q(c)\pi(d) + q(d)\pi(b)-q(d)\pi(c).
\end{align*}
Therefore we know
\begin{align*}
\DG(\pi;\cP^1)=2\max_{q\in\Delta(\cA)}q(d)(\pi(b)-\pi(c)) -\pi(d)(q(b)-q(c))
\end{align*}
Following similar computations, we have
\begin{align*}
\DG(\pi;\cP^2)=2\max_{q\in\Delta(\cA)}q(d)(\pi(a)+\pi(c)) - \pi(d)(q(a)+q(c)).
\end{align*}
We aim to show that for all policies $\pi$,
$\DG(\pi;\cP^1)+\DG(\pi;\cP^2)\geq\frac{1}{2}$. To do so, we consider
two cases. Going forward, we will use that $\DG(\pi;\cP\ind{i})\geq{}0$.
\paragraph{Case (1): $\boldsymbol{\pi(a)+\pi(c)\geq\frac{1}{2}}$} In
this case, we have $\DG(\pi;\cP^2)\geq\frac{1}{2}$, and thus $\DG(\pi;\cP^1)+\DG(\pi;\cP^2)\geq\frac{1}{2}$.
\paragraph{Case (2): $\boldsymbol{\pi(a)+\pi(c)<\frac{1}{4}}$} In this
case, let $\theta\ldef{}\pi(b)-\pi(c)$. Then we have
$\DG(\pi;\cP^1)\geq2\max\{\theta,\pi(d)\}$. We observe that $\theta+\pi(d) = \pi(b)+\pi(d)-\pi(c)>\frac{3}{4}-\frac{1}{4}=\frac{1}{2}$. This implies that $\DG(\pi;\cP^1)>\frac{1}{2}$, and thus $\DG(\pi;\cP^1)+\DG(\pi;\cP^2)\geq\frac{1}{2}$.

Having established that all $\pi$ satisfy
$\DG(\pi;\cP^1)+\DG(\pi;\cP^2)\geq\frac{1}{2}$ we can apply the Le Cam
two-point method (specifically, the variant based on the
Bretagnolle-Huber inequality (e.g., Theorem 14.2 in
\citet{lattimore2020bandit})), which leads to the following inequality
\begin{align*}
\inf_{\mathsf{Alg}}\sup_{\cP\in\scrP}\En_{\cD_{\pref}}[\DG(\hpi;\cP)]\geq\frac{1}{8}\exp\left(-n\cdot\Dkl{\rho\otimes\piref\otimes\piref\otimes \cP\ind{1}}{\rho\otimes\piref\otimes\piref\otimes \cP\ind{2}}\right).
\end{align*}
It can be observed that $\Dkl{\rho\otimes\piref\otimes\piref\otimes
\cP\ind{1}}{\rho\otimes\piref\otimes\piref\otimes \cP\ind{2}}=0$, since
$\ell^1(a^0,a^1)=\ell^2(a^0,a^1)=0$ for all $a^0,a^1\in\{a,b,c\}$, and
$\piref$ is supported on $\crl{a,b,c}$. We conclude that any policy
derived from $\cDpref$ must have
\begin{align*}
  \En\brk*{\DG(\pihat;\cP\ind{i})} \geq{} \frac{1}{8}
\end{align*}
for some $i$.
\end{proof}

 \subsection{Proof of \cref{thm:rdr}}
 \begin{proof}[\pfref{thm:rdr}]%
Let $\tpi$ be the global best response of $\hpi$:
\begin{align*}
	\tpi = \argmax_{\pi\in\Pi}\Eb_{x\sim \rho,a\sim \pi(x), b\sim\hpi(x)}\left[\lss(x,a,b)\right],
\end{align*}
and let $\tpi_{C}$ be the best response within $\Pi_{C}$ of $\hpi$ where $C\geq 1$ (recall that $\Pi_C:=\{\pi:\max_{x\in\Xc}
\Dchis{\pi(x)}{\piref(x)}\leq C\}$ denotes the set of policies with bounded $\chi^2$-divergence w.r.t. $\piref$):
\begin{align*}
	\tpi_{C} = \argmax_{\pi\in\Pi_{C}}\Eb_{x\sim \rho,a\sim \pi(x), b\sim\hpi(x)}\left[\lss(x,a,b)\right].
\end{align*}
Recall that $\ovr^t(x,a):=\Eb_{b\sim\pi^t(x)}[\hl(x,a,b)]$. Then we know
\begin{align}
	\lss(\tpi,\hpi)=&\suboptun(\hpi,C)+\underbrace{\frac{1}{T}\sum_{t=1}^T\left(\hr^t(\tpi_C)-\hr^t(\pi^t)\right)}_{(1)}+ \underbrace{\frac{1}{T}\sum_{t=1}^T\left(\lss(\tpi_{C},\pi^t)-\hl(\tpi_{C},\pi^t)\right)}_{(2)}\notag\\
 &+\underbrace{\frac{1}{T}\sum_{t=1}^T(\ovr^t(\tpi_C)-\hr^t(\tpi_C))}_{(3)} + \underbrace{\frac{1}{T}\sum_{t=1}^T(\hr^t(\pi^t)-\ovr^t(\pi^t))}_{(4)},\label{eq:bound}
\end{align}
\noindent where $r(\pi):=\Eb_{x\sim\rho,a\sim\pi(x)}[r(x,a)]$. The decomposition utilizes the fact that $\ovr^t(\pi^t)=0$ and $\ovr^t(\tpi_C) = \widehat \ell(\tpi_C, \pi^t)$. This implies that we only need to bound term (1)(2)(3)(4) in \cref{eq:bound} to upper bound the gap of $\hpi$. 

\paragraph{Bounding term (1)} Let $g_x(p)$ to denote the mixed divergence $\beta\Dfmix{p(x)}{\piref(x)}$. Then we have the following guarantee on regularized policy mirror descent (formal version of  \cref{lem:md-general-informal}):
\begin{lemma}
\label{lem:md-general}
For any $C\geq 0$, we have for all policy $\pi\in\Pi_{C}$ that
\begin{align*}
\frac{1}{T}\sum_{t=1}^T\left(\hr^t(\pi)-\hr^t(\pi^t)\right)\leq &\frac{2\beta C}{\eta T} +2\beta C -\frac{1}{T}\sum_{t=1}^{T+1}\Eb_{x\sim\rho} [g_x(\pi^t)]\\
&+ \frac{\eta}{2\beta} + \frac{1}{T}\sum_{t=1}^T\Eb_{x\sim\rho}\left[\left\langle\hr^t(x,\cdot)-G^t(\pi^{t+1},x,\cdot),\pi(x)-\pi^{t+1}(x)\right\rangle\right],
\end{align*}
where $G^t(\pi,x,a):=\beta\left((1+\frac{1}{\eta})\phi\left(\frac{\pi(a|x)}{\piref(a|x)}\right)-\frac{1}{\eta}\phi\left(\frac{\pi^t(a|x)}{\piref(a|x)}\right)\right)$ for all $\pi\in\Pi,x\in\Xc,a\in\Ac$.
\end{lemma}
To simplify writing, we use $\bpi^{t+1}$ to denote the minimizer of the following regularized RL objective:
\begin{align*}
\bpi^{t+1}(x):=\arg\min_{p\in\Delta(\Xc)}\left\langle -\hr^t(x,\cdot), p\right\rangle + \beta\Dfmix{p}{\piref(x)} + \frac{\beta}{\eta}B_x(p,\pi^t),\qquad\forall x\in\Xc.
\end{align*}
Then \cref{ass:md} indicates that $\bpi^{t+1}\in\Pi$ for all $t\in[T]$. In addition, by introducing Lagrangian multipliers into the above optimization problem and following similar arguments in the proof of \cref{lem:rlhf-mix-solution}, we know
\begin{align}
\label{eq:f-hr}
f^{\beta,\eta}_{\bpi^{t+1},\pi^t}(x,a,b)-(\hr^t(x,a)-\hr^t(x,b)) = 0, \qquad\forall x\in\Xc,a,b\in\Ac.
\end{align}
Recall that by definition $f^{\beta,\eta}_{\pi,\pi^t}(x,a,b)= G^t(\pi,x,a)-G^t(\pi,x,b)$ for all policies $\pi\in\Pi$. This implies that we have
\begin{align*}
&\Eb_{x\sim\rho}\left[\left\langle\hr^t(x,\cdot)-G^t(\pi^{t+1},x,\cdot),\pi(x)-\pi^{t+1}(x)\right\rangle\right]\\
=&\Eb_{x\sim\rho}\left[\left\langle\hr^t(x,\cdot)-G^t(\pi^{t+1},x,\cdot),\pi(x)-\piref(x)\right\rangle\right] + \Eb_{x\sim\rho}\left[\left\langle\hr^t(x,\cdot)-G^t(\pi^{t+1},x,\cdot),\piref(x)-\pi^{t+1}(x)\right\rangle\right]\\
=& \underbrace{(f^{\beta,\eta}_{\bpi^{t+1},\pi^t}-f^{\beta,\eta}_{\pi^{t+1},\pi^t})(\rho,\pi,\piref)}_{(5)} + \underbrace{(f^{\beta,\eta}_{\pi^{t+1},\pi^t} - f^{\beta,\eta}_{\bpi^{t+1},\pi^t})(\rho,\pi^{t+1},\piref)}_{(6)},
\end{align*}
where we use $f(\rho,\pi,\pi')$ to denote the expectation $\Eb_{x\sim\rho,a\sim\pi(x),b\sim\pi'(x)}[f(x,a,b)]$ and the last step utilizes \cref{eq:f-hr}. Therefore, to bound term (1), we need to bound term (5) and (6) respectively. To simplify writing, we define $L(\pi,\pi',\pi'')$ as follows:
\begin{align*}
L(\pi,\pi',\pi''):=\Eb_{x\sim\rho,a\sim\piref(x),b\sim\piref(x)}\left[\left(\clip_4(f^{\beta,\eta}_{\pi,\pi''}(x,a,b))-\clip_4(f^{\beta,\eta}_{\pi',\pi''}(x,a,b))\right)^2\right],
\end{align*}

Note that we have the following guarantee of least squares regression from the literature (Lemma 15 in \citet{song2022hybrid})
\begin{lemma}[least squares regression]
\label{lem:lsr-general}
Let $\{(y_i,z_i)\}_{i=1}^K$ be a dataset of $K$ points where each point are independently sampled from $y_i\sim\mu$ and $z_i\sim p(\cdot|y_i):=h^*(y_i)+\eps_i$. Let $\Hc:\Yc\to[-R,R]$ be a real valued functions where $h^*\in\Hc$ and $R>0$. Then if $\{\eps_i\}_{i=1}^K$ are independent random variables such that $\Eb[z_i|y_i]=h^*(y_i)$, the least squares solution $\widehat{h}=\argmin_{h\in\Hc}\sum_{i=1}^K(h(y_i)-z_i)^2$ satisfies with probability at least $1-\delta$ that
\begin{align*}
\Eb_{x\sim\mu}[(\widehat{h}(y)-h^*(y))^2]\lesssim\frac{R^2\log(|\Hc|/\delta)}{K}.
\end{align*}
\end{lemma}
The proof of the above lemma is omitted. Applying  \cref{lem:lsr-general} to the least sqaures solution $\pi^{t+1}$, we have the following concentration lemma:
\begin{lemma}[concentration in optimization]
\label{lem:lsr-opt}
Suppose  \cref{ass:md} and \cref{ass:vmax-g} hold. Then with probability at least $1-\delta/4$, we have for all policy $t\in[T]$ that
\begin{align*}
L(\pi^{t+1},\bpi^{t+1},\pi^t)\leq \frac{\Ccon\log(|\Pi|/\delta)}{m}:=  \eop^2,
\end{align*}
where $\Ccon>0$ is a universal constant.
\end{lemma}
In the following discussion, we use $\Ec_1$ to denote the event in  \cref{lem:lsr-opt}. Then under  $\Ec_1$, by following the same arguments in the proof of  \cref{lem:clip-dpo-estimation}, we have the following bound on $\Vert f^{\beta,\eta}_{\bpi^{t+1},\pi^t}-f^{\beta,\eta}_{\pi^{t+1},\pi^t}\Vert_{1,\pi\times\piref}$:
\begin{align}
\label{lem:clip-general}
\Vert f^{\beta,\eta}_{\bpi^{t+1},\pi^t}-f^{\beta,\eta}_{\pi^{t+1},\pi^t}\Vert_{1,\pi\times\piref}\leq \Vmax\sqrt{\left(1+2\Dchis{\pi}{\piref}\right)  \eop^2},\qquad\forall \pi\in\Pi, t\in[T].
\end{align}
Therefore, with  \cref{lem:clip-general} we know that conditioned on $\Ec_1$, for any policy $\pi\in\Pi_{C}$ we have
\begin{align*}
(5)\leq \Vmax\sqrt{3C  \eop^2},\quad(6)\leq \Vmax\sqrt{\left(1+2\Dchis{\pi^{t+1}}{\piref}\right)  \eop^2}\leq \frac{\Vmax^2  \eop^2}{\beta} + \frac{1}{2}\Eb_{x\sim\rho} [g_x(\pi^{t+1})] + \Vmax\eop,
\end{align*}
where we use AM-GM inequality in the last step, the definition of $g_x(\pi) := \beta \Dfmix{\pi(\cdot | x)}{\piref(\cdot | x)}$, and $ \Dfmix{p(x)}{\piref(x)} \geq \Dchis{p(x)}{\piref(x)}$ since KL is non-negative

\medskip

\noindent In summary, conditioned on $\Ec_1$, we have
\begin{align}
\label{eq:bound-1}
(1)\leq &\frac{2\beta C}{\eta T} +2\beta C -\frac{1}{2T}\sum_{t=1}^{T+1}\Eb_{x\sim\rho} [g_x(\pi^t)]+ \frac{\eta}{2\beta} +\Vmax\sqrt{4C  \eop^2} + \frac{\Vmax^2  \eop^2}{\beta}.
\end{align}

\paragraph{Bounding term (2)} From Cauchy-Schwartz's inequality, we have
\begin{align*}
	&\lss(\tpi_C,\pi^t)-\hl(\tpi_C,\pi^t)\\
	&\qquad\leq\sqrt{\Eb_{x\sim\rho,a\sim\piref(x),b\sim\piref(x)}[(\lss(x,a,b)-\hl(x,a,b))^2]\left(1+2\Dchis{\rho \otimes\tpi_C \otimes\pi^t}{\rho \otimes\piref \otimes\piref}\right)},
\end{align*}
where $\rho \otimes\pi_1 \otimes\pi_2$ denotes the joint distribution of $(x,a,b)$ where $x\sim\rho,a\sim\pi_1(x),b\sim\pi_2(x)$ for all $\pi_1,\pi_2\in\Pi$. Applying the guarantee of least squares regression (\cref{lem:lsr-general}) to the least squares solution $\hl$, we have under  \cref{ass:pre}, with probability at least $1-\delta/4$, the following event holds:
\begin{align}
\label{eq:lsr}
\mathbb{E}_{x\sim\rho, y^0\sim\piref(x), y^1\sim\piref(x)} \left[\left(\hl(x, y^0, y^1) - \lss(x, y^0, y^1) \right)^2\right]\leq O\left(\frac{\ln( |\Lc| / \delta)}{n}\right):= \esg^2.
\end{align}
Denote the event in \cref{eq:lsr} by $\Ec_2$. On the other hand, we can obtain that:
\begin{align*}
	1 + 2\Dchis{\rho \otimes\tpi_C \otimes\pi^t}{\rho \otimes\piref \otimes\piref}&= \sum_{x}\rho(x)\sum_{a}\frac{(\tpi_C(a|x))^2}{\piref(a|x)}\sum_{b}\frac{(\pi^t(b|x))^2}{\piref(b|x)}\\
 &=\sum_{x}\rho(x)\left(1 + 2\Dchis{\tpi_C(x)}{\piref(x)}\right)\left(1+2\Dchis{\pi^t(x)}{\piref(x)}\right)\\
 &\leq 6C\left(\Eb_{x\sim\rho}\left[\Dchis{\pi^t(x)}{\piref(x)}\right]+1\right)
\end{align*}
where the last step is due to $\tpi_C\in\Pi_C$. Therefore, conditioned on $\Ec_2$, we have
\begin{align*}
&\lss(\tpi_C,\pi^t)-\hl(\tpi,\pi^t)\leq \sqrt{6C\Eb_{x\sim\rho}\left[\Dchis{\pi^t(x)}{\piref(x)}\right] \esg^2} + \sqrt{6C \esg^2}\\
&\qquad\leq \frac{1}{2}\Eb_{x\sim\rho} [g_x(\pi^t)] + \frac{3C \esg^2}{\beta} + \sqrt{6C \esg^2}.
\end{align*}
In summary, we have
\begin{align}
\label{eq:bound-2}
&\frac{1}{T}\sum_{t=1}^T\lss(\tpi_C,\pi^t)-\hl(\tpi,\pi^t)\leq \frac{1}{2T}\sum_{t=1}^T\Eb_{x\sim\rho} [g_x(\pi^t)] + \frac{3C \esg^2}{\beta} + \sqrt{6C \esg^2}.
\end{align}

\paragraph{Bounding term (3)} Recall that $\hr^t(x,a)=\hl(x,a,b_t)$ where $b_t\sim \pi^t(x)$ is an unbiased estimator of $\ovr^t$. Fix any policy $\pi\in\Pi$, then from Azuma-Hoeffding's inequality, we have with probability at least $1-\delta'$ that
\begin{align*}
\left|\sum_{t=1}^{T}\hr^t(\pi)-\sum_{t=1}^T\ovr^t(\pi)\right|\lesssim \sqrt{T\log(1/\delta')}.
\end{align*}
By union bound, with probability at least $1-\delta/4$ we have that for all $\pi\in\Pi$:
\begin{align*}
\left|\sum_{t=1}^{T}\hr^t(\pi)-\sum_{t=1}^T\ovr^t(\pi)\right|\lesssim \sqrt{T\log(|\Pi|/\delta)}.
\end{align*}
Therefore, specifically for $\tpi_C$, we have
\begin{align}
\label{eq:bound-3}
(3)\lesssim\sqrt{\frac{\log(|\Pi|/\delta)}{T}}.
\end{align}

\paragraph{Bounding term (4)} From Azuma-Hoeffding's inequality, we have with probability at least $1-\delta/4$ that
\begin{align*}
\left|\sum_{t=1}^{T}\hr^t(\pi^t)-\sum_{t=1}^T\ovr^t(\pi^t)\right|\lesssim \sqrt{T\log(1/\delta')}.
\end{align*}
Therefore, we have
\begin{align}
\label{eq:bound-4}
(4)\lesssim\sqrt{\frac{\log(1/\delta)}{T}}.
\end{align}

\paragraph{Putting everything together} Substituting \cref{eq:bound-1}\eqref{eq:bound-2}\eqref{eq:bound-3}\eqref{eq:bound-4} into \eqref{eq:bound}, we have with probability at least $1-\delta$ that
\begin{align*}
	\lss(\tpi,\hpi)\lesssim\suboptun(\hpi,C)+\frac{C\beta}{\eta T} +C\beta + \frac{\eta}{\beta} &+ \Vmax\sqrt{C  \eop^2} + \frac{\Vmax^2  \eop^2}{2\beta}\\
& + \frac{C \esg^2}{\beta} + \sqrt{C \esg^2} + \sqrt{\frac{\log\frac{|\Pi|}{\delta}}{T}}.
\end{align*}
By selecting
\begin{align*}
T = \frac{mn}{n\Vmax^2 + m}, \qquad\beta = \frac{1}{\sqrt{T}}, \qquad\eta = \frac{1}{T},
\end{align*}
we have with probability at least $1-\delta$ that
\begin{align*}
\lss(\tpi,\hpi)&\lesssim\suboptun(\hpi,C)+ C\left(\frac{\Vmax\log(|\Pi|/\delta)}{\sqrt{m}} + \frac{\log(|\Pi||\Lc|/\delta)}{\sqrt{n}}\right)
\end{align*}

\noindent Note that due to the skew symmetry of $\lss$, we have:
\begin{align*}
&\min_{\pi\in\Pi}\Eb_{x\sim \rho,a\sim \hpi(x), b\sim\pi(x)}\left[\lss(x,a,b)\right]=-\max_{\pi\in\Pi}\Eb_{x\sim \rho,a\sim \pi(x), b\sim\hpi(x)}\left[\lss(x,a,b)\right]=-	\lss(\tpi,\hpi).
\end{align*}
This implies that $\DG(\hpi)\leq 2\lss(\tpi,\hpi)$, which concludes our proof.
 \end{proof}

\subsection{Proofs for Supporting Lemmas}
 \begin{proof}[\pfref{lem:md-general}]
First for all $t\in[T], s\in\Sc$ and any policy $\pi\in\Pi_C$, we have
	\begin{align*}
		&\left\langle \eta\hr^t(x), \pi(x) - \pi^t(x)\right\rangle + \eta g_x(\pi^t) - \eta g_x(\pi)\\
		=&\left\langle \eta\hr^t(x)-(1+\eta)\nabla g_x(\pi^{t+1}) + \nabla g_x(\pi^{t}), \pi(x) - \pi^{t+1}(x)\right\rangle\\
		& + \underbrace{\left\langle\nabla g_x(\pi^{t+1}) - \nabla g_x(\pi^t), \pi(x)-\pi^{t+1}(x)\right\rangle}_{(7)} +\underbrace{\left\langle \eta\hr^t(x), \pi^{t+1}(x)-\pi^t(x)\right\rangle}_{(8)}\\
		&  + \underbrace{\left\langle \eta\nabla g_x(\pi^{t+1}), \pi(x)-\pi^{t+1}(x)\right\rangle+ \eta g_x(\pi^t)-\eta g_x(\pi)}_{(9)},
	\end{align*}
 Note that we have
\begin{align*}
\left\langle \eta\hr^t(x)-(1+\eta)\nabla g_x(\pi^{t+1}) + \nabla g_x(\pi^{t}), \pi(x) - \pi^{t+1}(x)\right\rangle=\eta\left\langle\hr^t(x,\cdot)-G^t(\pi^{t+1},x,\cdot),\pi(x)-\pi^{t+1}(x)\right\rangle
\end{align*}
Next we bound the term (7)(8)(9) respectively.
	
	\paragraph{Bounding term (7)} Note that we have the following three point lemma:
	\begin{lemma}[three point lemma]
		\label{lem:three}
		For any $p_1,p_2,p_3:\Xc\mapsto\Delta(\Yc)$, we have for all $x\in\Xc$
		\begin{align*}
			\frac{1}{\beta}\left\langle \nabla g_x(p_1)-\nabla g_x(p_2),p_3(x)- p_1(x)\right\rangle = B_x(p_3, p_2) - B_x(p_3, p_1) - B_x(p_1, p_2).
		\end{align*}
	\end{lemma}
	\begin{proof}
		By definition, we know
		\begin{align*}
			\beta B_x(p,p')=g_x(p)- g_x(p') - \langle \nabla g_x(p'),p-p'\rangle.
		\end{align*}
		Substitute the definition into  \cref{lem:three} and we can prove the lemma.
	\end{proof}
	\noindent From  \cref{lem:three}, we can rewrite (7) as follows:
	\begin{align*}
		(7) = \beta\left(B_x(\pi,\pi^t) - B_x(\pi,\pi^{t+1}) -B_x(\pi^{t+1},\pi^t)\right).
	\end{align*}
	
	\paragraph{Bounding term (8)} From Cauchy-Schwartz inequality, we have
	\begin{align*}
		(8)\leq\sum_{a\in\Ac} \frac{\beta(\pi^{t+1}(a|x)-\pi^{t}(a|x))^2}{2\piref(a|x)}+ \frac{\piref(a|x)\eta^2(\hr^t(x,a))^2}{2\beta}\leq \beta B_x(\pi^{t+1},\pi^t) + \frac{\eta^2}{2\beta},
	\end{align*}
 where the last step comes from the definition of $B_x$.
	
	\paragraph{Bounding term (9)} Since $g_x$ is convex, we know
	\begin{align*}
		\left\langle \eta\nabla g_x(\pi^{t+1}), \pi-\pi^{t+1}\right\rangle\leq \eta g_x(\pi)- \eta g_x(\pi^{t+1}).
	\end{align*} 
	This implies that
	\begin{align*}
		(3)\leq \eta\left(g_x(\pi^t)-g_x(\pi^{t+1})\right).
	\end{align*}
	
	\noindent In summary, for all $t\in[T], s\in\Sc$ and any policy $\pi\in\Pi_C$, we have
	\begin{align*}
		&\left\langle \eta\hr^t(x), \pi(x) - \pi^t(x)\right\rangle + \eta g_x(\pi^t) - \eta g_x(\pi)\leq \beta\left(B_x(\pi,\pi^t) - B_x(\pi,\pi^{t+1})\right) \\
  &\qquad+ \eta\left(g_x(\pi^t)-g_x(\pi^{t+1})\right) + \frac{\eta^2}{2\beta} + \eta\left\langle\hr^t(x,\cdot)-G^t(\pi^{t+1},x,\cdot),\pi(x)-\pi^{t+1}(x)\right\rangle.
	\end{align*}
	
	This implies that for any policy $\pi\in\Pi_C$:
	\begin{align*}
		\sum_{t=1}^T\left(\hr^t(\pi)-\hr^t(\pi^t)\right)\leq& T\Eb_{x\sim\rho}[g_x(\pi)] - \sum_{t=1}^{T+1} \Eb_{x\sim\rho}[g_x(\pi^t)] + \frac{\beta}{\eta}\Eb_{x\sim\rho}\left[B_x(\pi,\pi^1)\right] + \frac{\eta T}{2\beta}\\
 & +\sum_{t=1}^T\Eb_{x\sim\rho}\left[\left\langle\hr^t(x,\cdot)-G^t(\pi^{t+1},x,\cdot),\pi(x)-\pi^{t+1}(x)\right\rangle\right]\\
 \leq&2TC\beta - \sum_{t=1}^{T+1} \Eb_{x\sim\rho}[g_x(\pi^t)] + \frac{2C\beta}{\eta} + \frac{\eta T}{2\beta}\\
 & +\sum_{t=1}^T\Eb_{x\sim\rho}\left[\left\langle\hr^t(x,\cdot)-G^t(\pi^{t+1},x,\cdot),\pi(x)-\pi^{t+1}(x)\right\rangle\right]
	\end{align*}
Here the last step uses the fact that $B_x(\cdot,\piref)=\frac{1}{\beta}g_x(\cdot)$ and $\pi\in\Pi_C$. This concludes our proof.
\end{proof}

\begin{proof}[\pfref{lem:lsr-opt}]
Let $\widehat{L}(\pi,\pi',\pi'')$ denote the empirical squared loss:
\begin{align*}
\widehat{L}(\pi,\pi',\pi''):=\sum_{(\ox,\oa,\ob)}\left(\clip_{4}(f^{\beta,\eta}_{\pi,\pi''}(\ox,\oa,\ob))-\clip_{4}(f^{\beta,\eta}_{\pi',\pi''}(\ox,\oa,\ob))\right)^2.
\end{align*}

Fix any $\pi',\pi''\in\Pi$ and consider the following LSR problems:
\begin{align*}
\pi(\pi',\pi''):=\argmin_{\pi\in\Pi}\widehat{L}(\pi,\pi',\pi'').
\end{align*}
Then from  \cref{lem:lsr-general}, we know with probability at least $1-\delta'$ that
\begin{align*}
L(\pi(\pi',\pi''),\pi',\pi'')\lesssim\frac{\log(|\Pi|/\delta')}{M}.
\end{align*}
Therefore, by union bound, we know with probability at least $1-\delta'$ that for all $\pi',\pi''\in\Pi$:
\begin{align*}
L(\pi(\pi',\pi''),\pi',\pi'')\lesssim\frac{\log(|\Pi|/\delta')}{M}.
\end{align*}
The proof is concluded by noticing that $\pi^{t+1}=\argmin_{\pi\in\Pi}\widehat{L}(\pi,\bpi^{t+1},\pi^{t})$ under  \cref{ass:md}.
 \end{proof}

\section{Proofs for \creftitle{sec:rlhf}}
\label{sec:proofs_rlhf}

The section contains the proofs for the main guarantee \rlhfalg in
\cref{sec:rlhf} (\cref{thm:rlhf}). We first prove two results,
\cref{thm:rlhf-sample} and \cref{cor:rlhf-smooth}, which correspond to
exact (i.e., including precise constants) versions of the two
statements in \cref{thm:rlhf}. We also analyze \rlhfalg with $\eta = 0$ in \cref{cor:rlhf-nosmooth}.

Throughout this section, we make use of the following $\eta$-smoothed version of the $L_1$ concentrability coefficient:
\begin{align*}
  \Csmth[\pi] 
  \ldef{} 
  \En_\pi\brk*{\frac{\pi(a\mid{}x)}{\piref(a\mid{}x) + \eta\pi(a\mid{}x)}}. 
\end{align*}  
It is easy to see that for any $\eta \ge 0$ we have $\Csmth \le
\Cone$, as well as $\Csmth \le \eta^{-1}$. 

\begin{theorem}
[General regret bound for \cref{alg:rlhf}]
\label{thm:rlhf-sample}
Suppose \cref{ass:reward-realizability} and
\cref{ass:rlhf-policy-realizability} hold for parameters $\beta > 0$ and 
  $\eta \in \brk[\big]{0,  \frac{\beta}{8\Rmax}}$. Then with
  probability at least $1-\delta$, the policy $\pihat$ produced by \rlhfalg (\cref{alg:rlhf}) satisfies 
  \begin{align*}
    J(\pistar) - J(\pihat) 
    \le&~
    2\sqrt{\Csmth[\pistar]\cdot \vepsstat^2} 
    + 2\beta \cdot \Csmth[\pistar]
    + 4\beta^{-1} \cdot \vepsstat^2
    \\
    &+ 4\beta\cdot \prn*{\min\crl*{\Cinf[\pistar],\eta^{-1}} + \min\crl*{\max_{\pi\in\Pi} \Cinf[\pi], \eta^{-1}}} \vepsx^2 
    + 2\Rmax\vepsx.   
  \end{align*}
  where 
  $\vepsstat^2 = \frac{32\Rmax^2 e^{4\Rmax} \log(3|\cR|/\delta)}{n}$ 
  and 
  $\vepsx = \sqrt{\frac{\log(3|\Pi|/\delta)}{2n_\nopref}}$. 
\end{theorem}

The following results are immediate consequences of \cref{thm:rlhf-sample}.

\begin{corollary}
[Smoothed \chis-regularization]
  \label{cor:rlhf-smooth}
   Given $\pistar$, let $\eta = \frac{\beta}{8\Rmax}$ and $\beta =
   2\sqrt{\frac{32\Rmax^2
       e^{4\Rmax}\log(3|\Rcal|/\delta)}{n\Cone[\pistar]}}$. Then under
   the preconditions of \cref{thm:rlhf-sample}, with
  probability at least $1-\delta$, the policy $\pihat$ produced by \rlhfalg (\cref{alg:rlhf}) satisfies 
  \begin{align*}
    J(\pistar) - J(\pihat) 
    \le&~
    20 \Rmax e^{2\Rmax}\sqrt{\frac{2\Cone[\pistar]\log(3|\Rcal|/\delta)}{n}} 
    + \Rmax\sqrt{\frac{2\log(3|\Pi|/\delta)}{n_\nopref}} 
    + \frac{32\Rmax\log(3|\Pi|/\delta)}{n_\nopref}.   
  \end{align*}

\end{corollary}

\begin{corollary}
[Non-smoothed \chis-regularization]
  \label{cor:rlhf-nosmooth}
  Given $\pistar$, let $\eta = 0$ and $\beta = 2\sqrt{\frac{32\Rmax^2 e^{4\Rmax}\log(3|\Rcal|/\delta)}{n\Cone[\pistar]}}$. 
  Then under the preconditions of \cref{thm:rlhf-sample}, with
  probability at least $1-\delta$, the policy $\pihat$ produced by \rlhfalg (\cref{alg:rlhf}) satisfies 
  \begin{align*}
    J(\pistar) - J(\pihat) 
    \le&~
    20\Rmax e^{2\Rmax}\sqrt{\frac{2\Cone[\pistar] \log(3|\Rcal|/\delta)}{n}} 
    + \Rmax\sqrt{\frac{2\log(3|\Pi|/\delta)}{n_\nopref}}
    \\
    &+ 32\prn*{\Cinf[\pistar] + \max_{\pi \in \Pi}\Cinf[\pi]} 
    \cdot 
    \frac{\log(3|\Pi|/\delta)}{n_\nopref}
    \cdot 
    \sqrt{\frac{2\log(3|\cR|/\delta)}{n}}.    
  \end{align*}  
\end{corollary}

\begin{proof}[\pfref{thm:rlhf-sample}]
The proof follows largely the same lines of analyses as the proof of \cref{thm:main-mix}. 
One difference is that in \cref{alg:rlhf}, we approximate the RLHF
objective using contexts are sampled from $\cDnopref$, so we require
additional concentration arguments to show that the empirical
objective approximates its population counterpart.

\paragraph{Basic concentration results}
We begin by stating the two concentration inequalities, 
which, given the reward model $\rhat$ produced in \cref{eq:reward-estimation}, 
bound the error between $\Jhatsmthr[\rhat]$ and  its the population version $\Jsmthr[\rhat]$. 

We will handle the return and regularization terms separately, 
which will later allow us to obtain tighter bounds. Define 
\begin{align*}
  \wh{J}(\pi) 
  &\ldef{} \frac{1}{n_\nopref}\sum_{x \in \cDnopref}\En_{\pi}\brk*{\wh r(x,a)\mid{}x},
  \intertext{and}
  \Chatsmth(\pi) 
  &\ldef{}
  \frac{1}{n_\nopref}\sum_{x \in \cDnopref}\En_{\pi}\brk*{\sum_a \frac{\pi^2(a\mid{}x)}{\piref(a\mid{}x) + \eta\pi(a\mid{}x)} \mid{} x},
\end{align*}  
so that $\Jhatsmthr[\rhat](\pi) = \wh{J}(\pi) - \beta \Chatsmth(\pi)$. 

Fix $\delta' \in (0, 1]$, which we will specify at the end of this proof. Since $\max_{x} \En_\pi[\wh r(x,a)\mid{} x] \le \Rmax$, a straightforward application of Hoeffding's inequality guarantees that 
with probability at most $1-\delta'$,
for all $\pi \in \Pi$ we have that
\begin{align}
  \label{eq:context-hoeffding}
  \abs*{\wh{J}(\pi) - \En_\pi[\rhat(x,a)]}
  \le 
  \Rmax\sqrt{\frac{\log(2|\Pi|/\delta')}{2n_\nopref}}.
\end{align} 

Next, we consider the regularization term. 
Since $\sum_a \frac{\pi^2(a\mid{}x)}{\piref(a\mid{}x) + \eta\pi(a\mid{}x)} \le \min\{\Cinf, \eta^{-1}\}$ for any $x \in \cX$, 
we use Bernstein's inequality to derive the following result.  
\begin{lemma}
  \label{lem:chismooth-bernstein}
  With probability at least $1-\delta$, for any $\pi \in \Pi$, we have 
  \begin{align*}
    \abs*{\wh\cC_\eta^{\pi} - \Csmth[\pi]} 
    \le
    \frac{\Cone[\pi]}{2} + \frac{2\min\{\Cinf, \eta^{-1}\}\log(2|\Pi|/\delta)}{n_\nopref}.
  \end{align*}  
\end{lemma}

Define $\vepsx \ldef{} \sqrt{\frac{\log(2|\Pi|/\delta')}{2n_\nopref}}$. The above lemma implies that for all $\pi \in \Pi$, we have 
\begin{align*}
  \Chatsmth[\pi] \le~ \frac{3\Cone[\pi]}{2} + 4\min\{\Cinf, \eta^{-1}\}\cdot\vepsx^2, \mathand
  \Chatsmth[\pi] \ge~ \frac{\Cone[\pi]}{2} - 4\min\{\Cinf, \eta^{-1}\}\cdot\vepsx^2.
\end{align*}  
Together with \cref{eq:context-hoeffding}, this implies that for all $\pi\in\Pi$,
\begin{align}
  \Jhatsmthr[\rhat](\pi) 
  = \wh{J}(\pi) - \beta\Chatsmth
  \le&~ 
  \En_\pi\brk{\rhat(x,a)} 
  - \frac{\beta\Csmth}{2} 
  + 4\beta\min\{\Cinf, \eta^{-1}\}\vepsx^2
  + \Rmax\vepsx, 
  \label{eq:context-upper}
\intertext{and}
  \Jhatsmthr[\rhat](\pi) 
  = \wh{J}(\pi) - \beta\Chatsmth
  \ge&~ 
  \En_\pi\brk{\rhat(x,a)}
  - \frac{3\beta\Csmth}{2}
  - 4\beta\min\{\Cinf, \eta^{-1}\}\vepsx^2
  - \Rmax\vepsx.
  \label{eq:context-lower}
\end{align}

\paragraph{Estimation error bounds}

Next, we state the following off- and on-policy reward estimation
error bounds for the reward model $\rhat$, analogous to \cref{lem:clip-dpo-reward} and \cref{lem:clip-dpo-estimation} for \algshort.

\begin{lemma}
  \label{lem:rlhf-reward}
  Suppose \cref{ass:reward-realizability} holds. Then with probability at least $1-\delta$, the reward model $\rhat$ learned in \cref{eq:reward-estimation} satisfies 
  \begin{align*}
    \vepsstat^2 \rdef{} \En_{\piref,\piref}
    \brk*{\prn*{\prn*{\rhat(x,a) - \rhat(x,b)} 
    - \prn*{\rstar(x,a) - \rstar(x,b)}}^2} 
    \le 
    \frac{32\Rmax^2 e^{4\Rmax} \log(|\Pi|/\delta)}{n}.    
  \end{align*}
\end{lemma}
 
\begin{lemma}
  \label{lem:rlhf-estimation}
  Under the event in
  \cref{lem:rlhf-reward}, we have that for all $\pi: \cX
  \rightarrow \Delta(\cA)$, 
  \begin{align*}
    \En_{\pi,\piref}
    \brk*{\abs*{
    \prn*{\rhat(x,a) - \rhat(x,b)} 
    - \prn*{\rstar(x,a) - \rstar(x,b)}
    }} 
    \le 
    2\sqrt{\Csmth[\pi]\vepsstat^2}
    + 2\Csmth[\pi]\Rmax\eta,   
  \end{align*} 
  where $\vepsstat^2$ is defined in \cref{lem:rlhf-reward}. 
\end{lemma}

\paragraph{Regret decomposition}
Equipped with these concentration and estimation error bounds, we now
bound the regret of \cref{alg:rlhf} using a pessimism-based analysis
similar to the proof of \cref{thm:main-mix}. Condition on the events in 
  \cref{eq:context-hoeffding}, 
  \cref{lem:chismooth-bernstein}, and 
  \cref{lem:rlhf-reward}, which hold together with probability at least $1-3\delta'$.  
We decompose the regret of $\pihat$ 
using $\Jhatsmthr[\rhat]$, 
then leverage the inequalities in 
\cref{eq:context-upper} and \cref{eq:context-lower}:
\begin{align*}
  J(\pistar) - J(\pihat)
  =&~  
  J(\pistar) - \Jhatsmthr[\rhat](\pistar) 
  + \Jhatsmthr[\rhat](\pistar) - J(\pihat)
  \\
  \le&~ 
  J(\pistar) - \Jhatsmthr[\rhat](\pistar) 
  + \Jhatsmthr[\rhat](\pihat) - J(\pihat)
  \\
  \le&~ 
  J(\pistar) 
  - \En_{\pistar}\brk{\rhat(x,a)} 
  + \frac{3\beta\Csmth[\pistar]}{2}  
  + 4\beta\min\{\Cinf[\pistar], \eta^{-1}\}\vepsx^2
  + \Rmax\vepsx
  \\
  &+ \En_{\pihat}\brk{\rhat(x,a)} 
  - \frac{\beta\Csmth[\pihat]}{2} 
  + 4\beta\min\{\Cinf[\pihat], \eta^{-1}\}\vepsx^2
  + \Rmax\vepsx
  - J(\pihat)
  \\
  =&~ 
  \En_{\pistar,\piref}\brk{\rstardiff(x,a,b) - \rhatdiff(x,a,b)}
  + \frac{3\beta\Csmth[\pistar]}{2}   
  + \En_{\pihat,\piref}\brk{\rhatdiff(x,a,b) - \rstardiff(x,a,b)}
  - \frac{\beta\Csmth[\pihat]}{2}
  \\
  &+ 4\beta\vepsx^2 \prn*{\min\{\Cinf[\pistar], \eta^{-1}\} + \min\{\Cinf[\pihat], \eta^{-1}\}}
  + 2\Rmax\vepsx. 
\end{align*} 
In the last line above, 
we have introduced the notation 
$\rstardiff(x,a,b) = \rstar(x,a) - \rstar(x,b)$
and $\rhatdiff(x,a,b) = \rhat(x,a) - \rhat(x,b)$, 
and centered the returns. 
Next, applying \cref{lem:rlhf-estimation} to bound the reward
estimation error above, we have
\begin{align*}
  J(\pistar) - J(\pihat)
  \le&~
  2\sqrt{\Csmth[\pistar]\vepsstat^2}
  + 2\eta\Rmax\Csmth[\pistar]
  + \frac{3\beta\Csmth[\pistar]}{2}   
  \\
  &+ 2\sqrt{\Csmth[\pihat]\vepsstat^2}
  + 2\eta\Rmax\Csmth[\pihat]
  - \frac{\beta\Csmth[\pihat]}{2}
  \\
  &+ 4\beta\vepsx^2 \prn*{\min\{\Cinf[\pistar], \eta^{-1}\} + \min\{\Cinf[\pihat], \eta^{-1}\}}
  + 2\Rmax\vepsx. 
\end{align*}

Applying the AM-GM inequality to 
$2\sqrt{\Csmth[\pihat]\vepsstat^2}$ 
for $\eta \in \brk*{0,\frac{\beta}{4\Rmax}}$, we have 
\begin{align*}
  2\sqrt{\Csmth[\pihat]\vepsstat^2}
  =&~ 
  \sqrt{(\beta-4\eta\Rmax)\Csmth[\pihat]\cdot\frac{4\vepsstat^2}{(\beta-4\eta\Rmax)}}
  \\
  \le&~ 
  \frac{\beta \Csmth[\pihat]}{2} 
  - 2\eta\Rmax \Csmth[\pihat] 
  + \frac{2\vepsstat^2}{\beta-4\eta\Rmax}
  \\
  \le&~ 
  \frac{\beta \Csmth[\pihat]}{2} 
  - 2\eta\Rmax \Csmth[\pihat] 
  + \frac{4\vepsstat^2}{\beta},
\end{align*}
where in the last line we use the fact that $\eta \le \frac{\beta}{8\Rmax}$ so $4\eta\Rmax \le \frac{\beta}{2}$. Then plugging this back into our regret decomposition cancels out the $\Csmth[\pihat]$ terms to give
\begin{align*}
  J(\pistar) - J(\pihat) 
  \le&~
  2\sqrt{\Csmth[\pistar]\vepsstat^2}
  + 2\eta\Rmax\Csmth[\pistar]
  + \frac{3\beta\Csmth[\pistar]}{2} 
  + \frac{4\vepsstat^2}{\beta}
  \\
  &+ 4\beta\vepsx^2 \prn*{\min\{\Cinf[\pistar], \eta^{-1}\} + \min\{\Cinf[\pihat], \eta^{-1}\}} 
  + 2\Rmax\vepsx
  \\
  \le&~  
  2\sqrt{\Csmth[\pistar]\vepsstat^2}
  + 2\beta\Csmth[\pistar] 
  + \frac{4\vepsstat^2}{\beta}
  \\
  &+ 4\beta\vepsx^2 \prn*{\min\{\Cinf[\pistar], \eta^{-1}\} + \min\{\Cinf[\pihat], \eta^{-1}\}} 
  + 2\Rmax\vepsx,
\end{align*}
where in the last line we consolidate $\Csmth[\pistar]$ terms by again using $4\eta\Rmax \le \frac{\beta}{2}$.
Plugging in $\delta' = \delta/3$ and the values for $\vepsstat^2$ and $\vepsx$ results in the theorem statement.

\end{proof}

\begin{proof}[\pfref{cor:rlhf-smooth}]
  When $\eta = \frac{\beta}{8\Rmax}$, \cref{thm:rlhf-sample} states that
  \begin{align*}
    J(\pistar) - J(\pihat) 
    \le&~
    2\sqrt{\Csmth[\pistar]\vepsstat^2} 
    + 2 \beta \Csmth[\pistar]
    + \frac{4\vepsstat^2}{\beta}
    + 4\beta\vepsx^2 \cdot \prn*{ \min\crl*{\Cinf[\pistar],\eta^{-1}} +\min\crl*{\max_{\pi\in\Pi}\Cinf[\pi], \eta^{-1}}} 
    + 2\Rmax\vepsx
    \\
    \le&~
    2\sqrt{\Csmth[\pistar]\vepsstat^2} 
    + 2 \beta \Csmth[\pistar]
    + \frac{4\vepsstat^2}{\beta}
    + 8\beta\vepsx^2 \cdot \eta^{-1} 
    + 2\Rmax\vepsx
    \\
    =&~    
    2\sqrt{\Csmth[\pistar]\vepsstat^2} 
      + 2\beta \Csmth[\pistar]
      + \frac{4\vepsstat^2}{\beta}
      + 64\Rmax\vepsx^2 
      + 2\Rmax\vepsx.  
  \end{align*}
  Setting $\beta = 2\sqrt{\frac{\vepsstat^2}{\Cone[\pistar]}}$, we obtain 
  \begin{align*}
    J(\pistar) - J(\pihat) 
    \le&~ 
    5\sqrt{\Csmth[\pistar]\vepsstat^2} 
    + 64\Rmax\vepsx^2 
    + 2\Rmax\vepsx.
  \end{align*}
\end{proof}

\begin{proof}[\pfref{cor:rlhf-nosmooth}]
  When $\eta = 0$, \cref{thm:rlhf-sample} states that 
  \begin{align*}
    J(\pistar) - J(\pihat) 
    \le&~
    2\sqrt{\Cone[\pistar]\vepsstat^2} 
    + 2 \beta \Cone[\pistar]
    + \frac{4\vepsstat^2}{\beta}
    + 4\beta\vepsx^2 \cdot \prn*{\Cinf[\pistar] + \max_{\pi\in\Pi}\Cinf[\pi]} 
    + 2\Rmax\vepsx    
  \end{align*}  
  Setting $\beta = 2\sqrt{\frac{\vepsstat^2}{\Cone[\pistar]}}$, we obtain 
  \begin{align*}
    J(\pistar) - J(\pihat) 
    \le&~
    5\sqrt{\Cone[\pistar]\vepsstat^2} 
    + 8\vepsstat\vepsx^2 \cdot \prn*{\Cinf[\pistar] + \max_{\pi\in\Pi}\Cinf[\pi]} 
    + 2\Rmax\vepsx.   
  \end{align*}
\end{proof}

\begin{proof}[\pfref{lem:rlhf-reward}]
  We use similar reasoning and notation to the proof of \cref{lem:clip-dpo-reward}. 
  Since $\rstar\in\cR$ under \cref{ass:reward-realizability}, 
  \cref{lem:mle} guarantees that with probability at least $1-\delta$ we have 
  \begin{align*}
    \En_{\piref,\piref}\brk*{\Dhels{P_{\rhat}(\cdot\mid{}x,a,b)}{P_{\rstar}(\cdot\mid{}x,a,b)}}
    \le&~ 
    \frac{2\log(|\cR|/\delta)}{n}. 
  \end{align*}
  Since $|r(x,a) - r(x,b)| \le \Rmax$ for all $r \in \cR$ under \cref{ass:reward-realizability}, 
  we then apply \cref{lem:lipschitz} with $R = V = \Rmax$.
  \begin{align*}
    &\En_{\piref,\piref}
    \brk*{\prn*{\rhat(x,a) - \rhat(x,b)
    - \prn*{\rstar(x,a) - \rstar(x,b)}}^2}
    \\
    &\quad\le 
    16e^{4\Rmax} \Rmax^2 \cdot \En_{\piref,\piref}\brk*{\Dhels{P_{\rhat}(\cdot\mid{}x,a,b)}{P_{\rstar}(\cdot\mid{}x,a,b)}}
    \\
    &\quad\le 
    32e^{4\Rmax} \Rmax^2 \cdot \frac{\log(|\cR|/\delta)}{n}.
  \end{align*}
\end{proof}

\begin{proof}[\pfref{lem:rlhf-estimation}]
  Abbreviate $\rstardiff(x,a,b) = \rstar(x,a) - \rstar(x,b)$, and $\rhatdiff(x,a,b) = \rhat(x,a) - \rhat(x,b)$. For a pair of policies $\pi,\pi'$ and $p \ge 1$, we define the norm 
$\nrm{\cdot}_{p, \pi\times\pi'} \ldef{}
\prn*{\En_{\rho,a\sim\pi,b\sim\pi'}\brk{|\cdot|^p}}^{1/p}$, so that $    \En_{\pi,\piref}\brk*{\abs*{\rstardiff(x,a,b) -
      \rhatdiff(x,a,b)}} =    \nrm*{\rstardiff -
    \rhatdiff}_{1,\pi\times\piref}$. Then via Cauchy-Schwarz,
  \begin{align*}
    \begin{aligned}
\nrm*{\rstardiff - \rhatdiff}_{1,\pi\times\piref}      \le&~ 
      \sqrt{\En_\rho\brk*{\sum_{a,b} \frac{\pi^2(a\mid{}x)\piref^2(b\mid{}x)}{(\piref(a\mid{}x) + \eta\pi(a\mid{}x))\piref(b\mid{}x)}}}
      \\
      &\cdot    
      \sqrt{\En_\rho\brk*{\sum_{a,b} (\piref(a\mid{}x) + \eta\pi(a\mid{}x))\piref(b\mid{}x)\prn*{\rstardiff(x,a,b) - \rhatdiff(x,a,b)}^2}}
      \\
      =&~ 
      \sqrt{\Csmth[\pi] \cdot 
      \prn*{\nrm*{\rstardiff - \rhatdiff}_{2,\piref\times\piref}^2 
      + \eta\nrm*{\rstardiff - \rhatdiff}_{2,\pi\times\piref}^2}}
      \\
      \le&~ 
      \sqrt{\Csmth[\pi]\cdot\nrm*{\rstardiff - \rhatdiff}_{2,\piref\times\piref}^2} 
      +\sqrt{2\eta\Rmax\Csmth[\pi] \cdot \nrm*{\rstardiff - \rhatdiff}_{1,\pi\times\piref}}. 
    \end{aligned} 
  \end{align*} 
  Applying the AM-GM inequality to the second term, we obtain 
  \begin{align*}
    \nrm*{\rstardiff - \rhatdiff}_{1,\pi\times\piref} 
    \le&~ 
    \sqrt{\Csmth[\pi]\cdot\nrm*{\rstardiff - \rhatdiff}_{2,\piref\times\piref}^2} 
+ \eta\Rmax\Csmth[\pi] %
    + \frac{1}{2}\nrm*{\rstardiff - \rhatdiff}_{1,\pi\times\piref}. 
  \end{align*}
  Rearranging, 
  \begin{align*}
    \nrm*{\rstardiff - \rhatdiff}_{1,\pi\times\piref}
    \le&~ 
    2\sqrt{\Csmth[\pi] %
    \cdot\nrm*{\rstardiff - \rhatdiff}_{2,\piref\times\piref}^2}
    + 2\eta\Rmax\Csmth[\pi]. %
  \end{align*}
\end{proof}

\newpage

\end{document}